\documentclass[english]{article}
\usepackage[T1]{fontenc}
\usepackage[latin9]{inputenc}
\usepackage{float}
\usepackage{mathtools}
\usepackage{dsfont}
\usepackage{amsmath}
\usepackage{amsthm}
\usepackage{amssymb}
\usepackage{graphicx}

\makeatletter

\floatstyle{ruled}
\newfloat{algorithm}{tbp}{loa}
\providecommand{\algorithmname}{Algorithm}
\floatname{algorithm}{\protect\algorithmname}

\theoremstyle{plain}
\newtheorem{thm}{\protect\theoremname}[section]
\ifx\proof\undefined
\newenvironment{proof}[1][\protect\proofname]{\par
	\normalfont\topsep6\p@\@plus6\p@\relax
	\trivlist
	\itemindent\parindent
	\item[\hskip\labelsep\scshape #1]\ignorespaces
}{%
	\endtrivlist\@endpefalse
}
\providecommand{\proofname}{Proof}
\fi
\theoremstyle{plain}
\newtheorem{cor}[thm]{\protect\corollaryname}
\theoremstyle{plain}
\newtheorem{lem}[thm]{\protect\lemmaname}
\theoremstyle{remark}
\newtheorem{rem}[thm]{\protect\remarkname}

\usepackage{iclr2023_conference,times}

\usepackage{hyperref}
\usepackage{url}

\iclrfinalcopy

\@ifundefined{showcaptionsetup}{}{%
 \PassOptionsToPackage{caption=false}{subfig}}
\usepackage{subfig}
\makeatother

\usepackage{babel}
\providecommand{\corollaryname}{Corollary}
\providecommand{\lemmaname}{Lemma}
\providecommand{\remarkname}{Remark}
\providecommand{\theoremname}{Theorem}

\begin{document}
\global\long\def\R{\mathbb{R}}%
\global\long\def\diag{\mathbf{\text{diag}}}%
\global\long\def\bL{\mathbf{\mathbf{L}}}%
\global\long\def\bb{\mathbf{\mathbf{b}}}%
\global\long\def\E{\mathbb{E}}%
\global\long\def\te{e(p)}%
\global\long\def\c{c}%
\global\long\def\poly{\mathrm{poly}}%
\global\long\def\avx{\overline{x}}%

\title{On the Convergence of AdaGrad(Norm) on $\R^{d}$: Beyond Convexity,
Non-Asymptotic Rate and Acceleration}

\author{Zijian Liu\thanks{ Equal contribution, corresponding authors.} \\
New York University\\
\small{\texttt{zl3067@nyu.edu}} \\
\And
Ta Duy Nguyen\footnotemark[1]~ \& Alina Ene \\
Boston University \\
\small{\texttt{\{taduy,aene\}@bu.edu}} \\
\And
Huy L. Nguyen \\
Northeastern University \\
\small{\texttt{hu.nguyen@northeastern.edu}}
}

\maketitle
\begin{abstract}
Existing analysis of AdaGrad and other adaptive methods for smooth
convex optimization is typically for functions with bounded domain
diameter. In unconstrained problems, previous works guarantee an asymptotic
convergence rate without an explicit constant factor that holds true
for the entire function class. Furthermore, in the stochastic setting,
only a modified version of AdaGrad, different from the one commonly
used in practice, in which the latest gradient is not used to update
the stepsize, has been analyzed. Our paper aims at bridging these
gaps and developing a deeper understanding of AdaGrad and its variants
in the standard setting of smooth convex functions as well as the
more general setting of quasar convex functions. First, we demonstrate
new techniques to explicitly bound the convergence rate of the vanilla
AdaGrad for unconstrained problems in both deterministic and stochastic
settings. Second, we propose a variant of AdaGrad for which we can
show the convergence of the last iterate, instead of the average iterate.
Finally, we give new accelerated adaptive algorithms and their convergence
guarantee in the deterministic setting with explicit dependency on
the problem parameters, improving upon the asymptotic rate shown in
previous works. 
\end{abstract}

\section{Introduction}

In recent years, the prevalence of machine learning models has motivated
the development of new optimization tools, among which adaptive methods
such as Adam \citep{kingma2014adam}, AmsGrad \citep{reddi2018convergence},
AdaGrad \citep{duchi2011adaptive} emerge as the most important class
of algorithms. These methods do not require the knowledge of the problem
parameters when setting the stepsize as traditional methods like SGD,
while still showing robust performances in many ML tasks. 

However, it remains a challenge to analyze and understand the properties
of these methods. Take AdaGrad and its variants for example. In its
vanilla scalar form, also known as AdaGradNorm, the step size is set
using the cumulative sum of the gradient norm of all iterates so far.
The work of \cite{ward2020adagrad} has shown the convergence of this
algorithm for non-convex funtions by bounding the decay of the gradient
norms. However, in convex optimization, usually we require a stronger
convergence criterion---bounding the function value gap. This is
where we lack theoretical understanding. Even in the deterministic
setting, most existing works \citep{levy2017online, levy2018online, ene2021adaptive}
rely on the assumption that the domain of the function is bounded.
The dependence on the domain diameter can become an issue if it is
unknown or cannot be readily estimated. Other works for unconstrained
problems \citep{antonakopoulos2020adaptive, antonakopoulos2022undergrad}
offer a convergence rate that depends on the limit of the step size
sequence. This limit is shown to exist for each function, but without
an explicit value, and more importantly, it is not shown to be a constant
for the entire function class. This means that these methods essentially
do not tell us how fast the algorithm converges in the worst case.
Another work by \citet{ene2022adaptive} gives an explicit rate of
convergence for the entire class but requires the strong assumption
that the gradients are bounded even in the smooth setting and the
convergence guarantee has additional error terms depending on this
bound.

In the stochastic setting, one common approach is to analyze a modified
version of AdaGrad with off-by-one step size, i.e. the gradient at
the current time step is not taken into account when setting the new
step size. This is where the gap between theory and practice exists.

\subsection{Our Contribution}

In this paper, we make the following contributions. First, we demonstrate
a method to show an explicit non-asymptotic convergence rate of AdaGradNorm
and AdaGrad on $\R^{d}$ in the deterministic setting. Our method
extends to a more general function class known as $\gamma$-quasar
convex functions with a weaker condition for smoothness. To the best
of our knowledge, we are the first to prove this result. Second, we
present new techniques to analyze stochastic AdaGradNorm and offer
an explicit convergence guarantee for $\gamma$-quasar convex optimization
on $\R^{d}$ with a mild assumption on the noise of the gradient estimates.
We propose two new variants of AdaGradNorm which demonstrate the convergence
of the last iterate instead of the average iterate as shown in AdaGradNorm.
Finally, we propose a new accelerated algorithm with two variants
and show their non-asymptotic convergence rate in the deterministic
setting.

\subsection{Related Work }

\paragraph{Adaptive methods }

There has been a long line of works on adaptive methods, including
AdaGrad \citep{duchi2011adaptive}, RMSProp \citep{tieleman2012lecture}
and Adam \citep{kingma2014adam}. AdaGrad was first designed for stochastic
online optimization; subsequent works \citep{levy2017online,kavis2019unixgrad,bach2019universal,antonakopoulos2020adaptive,ene2021adaptive}
analyzed AdaGrad and various adaptive algorithms for convex optimization
and generalized them for variational inequality problems. These works
commonly assume that the optimization problem is contrained in a set
with bounded diameter. \cite{li2019convergence} are the first to
analyze a variant of AdaGrad for unbounded domains where the latest
gradient is not used to construct the step size, which differs from
the standard version of AdaGrad commonly used in practice. However,
the algorithm and analysis of \cite{li2019convergence} set the initial
step size based on the smoothness parameter and thus they do not adapt
to it. Other works provide convergence guarantees for adaptive methods
for unbounded domains, yet without explicit dependency on the problem
parameters \citep{antonakopoulos2020adaptive, antonakopoulos2022undergrad},
or for a class of strongly convex functions \citep{xie2020linear}.
Another work by \cite{ene2022adaptive} requires the strong assumption
that the gradients are bounded even for smooth functions and the convergence
guarantee has additional error terms depending on the gradient upperbound.
Our work analyzes the standard version of AdaGrad for unconstrained
and general convex problems and shows explicit convergence rate in
both the deterministic and stochastic setting. 

Accelerated adaptive methods have been designed to achieve $O(1/T^{2})$
and $O(1/\sqrt{T})$ respectively in the deterministic and stochastic
setting in the works of \cite{levy2018online,ene2022adaptive,antonakopoulos2022undergrad}.
We show different variants and demonstrate the same but explicit accelerated
convergence rate in the deterministic setting for unconstrained problems.

\paragraph{Analysis beyond convexity}

The convergence of some variants of AdaGrad has been established for
nonconvex functions in the work of \cite{li2019convergence,ward2020adagrad,faw2022power}
under various assumptions. Other works \citep{li2020high, kavis2022high}
demonstrate the convergence with high probability. We refer the reader
to \cite{faw2022power} for a more detailed survey on AdaGrad-style
methods for nonconvex optimization. In general, the criterion used
to study these convergence rates is the gradient norm of the function,
which is weaker than the function value gap normally used in the study
of convex functions. In comparison, we study the convergence of AdaGrad
via the function value gap for a broader notion of convexity, known
as quasar-convexity, as well as a more generalized definition of smoothness.

\section{Preliminaries}

We consider the following optimization problem: $\mathrm{minimize}_{x\in\R^{d}}F(x)$,
where $F$ is differentiable satisfying $F^{*}=\inf_{x\in\R^{d}}F(x)>-\infty$
and $x^{*}\in\arg\min_{x\in\R^{d}}F(x)\neq\emptyset$. We will use
the following notations throughout the paper: $a^{+}=\max\left\{ a,0\right\} $,
$a\lor b=\max\left\{ a,b\right\} $, $\left[n\right]=\left\{ 1,2,\cdots,n\right\} $,
and $\|\cdot\|$ denotes the $\ell_{2}$-norm $\|\cdot\|_{2}$ for
simplicity.

Additionally, we list below the assumptions that will be used in the
paper.

\textbf{1.} \textbf{$\gamma$-quasar convexity}: There exists $\gamma\in(0,1]$
such that $F^{*}\geq F(x)+\frac{1}{\gamma}\langle\nabla F(x),x^{*}-x\rangle,\forall x\in\R^{d}$
where $x^{*}\in\arg\min_{x\in\R^{d}}F(x)$. When $\gamma=1$, $F$
is also known as star-convex.

\textbf{1'.} \textbf{Convexity}: $F$ is convex. This stronger assumption
implies that Assumption 1 holds with $\gamma=1$.

\textbf{2.} \textbf{Weak} $L$\textbf{-smoothness}: $\exists L>0$
such that $F(x)-F^{*}\geq\|\nabla F(x)\|^{2}/2L,\forall x\in\R^{d}$.

\textbf{2'.} $L$\textbf{-smoothness}: $\exists L>0$ such that $F(x)\leq F(y)+\langle\nabla F(y),x-y\rangle+\frac{L}{2}\|x-y\|^{2},\forall x,y\in\R^{d}$.

\textbf{2''.} $\bL$\textbf{-smoothness}: $\exists\bL=\diag\left(L_{i\in\left[d\right]}\right)$
with $L_{i}>0$ such that $F(x)\leq F(y)+\langle\nabla F(y),x-y\rangle+\frac{1}{2}\|x-y\|_{\bL}^{2},\forall x,y\in\R^{d}$
where $\|a\|_{\bL}=\sqrt{\langle a,\bL a\rangle}$. 

In the stochastic setting, we assume that we have access to a stochastic
gradient oracle $\widehat{\nabla}F(x)$ that is independent of the
history of the randomness and it satisfies the following assumptions:

\textbf{3.} \textbf{Unbiased gradient estimate}: $\E[\widehat{\nabla}F(x)\mid x]=\nabla F(x)$.

\textbf{4. Sub-Weibull noise}: $\E\left[\exp\left((\|\widehat{\nabla}F(x)-\nabla F(x)\|/\sigma)^{1/\theta}\right)\mid x\right]\leq\exp(1)$
for some $\theta>0$.

Here, we give a brief discussion of our assumptions. Assumption 1
is introduced by \cite{hinder2020near} and it is strictly weaker
than Assumption 1'. Assumption 2 is a relaxation of Assumption 2',
the latter is the standard definition of smoothness used in many existing
works (see \cite{pmlr-v130-guille-escuret21a} for a detailed comparison
between different smoothness conditions). Assumption 2'' is used to
analyze the AdaGrad algorithm which uses per-coordinate step sizes.
Assumption 3 is a standard assumption in stochastic optimization problems.
Assumption 4 is more general and encapsulates sub-Gaussian ($\theta=1/2$,
used in \cite{li2019convergence}) and sub-exponential noise ($\theta=1$).
We refer the reader to \cite{vladimirova2020sub} for more discussion
on sub-Weibull noise.

\section{Convergence of AdaGradNorm on $\protect\R^{d}$ under $\gamma$-quasar
convexity\label{sec:Main-AdaGrad}}

We first turn our attention to AdaGradNorm (Algorithm \ref{alg:AdaGradNorm})
in the deterministic setting, which will serve as the basis for the
understanding of Stochastic AdaGradNorm (Algorithm \ref{alg:AdaGradNorm-stochastic})
and deterministic AdaGrad (Algorithm \ref{alg:AdaGrad}). To the best
of our knowledge, we are the first to present the explict convergence
rate of these three algorithms on $\R^{d}$. Due to the space limit,
we defer the theorem of the convergence guarantee of AdaGrad and its
proof to Section \ref{subsec:Appendix-AdaGrad} in the appendix.

\begin{figure*}[t]

\begin{minipage}[t]{0.475\columnwidth}%
\begin{algorithm}[H]
\caption{AdaGradNorm}
\label{alg:AdaGradNorm}

Initialize: $x_{1},\eta>0$

for $t=1$ to $T$

$\quad$$b_{t}=\sqrt{b_{0}^{2}+\sum_{i=1}^{t}\|\nabla F(x_{i})\|^{2}}$

$\quad$$x_{t+1}=x_{t}-\frac{\eta}{b_{t}}\nabla F(x_{t})$
\end{algorithm}
\end{minipage}\hfill{}%
\begin{minipage}[t]{0.475\columnwidth}%
\begin{algorithm}[H]
\caption{Stochastic AdaGradNorm}
\label{alg:AdaGradNorm-stochastic}

Initialize: $x_{1},\eta>0$

for $t=1$ to $T$

$\quad$$b_{t}=\sqrt{b_{0}^{2}+\sum_{i=1}^{t}\|\widehat{\nabla}F(x_{i})\|^{2}}$

$\quad$$x_{t+1}=x_{t}-\frac{\eta}{b_{t}}\|\widehat{\nabla}F(x_{t})\|^{2}$
\end{algorithm}
\end{minipage}

\end{figure*}

\subsection{AdaGradNorm\label{sec:Main-AdaGradNorm}}

Previous analysis of AdaGradNorm often aims at bounding the gradient
norm of smooth nonconvex functions, or is conducted for smooth convex
functions in constrained problems with a bounded domain. Bounding
the gradient norm is strictly weaker than bounding the function value
gap due to the fact that $\|\nabla F(x)\|^{2}\leq2L(F(x)-F^{*})$,
where $L$ is the smoothness parameter. For convex functions, the
common analysis will always meet the following intermediate step
\[
F(x_{t})-F^{*}\le\frac{b_{t}}{2\eta}\left[\|x_{t}-x^{*}\|^{2}-\|x_{t+1}-x^{*}\|^{2}\right]+\text{Other terms}.
\]
Assuming a bounded domain is a way to making the terms $\frac{b_{t}}{2\eta}\left[\|x_{t}-x^{*}\|^{2}-\|x_{t+1}-x^{*}\|^{2}\right]$
telescope after taking the sum over all iterations $t$. This is critical
in the analysis, but at the same time leads to the dependence on the
domain diameter, which can be hard to estimate. For unconstrained
problems, a natural approach is to divide the terms by $b_{t}$, so
that the remaining terms $\frac{1}{2\eta}\left[\|x_{t}-x^{*}\|^{2}-\|x_{t+1}-x^{*}\|^{2}\right]$
can telescope. Our key insight is that we can bound the function value
gap via the step size $b_{t}$, which in turn can be bounded via the
function value gap. This self-bounding argument allows us to finally
prove the convergence rate. This result holds under more general conditions
than convexity and smoothness (Assumptions 1 and 2).
\begin{thm}
\label{thm:Main-AdaGradNorm-rate}With Assumptions 1 and 2, AdaGradNorm
(Algorithm \ref{alg:AdaGradNorm}) admits
\[
\frac{\sum_{t=1}^{T}F(x_{t})-F^{*}}{T}\leq\frac{\left(\frac{2L\|x_{1}-x^{*}\|^{2}}{\gamma\eta}+\frac{4\eta L}{\gamma}\log^{+}\frac{2\eta L}{\gamma b_{0}}+b_{0}\right)\left(\frac{\|x_{1}-x^{*}\|^{2}}{\gamma\eta}+\frac{2\eta}{\gamma}\log^{+}\frac{2\eta L}{\gamma b_{0}}\right)}{T}
\]
\end{thm}
\begin{proof}
Starting from the $\gamma$-quasar convexity of $F$, we have
\begin{align*}
F(x_{t})-F^{*} & \leq\frac{\langle\nabla F(x_{t}),x_{t}-x^{*}\rangle}{\gamma}=\frac{b_{t}}{\gamma\eta}\langle x_{t}-x_{t+1},x_{t}-x^{*}\rangle\\
 & =\frac{b_{t}}{2\gamma\eta}\left[\|x_{t}-x^{*}\|^{2}-\|x_{t+1}-x^{*}\|^{2}+\|x_{t+1}-x_{t}\|^{2}\right]
\end{align*}
Notice that $x_{t+1}-x_{t}=-\eta b_{t}^{-1}\nabla F(x_{t})$. Dividing
both sides by $b_{t}$ and taking the sum over $t$, we obtain 
\begin{align*}
\sum_{t=1}^{T}\frac{F(x_{t})-F^{*}}{b_{t}} & \leq\frac{\|x_{1}-x^{*}\|^{2}}{2\gamma\eta}+\sum_{t=1}^{T}\frac{\eta}{2\gamma b_{t}^{2}}\|\nabla F(x_{t})\|^{2}.
\end{align*}
Note that $F$ also satisfies Assumption $2$, i.e., $F(x_{t})-F^{*}\geq\frac{\|\nabla F(x_{t})\|^{2}}{2L}.$
Therefore
\begin{align*}
\sum_{t=1}^{T}\frac{F(x_{t})-F^{*}}{2b_{t}}+\frac{\|\nabla F(x_{t})\|^{2}}{4Lb_{t}} & \leq\sum_{t=1}^{T}\frac{F(x_{t})-F^{*}}{b_{t}}\leq\frac{\|x_{1}-x^{*}\|^{2}}{2\gamma\eta}+\sum_{t=1}^{T}\frac{\eta}{2\gamma b_{t}^{2}}\|\nabla F(x_{t})\|^{2}\\
\Rightarrow\sum_{t=1}^{T}\frac{F(x_{t})-F^{*}}{b_{t}} & \leq\frac{\|x_{1}-x^{*}\|^{2}}{\gamma\eta}+\underbrace{\sum_{t=1}^{T}\left(\frac{\eta}{\gamma b_{t}^{2}}-\frac{1}{2Lb_{t}}\right)\|\nabla F(x_{t})\|^{2}}_{A}.
\end{align*}
We can bound the term $A$ by the technique commonly used in the analysis
of adaptive methods. Let $\tau$ be the last $t$ such that $b_{t}\le\frac{2\eta L}{\gamma}$.
If $b_{1}>\frac{2\eta L}{\gamma}$, we have $A<0\leq\frac{2\eta}{\gamma}\log^{+}\frac{2\eta L}{\gamma b_{0}}$.
Otherwise
\begin{align*}
A & \le\sum_{t=1}^{\tau}\left(\frac{\eta}{\gamma b_{t}^{2}}-\frac{1}{2Lb_{t}}\right)\|\nabla F(x_{t})\|^{2}\le\sum_{t=1}^{\tau}\frac{\eta}{\gamma}\frac{b_{t}^{2}-b_{t-1}^{2}}{b_{t}^{2}}\le\frac{\eta}{\gamma}\sum_{t=1}^{\tau}\log\frac{b_{t}^{2}}{b_{t-1}^{2}}\le\frac{2\eta}{\gamma}\log^{+}\frac{2\eta L}{\gamma b_{0}}.
\end{align*}
Thus we always have $A\le\frac{2\eta}{\gamma}\log^{+}\frac{2\eta L}{\gamma b_{0}}$,
and obtain 
\begin{align*}
\sum_{t=1}^{T}\frac{F(x_{t})-F^{*}}{b_{t}} & \leq\frac{\|x_{1}-x^{*}\|^{2}}{\gamma\eta}+\frac{2\eta}{\gamma}\log^{+}\frac{2\eta L}{\gamma b_{0}},
\end{align*}
which gives 
\begin{align*}
\sum_{t=1}^{T}F(x_{t})-F^{*} & \leq b_{T}\left(\frac{\|x_{1}-x^{*}\|^{2}}{\gamma\eta}+\frac{2\eta}{\gamma}\log^{+}\frac{2\eta L}{\gamma b_{0}}\right).
\end{align*}
Note that by Assumption 2 again, we have
\begin{align*}
b_{T} & =\sqrt{b_{0}^{2}+\sum_{t=1}^{T}\|\nabla F(x_{t})\|^{2}}\leq\sqrt{b_{0}^{2}+\sum_{t=1}^{T}2L\left(F(x_{t})-F^{*}\right)}.
\end{align*}
Let $\Delta_{T}=\sum_{t=1}^{T}F(x_{t})-F^{*}$, then
\begin{align*}
\Delta_{T} & \leq\sqrt{b_{0}^{2}+2L\Delta_{T}}\left(\frac{\|x_{1}-x^{*}\|^{2}}{\gamma\eta}+\frac{2\eta}{\gamma}\log^{+}\frac{2\eta L}{\gamma b_{0}}\right)\\
\Rightarrow\Delta_{T} & \leq\left(\frac{2L\|x_{1}-x^{*}\|^{2}}{\gamma\eta}+\frac{4\eta L}{\gamma}\log^{+}\frac{2\eta L}{\gamma b_{0}}+b_{0}\right)\left(\frac{\|x_{1}-x^{*}\|^{2}}{\gamma\eta}+\frac{2\eta}{\gamma}\log^{+}\frac{2\eta L}{\gamma b_{0}}\right).
\end{align*}
Dividing both sides by $T$, we get the desired result.
\end{proof}

When $F$ is convex (which implies $\gamma=1$), using the above theorem
and convexity, we obtain the following convergence rate for the average
iterate:
\begin{cor}
With Assumptions 1' and 2, for $\bar{x}_{T}=\frac{\sum_{t=1}^{T}x_{t}}{T}$,
AdaGradNorm (Algorithm \ref{alg:AdaGradNorm}) admits
\[
F(\bar{x}_{T})-F^{*}\leq\frac{\left(\frac{2L\|x_{1}-x^{*}\|^{2}}{\eta}+4\eta L\log^{+}\frac{2\eta L}{b_{0}}+b_{0}\right)\left(\frac{\|x_{1}-x^{*}\|^{2}}{\eta}+2\eta\log^{+}\frac{2\eta L}{b_{0}}\right)}{T}.
\]
\end{cor}
The rate in Theorem \ref{thm:Main-AdaGradNorm-rate} can be improved
by a factor $1/\gamma$ by replacing Assumption 2 by 2'. The details
and the proof are deferred into Section \ref{subsec:Appendix-AdaGradNorm}
in the appendix.

\subsection{Stochastic AdaGradNorm}

In this section, we consider the stochastic setting where we only
have access to an unbiased gradient estimate $\widehat{\nabla}F(x_{t})$
of $\nabla F(x_{t})$ (Assumption 3). As expected for a stochastic
method, the accumulation of noise is the reason that we can only expect
an $O(1/\sqrt{T})$ convergence rate, instead of $O(1/T)$. This convergence
rate is already shown by prior works \citep{levy2018online} under
the bounded domain assumption. However, in an unbounded domain, when
extending our previous analysis to the stochastic setting, that is,
dividing both sides by $b_{t}$, we will face several challenges.
One of such is the term $b_{t}^{-1}\langle\nabla F(x_{t})-\widehat{\nabla}F(x_{t}),x_{t}-x^{*}\rangle$.
To handle this term, often we see that existing works, such as \cite{li2019convergence},
analyze a modified version of Stochastic AdaGradNorm with off-by-one
stepsize, i.e., $b_{t}=\sqrt{b_{0}^{2}+\sum_{i=1}^{t-1}\|\widehat{\nabla}F(x_{i})\|^{2}}$
in which the latest gradient $\widehat{\nabla}F(x_{t})$ is not used
to calculate $b_{t}$. This allows to decouple the dependency of $b_{t}$
on the randomness at time $t$, thus in expectation $\E[b_{t}^{-1}\langle\nabla F(x_{t})-\widehat{\nabla}F(x_{t}),x_{t}-x^{*}\rangle]=0$.
Yet, this analysis does not apply to the standard algorithm which
is more commonly used in practice.

To the best of our knowledge, we are the first to propose a new technique
that can show the convergence of Algorithm \ref{alg:AdaGradNorm-stochastic}
on $\R^{d}$ without going through the off-by-one stepsize. Here,
we briefly compared the assumptions in our analysis with the assumptions
in \cite{li2019convergence}. Assumptions 2' and 3 used in both works
are standard. Meanwhile, Assumptions 1 ($\gamma$-quasar convexity)
and 4 (sub-Weibull noise) in our analysis are much weaker than the
convexity and sub-Gaussian noise assumptions in \cite{li2019convergence}.
Besides, we note that, while the guarantee in \cite{li2019convergence}
is a bound on $\E\big[\sqrt{(\sum_{t=1}^{T}F(x_{t})-F(x^{*}))/T}\big]$,
we will present a bound for $\E\big[(\sum_{t=1}^{T}F(x_{t})-F(x^{*}))/T\big]$,
which is a stronger criterion that is often used in convex analysis.
We also remark that the algorithm and analysis of \cite{li2019convergence}
still require the smoothness parameter to set the initial stepsize,
thus their method is not fully adaptive.

The first observation is that, if we let $\xi_{t}\coloneqq\widehat{\nabla}F(x_{t})-\nabla F(x_{t})$
be the stochastic error and $M_{T}\coloneqq\max_{t\in\left[T\right]}\|\xi_{t}\|^{2}$,
$M_{T}$ is bounded by $\sigma^{2}\log^{2\theta}\frac{eT}{\delta}$
with probability at least $1-\delta$ (c.f. Lemma \ref{lem:Appendix-LiOrabona-lemma}
in Appendix \ref{sec:Appendix-adagrad}), which can give a high probability
bound on $b_{T}$.
\begin{lem}
\label{lem:Main-AdaGradNorm-Stoc-b_T}Suppose $F$ satisfies Assumptions
2' and 4, if $M_{T}\leq\sigma^{2}\log^{2\theta}\frac{eT}{\delta}$,
then

\[
b_{T}\le2b_{0}+\frac{4(F(x_{1})-F^{*})}{\eta}+4\eta L\log^{+}\frac{\eta L}{b_{0}}+4\sigma\sqrt{T\log^{2\theta}\frac{eT}{\delta}\log\left(1+\frac{16\sigma^{2}T\log^{2\theta}\frac{eT}{\delta}}{b_{0}^{2}}\right)}.
\]
\end{lem}
Lemma \ref{lem:Main-AdaGradNorm-Stoc-b_T} gives us an insight: $b_{t}=\widetilde{O}(1+\sigma\sqrt{t\log^{2\theta}t})$.
Note that this can be expected since we know the classic choice of
the step size for SGD is of the order of $O(1+\sigma\sqrt{t})$. Hence,
if we are willing to accept extra $\log$ terms in the convergence
guarantee, the appearance of $\log b_{t}$ is accommodatable. Next
we will introduce our novel technique, which, to the best of our knowledge,
is the first method that allows us to analyze the standard Stochastic
AdaGradNorm on $\R^{d}$.
\begin{lem}
\label{lem:Main-function-gap-over-bT}Suppose $F$ satisfies Assumptions
1 and 3 then 
\begin{align}
\E\left[\frac{\sum_{t=1}^{T}F(x_{t})-F(x^{*})}{b_{T}}\right] & \le\frac{\left\Vert x_{1}-x^{*}\right\Vert ^{2}}{\gamma\eta}+\frac{2\eta}{\gamma}\E\left[\frac{M_{T}}{b_{0}^{2}}+\log\frac{b_{T}}{b_{0}}\right].\label{eq:lemma-gap-bt}
\end{align}
\end{lem}
\begin{proof}[Proof sketch]
Starting from the $\gamma$-quasar convexity, with simple transformations,
we obtain
\begin{align*}
F(x_{t})-F^{*} & \leq\frac{\langle-\xi_{t},x_{t}-x^{*}\rangle}{\gamma}+\frac{b_{t}}{2\gamma\eta}\left(\left\Vert x_{t}-x^{*}\right\Vert ^{2}-\left\Vert x_{t+1}-x^{*}\right\Vert ^{2}+\left\Vert x_{t+1}-x_{t}\right\Vert ^{2}\right).
\end{align*}
Here we introduce our novel technique: instead of dividing by $b_{t}$,
we divide both sides by $2b_{t}-b_{0}$. This divisor causes a slight
non-uniformity between the coefficients of the distance terms $\left\Vert x_{t}-x^{*}\right\Vert ^{2}$
making the sum of them not telescoping. However, this is exactly what
we want to handle the difficult term $\frac{\langle-\xi_{t},x_{t}-x^{*}\rangle}{\gamma(2b_{t}-b_{0})}$
which does not disappear after taking the expectation.
\begin{align*}
\E\left[\frac{F(x_{t})-F^{*}}{2b_{t}-b_{0}}\right] & \leq\E\left[\frac{\langle-\xi_{t},x_{t}-x^{*}\rangle}{\gamma(2b_{t}-b_{0})}+\frac{b_{t}}{2b_{t}-b_{0}}\times\frac{\left\Vert x_{t}-x^{*}\right\Vert ^{2}-\left\Vert x_{t+1}-x^{*}\right\Vert ^{2}+\left\Vert x_{t+1}-x_{t}\right\Vert ^{2}}{2\gamma\eta}\right].
\end{align*}
The key step is to use Cauchy-Schwarz inequality for the term 
\begin{align*}
\left|\langle-\xi_{t},x_{t}-x^{*}\rangle\right| & \le\frac{\lambda}{2}\left\Vert \xi_{t}\right\Vert ^{2}+\frac{1}{2\lambda}\left\Vert x_{t}-x^{*}\right\Vert ^{2}
\end{align*}
with the appropriate coefficient $\lambda$ so that the term $\|x_{t}-x^{*}\|^{2}$
can be absorbed to make a telescoping sum $\frac{b_{t-1}\left\Vert x_{t}-x^{*}\right\Vert ^{2}}{2b_{t-1}-b_{0}}-\frac{b_{t}\left\Vert x_{t+1}-x^{*}\right\Vert ^{2}}{2b_{t}-b_{0}}$.
The remaining terms are free of $x^{*}$; hence can be more easily
bounded. We can obtain
\begin{align*}
\E\left[\frac{F(x_{t})-F^{*}}{2b_{t}-b_{0}}\right] & \leq\E\left[Z_{t}\left\Vert \xi_{t}\right\Vert ^{2}+\frac{b_{t-1}\left\Vert x_{t}-x^{*}\right\Vert ^{2}}{2\gamma\eta(2b_{t-1}-b_{0})}-\frac{b_{t}\left\Vert x_{t+1}-x^{*}\right\Vert ^{2}}{2\gamma\eta(2b_{t}-b_{0})}+\frac{\eta\left\Vert \widehat{\nabla}F(x_{t})\right\Vert ^{2}}{2\gamma b_{t}^{2}}\right],
\end{align*}
where $Z_{t}=\frac{\eta}{\gamma b_{0}}\left(\frac{1}{2b_{t-1}-b_{0}}-\frac{1}{2b_{t}-b_{0}}\right)$.
Now we have a telescoping sum $\frac{b_{t-1}\left\Vert x_{t}-x^{*}\right\Vert ^{2}}{2\gamma\eta(2b_{t-1}-b_{0})}-\frac{b_{t}\left\Vert x_{t+1}-x^{*}\right\Vert ^{2}}{2\gamma\eta(2b_{t}-b_{0})}$.
Taking the sum over $t$, we have 
\begin{align*}
\E\left[\sum_{t=1}^{T}\frac{F(x_{t})-F^{*}}{2b_{t}}\right] & \le\E\left[\sum_{t=1}^{T}\frac{F(x_{t})-F^{*}}{2b_{t}-b_{0}}\right]\\
 & \le\frac{\left\Vert x_{1}-x^{*}\right\Vert ^{2}}{2\gamma\eta}+\E\left[\sum_{t=1}^{T}Z_{t}\left\Vert \xi_{t}\right\Vert ^{2}\right]+\E\left[\sum_{t=1}^{T}\frac{\eta\left\Vert \widehat{\nabla}F(x_{t})\right\Vert ^{2}}{2\gamma b_{t}^{2}}\right].
\end{align*}
Proceeding to bound each term, we will obtain Lemma \ref{lem:Main-function-gap-over-bT}.
\end{proof}

We emphasize the following crucial aspect of Lemma \ref{lem:Main-function-gap-over-bT}:
the inequality gives us a relationship between the function gap and
the stepsize $b_{T}$, which we know how to bound with high probability
under Assumptions 2' and 4. On the other hand, this relationship
is not ideal due to the fact that on the L.H.S. of (\ref{eq:lemma-gap-bt}),
we have not obtained a decoupling between the function gap and $b_{T}$.
To this end, we introduce the second novel technique. Let $\Delta_{T}\coloneqq\sum_{t=1}^{T}F(x_{t})-F(x^{*})$,
we write 
\[
\Delta_{T}=\Delta_{T}\mathds{1}_{E(\delta)}+\Delta_{T}\mathds{1}_{E^{c}(\delta)}
\]
where we define the event $E(\delta)=\left\{ M_{T}\leq\sigma^{2}\log^{2\theta}\frac{eT}{\delta}\right\} $.
For the first term, when $E(\delta)$ happens, we also know from Lemma
\ref{lem:Main-AdaGradNorm-Stoc-b_T} that the stepsize is bounded.
Thus we can bound 
\[
\E\left[\Delta_{T}\mathds{1}_{E(\delta)}\right]=\E\left[\frac{\sum_{t=1}^{T}F(x_{t})-F(x^{*})}{b_{T}}b_{T}\mathds{1}_{E(\delta)}\right]
\]
which leads us back to Lemma \ref{lem:Main-function-gap-over-bT}.
We can bound the second term using a tail bound for the event $E^{c}(\delta)$,
knowing from the first observation that $\Pr\left[E^{c}(\delta)\right]\le\delta$.
From this insight, and using the self-bounding argument as in the
proof of Theorem \ref{thm:Main-AdaGradNorm-rate}, we finally obtain
the following result.
\begin{thm}
\label{thm:Main-AdaGradNorm-Stoc-rate}Suppose $F$ satisfies Assumptions
1, 2', 3 and 4, Stochastic AdaGradNorm (Algorithm \ref{alg:AdaGradNorm-stochastic})
admits
\[
\E\left[\frac{\sum_{t=1}^{T}F(x_{t})-F(x^{*})}{T}\right]=O\left(\left(1+\poly\left(\sigma^{2}\log^{2\theta}T,\log(1+\sigma^{2}T\log^{2\theta}T)\right)\right)\left(\frac{1}{T}+\frac{\sigma\log^{\theta}T}{\sqrt{T}}\right)\right).
\]
\end{thm}
\begin{rem}
In the big-$O$ notation, we only show the dependency on $\sigma,T$
and $\theta$ for simplicity. The dependency on the other parameters
will be made explicit in the proof of the theorem. By setting $\sigma=0$,
we obtain the standard convergence rate $\E\big[(\sum_{t=1}^{T}F(x_{t})-F(x^{*}))/T\big]=O(1/T)$
as shown in Section \ref{sec:Main-AdaGradNorm} for the deterministic
setting. This means our analysis adapts to the noise parameter $\sigma$.
\end{rem}
Finally, it is worth pointing out that even when we relax Assumption
2' to Assumption 2, we can still provide a convergence guarantee for
Stochastic AdaGradNorm. We present the result in Theorem \ref{thm:Appendix-Stoc-AdaGradNorm-weak-rate}
in the appendix.

\section{Last iterate convergence of variants of AdaGradNorm for $\gamma$-quasar
convex and smooth minimization on $\protect\R^{d}$ \label{sec:Main-Last}}

In Section \ref{sec:Main-AdaGrad}, under Assumptions 1 and 2, we
proved that the average iterate produced by AdaGradNorm converges
at the $1/T$ rate, i.e., $\left(\sum_{t=1}^{T}F(x_{t})-F^{*}\right)/T=O(1/T)$.
A natural question is whether there exists an adaptive algorithm that
can guarantee the convergence of the last iterate. In this section,
we give an affirmative answer by presenting two simple variants of
AdaGradNorm and show convergence of the last iterate under Assumptions
1 and 2'. 

\begin{figure*}[t]

\begin{minipage}[t]{0.475\columnwidth}%
\begin{algorithm}[H]
\caption{AdaGradNorm-Last}
\label{alg:AdaGradNorm-Last1}

Initialize: $x_{1},\eta>0,\Delta>0,p_{t}>0$

for $t=1$ to $T$

$\quad$$b_{t}=\left(b_{0}^{2+\Delta}+\sum_{i=1}^{t}\frac{\|\nabla F(x_{i})\|^{2}}{p_{i}}\right)^{\frac{1}{2+\Delta}}$

$\quad$$x_{t+1}=x_{t}-\frac{\eta}{b_{t}}\nabla F(x_{t})$
\end{algorithm}
\end{minipage}\hfill{}%
\begin{minipage}[t]{0.475\columnwidth}%
\begin{algorithm}[H]
\caption{AdaGradNorm-Last}

\label{alg:AdaGradNorm-Last2}

Initialize: $x_{1},\eta>0,\delta\in[2/3,1),p_{t}>0$

for $t=1$ to $T$

$\quad$$b_{t}=\left(b_{0}^{2}+\sum_{i=1}^{t}\frac{\|\nabla F(x_{i})\|^{2}}{p_{i}}\right)^{\frac{1}{2}}$

$\quad$$x_{t+1}=x_{t}-\frac{\eta}{b_{t}^{\delta}b_{t-1}^{1-\delta}}\nabla F(x_{t})$
\end{algorithm}
\end{minipage}

\end{figure*}

In Algorithm \ref{alg:AdaGradNorm-Last1}, by setting $p_{i}=i^{-1}$,
$\|\nabla F(x_{i})\|^{2}$ has a bigger coefficient than in the standard
AdaGradNorm. Should we use the $\frac{1}{2}$-power ($\Delta=0$)
instead of $\frac{1}{2+\Delta}$ with $\Delta>0$, $b_{t}$ will grow
faster compared with the same term in AdaGradNorm. We will see later
that $\Delta=0$ still leads to the convergence of the last iterate.
However, we first focus on the easier case with $\Delta>0$ and state
convergence rate of Algorithm \ref{alg:AdaGradNorm-Last1} in Theorem
\ref{thm:Main-AdaGradNorm-Last1-rate}.
\begin{thm}
\label{thm:Main-AdaGradNorm-Last1-rate}With Assumptions 1 and 2',
by taking $p_{t}=\frac{1}{t}$ in Algorithm \ref{alg:AdaGradNorm-Last1},
we have
\[
F(x_{T+1})-F^{*}\leq\frac{\left(\frac{2}{\eta}\left(\frac{\|x_{1}-x^{*}\|^{2}}{\gamma\eta}+h(\Delta)+g(\Delta)\right)+b_{0}^{\Delta}\right)^{\frac{1}{\Delta}}\left(\frac{\|x_{1}-x^{*}\|^{2}}{\gamma\eta}+h(\Delta)+g(\Delta)\right)}{T}
\]
where
\begin{align*}
h(\Delta) & \coloneqq\begin{cases}
\frac{\left(2+\Delta\right)\eta\left(\eta L\right)^{\Delta}}{2}\log^{+}\frac{\eta L}{b_{0}} & \Delta\geq1\\
\frac{\left(2+\Delta\right)\eta^{2}L}{2b_{0}^{1-\Delta}}\log^{+}\frac{\eta L}{b_{0}} & \Delta\in\left(0,1\right)
\end{cases}\;\text{and}\; & g(\Delta)\coloneqq\frac{(2+\Delta)\eta}{\gamma}\left(\frac{2\eta L}{\gamma}\right)^{\Delta}\log^{+}\frac{2\eta L}{\gamma b_{0}}.
\end{align*}
\end{thm}
An issue with variant \ref{alg:AdaGradNorm-Last1} is that, when using
$\frac{1}{2+\Delta}$-power, the stepsize ceases to be scale-invariant.
Algorithm \ref{alg:AdaGradNorm-Last2} shows a different approach,
using the scale-invariant power $\frac{1}{2}$, but a different stepsize
$b_{t}^{\delta}b_{t-1}^{1-\delta}$, for a constant $\delta\in[2/3,1)$.
The tradeoff is that the provable convergence rate of the second variant
depends exponentially on the smoothness parameter. We also note that,
when $\delta=1$, we obtain the same algorithm as when setting $\Delta=0$
in the previous variant. 
\begin{rem}
\label{rem:Main-AdaGradNorm-Last2-remark}$b_{0}$ in every algorithm
is only for stabilization and is set to a constant that is very close
to $0$ in practice. However, the first stepsize in Algorithm \ref{alg:AdaGradNorm-Last2},
i.e., $b_{1}^{\delta}b_{0}^{1-\delta}$ will explode. To avoid this
issue, we can simply set the first stepsize as $b_{1}$ instead of
$b_{1}^{\delta}b_{0}^{1-\delta}$. We note that, under this change,
Algorithm \ref{alg:AdaGradNorm-Last2} still admits a provable convergence
rate. However, for simplicity, we keep $b_{1}^{\delta}b_{0}^{1-\delta}$
in both the description of the algorithm and its analysis.
\end{rem}
\begin{thm}
\label{thm:Main-AdaGradNorm-Last2-rate}With Assumptions 1 and 2',
by taking $p_{t}=\frac{1}{t}$ in Algorithm \ref{alg:AdaGradNorm-Last2},
we have
\[
F(x_{T+1})-F^{*}\leq\frac{\eta b_{0}\exp\left(\frac{k(\delta)}{1-\delta}\right)k(\delta)}{T},
\]
where $k(\delta)=\frac{\|x_{1}-x^{*}\|^{2}}{\gamma\eta^{2}}+\frac{\eta L}{b_{0}}\left(1-\left(\frac{b_{0}}{\eta L}\right)^{\frac{1}{\delta}}\right)^{+}+\frac{2}{\gamma\delta}\left(\frac{2\eta L}{\gamma b_{0}}\right)^{\frac{2}{\delta}-2}\log^{+}\frac{2\eta L}{\gamma b_{0}}$.
\end{thm}
To finish this section, we briefly discuss the case when $\Delta=0$
in Algorithm \ref{alg:AdaGradNorm-Last1} or equivalently $\delta=1$
in Algorithm \ref{alg:AdaGradNorm-Last2}. First, by seeing $\Delta$
tends to $0$, we can expect a convergence rate depending exponentially
on the problem parameters. When $\Delta=0$, while we can still expect
a bound of the function gap via the final stepsize $b_{T}$, bounding
$b_{T}$ becomes problematic. In the proof of Theorem \ref{thm:Main-AdaGradNorm-Last1-rate},
to bound $b_{T}$, we use the sum $\sum_{t=1}^{T}\frac{\|\nabla F(x_{t})\|^{2}}{b_{t}^{2}p_{t}}=\sum_{t=1}^{T}\frac{b_{t}^{2+\Delta}-b_{t-1}^{2+\Delta}}{b_{t}^{2}}$.
This sum only admits a lower bound in terms of $b_{T}$ when $\Delta>0$,
thus the argument does not work when $\Delta=0$. However, it is still
possible to give an asymptotic rate under the $\gamma$-quasar convexity
assumption. If we further assume that $F$ is convex, we can give
a non-asymptotic rate. The main idea on how to bound $b_{T}$ is as
follows. Let $\tau$ be the last time such that $b_{t}\leq\eta L/2$.
The increment from $b_{\tau+1}$ to $b_{T}$ can be bounded by observing
that the increase in each step $\|\nabla F(x_{t})\|^{2}\leq\frac{2}{3}p_{t}b_{t}^{2}$.
Moreover, the critical step is the increase from $b_{\tau}$ to $b_{\tau+1}$,
which again can be analyzed via the function gap and smoothness. We
present the asymptotic and a non-symptotic convergence rate and their
analysis in Sections \ref{subsec:Appendix-Last-Delta=00003D0-asy}
and \ref{subsec:Appendix-Last-Delta=00003D0-non-asy} in the appendix.

\section{Accelerated variants of AdaGradNorm for convex and smooth minimization
on $\protect\R^{d}$ \label{sec:Main-Acc}}

In this section, by using the stronger Assumption 1', we give two
new algorithms that achieve the accelerated rate $O(1/T^{2})$, matching
the optimal rate in $T$ for convex and smooth optimization for unconstrained
deterministic problems. Our new algorithms are adapted from the acceleration
scheme introduced in \cite{auslender2006interior} (see also \cite{lan2020first}).
They are also similar to existing adaptive accelerated methods designed
for bounded domains, including \cite{levy2018online,ene2021adaptive}.
However, previous analysis does not apply in unconstrained problems;
we therefore have to make necessary modifications. 

To the best of our knowledge, in unconstrained problems under the
deterministic setting, the only existing analysis for an accelerated
method was introduced in \cite{antonakopoulos2022undergrad}. Here
we discuss some limitations of this work. The convergence rate for
the weighted average iterate $\avx_{T+1/2}$ is given by
\[
f(\avx_{T+1/2})-f(x)\le O\left(\frac{1}{T^{2}}\left(R_{h}\lim_{t\to\infty}b_{t}+K_{h}\lim_{t\to\infty}b_{t}^{2}\right)\right)
\]
where $h$ is a $K_{h}$-strongly convex mirror map function, $R_{h}=\max h(x)-\min h(x)$
is the range of $h$. This result is only applicable when the domain
is unbounded but the range of the mirror map is bounded. Even in the
standard $\ell_{2}$ setup with $h(x)=\frac{1}{2}\|x\|^{2}$, this
assumption does not hold. Moreover, due to the term $\lim_{t\to\infty}b_{t}$,
the above guarantee is dependent on the particular function. Thus,
while a standard convergence guarantee is applicable to say, all SVM
models with Huber loss, the above guarantee varies for each SVM model
and there is no universal bound for all of them.

We further highlight some key differences between this work and ours.
While the convergence rate above depends on the convergence of the
stepsize, for both our variants, we will show an explicit convergence
rate that holds universally for the entire function class. Second,
the algorithm in \cite{antonakopoulos2022undergrad} is based on an
extra gradient method which requires to calculate gradients twice
in one iteration. Instead, our algorithms only need one gradient computation
per iteration. Finally, our algorithms guarantee the convergence of
the last iterate as opposed to that for the weighted average iterate
as shown above. 

\begin{figure*}[t]

\begin{minipage}[t]{0.475\columnwidth}%
\begin{algorithm}[H]
\caption{AdaGradNorm-Acc}
\label{alg:AdaGradNorm-Acc1}

Initialize: $x_{1}=w_{1},\eta>0,\Delta>0,a_{t}>0,q_{t}>0$

for $t=1$ to $T$

$\quad$$v_{t}=(1-a_{t})w_{t}+a_{t}x_{t}$

$\quad$$b_{t}=\left(b_{0}^{2+\Delta}+\sum_{i=1}^{t}\frac{\|\nabla F(v_{i})\|^{2}}{q_{i}^{2}}\right)^{\frac{1}{2+\Delta}}$

$\quad$$x_{t+1}=x_{t}-\frac{\eta}{q_{t}b_{t}}\nabla F(v_{t})$

$\quad$$w_{t+1}=(1-a_{t})w_{t}+a_{t}x_{t+1}$
\end{algorithm}
\end{minipage}\hfill{}%
\begin{minipage}[t]{0.475\columnwidth}%
\begin{algorithm}[H]
\caption{AdaGradNorm-Acc}
\label{alg:AdaGradNorm-Acc2}

Initialize: $x_{1}=w_{1},\eta>0,\delta\in[2/3,1),a_{t}>0,q_{t}>0$

for $t=1$ to $T$

$\quad$$v_{t}=(1-a_{t})w_{t}+a_{t}x_{t}$

$\quad$$b_{t}=\left(b_{0}^{2}+\sum_{i=1}^{t}\frac{\|\nabla F(v_{i})\|^{2}}{q_{i}^{2}}\right)^{\frac{1}{2}}$

$\quad$$x_{t+1}=x_{t}-\frac{\eta}{q_{t}b_{t}^{\delta}b_{t-1}^{1-\delta}}\nabla F(v_{t})$

$\quad$$w_{t+1}=(1-a_{t})w_{t}+a_{t}x_{t+1}$
\end{algorithm}
\end{minipage}

\end{figure*}

Algorithm \ref{alg:AdaGradNorm-Acc1} shows the first variant. For
an accelerated method, the step size typically has the form $b_{t}=\left(b_{0}^{2}+\sum_{i=1}^{t}s_{i}\|\nabla F_{i}\|^{2}\right)^{\frac{1}{2}}$
where $\nabla F_{i}$ is the gradient evaluated at time $i$, and
$s_{i}=O(i^{2})$. However in order to be able to give an explicit
convergence rate, Algorithm \ref{alg:AdaGradNorm-Acc1} uses a smaller
$b_{t}$ with power $\frac{1}{2+\Delta}$, with $\Delta>0$. When
$\Delta=0$, we can only show an asymptotic convergence rate, similarly
to \cite{antonakopoulos2022undergrad}. We first focus on the case
when $\Delta>0$. In the appendix we will discuss the convergence
of the algorithm when $\Delta=0$. We have the following theorem.
\begin{thm}
\label{thm:Main-AdaGradNorm-Acc1-rate}Suppose $F$ satisfies Assumptions
1' and 2', let $a_{t}=\frac{2}{t+1}$, $q_{t}=\frac{2}{t}$ in Algorithm
\ref{alg:AdaGradNorm-Acc1}, then
\[
F(w_{T+1})-F^{*}\leq\frac{4}{T(T+1)}\left(\frac{2\left\Vert x^{*}-x_{1}\right\Vert ^{2}}{\eta^{2}}+\frac{4h(\Delta)}{\eta}+b_{0}^{\Delta}\right)^{\frac{1}{\Delta}}\left(\frac{\left\Vert x^{*}-x_{1}\right\Vert ^{2}}{2\eta}+h(\Delta)\right)
\]
where 
\[
h(\Delta)=\begin{cases}
\frac{(2+\Delta)\left(2\eta L\right)^{\Delta-1}L\eta^{2}}{2}\log^{+}\frac{2\eta L}{b_{0}} & \Delta\geq1\\
\frac{(2+\Delta)L\eta^{2}}{2b_{0}^{1-\Delta}}\log^{+}\frac{2\eta L}{b_{0}} & \Delta\in\left(0,1\right).
\end{cases}
\]
\end{thm}
Similarly to the second variant in the previous section, we also have
a scale-invariant accelerated algorithm, shown in Algorithm \ref{alg:AdaGradNorm-Acc2}
using power $\frac{1}{2}$ but a smaller stepsize $b_{t}^{\delta}b_{t-1}^{1-\delta}$.
This algorithm also has an exponential dependency on the problem parameters,
which is given in the following theorem.
\begin{rem}
Similar to Remark \ref{rem:Main-AdaGradNorm-Last2-remark}, the first
stepsize in Algorithm \ref{alg:AdaGradNorm-Acc2}, i.e., $b_{1}^{\delta}b_{0}^{1-\delta}$
can be replaced by $b_{1}$. However, for simplicity, we keep $b_{1}^{\delta}b_{0}^{1-\delta}$
in both the description of the algorithm and its analysis.
\end{rem}
\begin{thm}
\label{thm:Main-AdaGradNorm-Acc2-rate}Suppose $F$ satisfies Assumptions
1' and 2', let $a_{t}=\frac{2}{t+1}$, $q_{t}=\frac{2}{t}$ in Algorithm
\ref{alg:AdaGradNorm-Acc2}, then
\[
F(w_{T+1})-F^{*}\le\frac{4\eta b_{0}\exp\left(\frac{2s(\delta)}{1-\delta}\right)s(\delta)}{T(T+1)},
\]
where $s(\delta)=\frac{\left\Vert x^{*}-x_{1}\right\Vert ^{2}}{2\eta^{2}}+\frac{\eta L}{b_{0}}\left(1-\left(\frac{b_{0}}{2\eta L}\right)^{\frac{1}{\delta}}\right)^{+}$.
\end{thm}
Similarly to the previous section, we give a more detailed discussion
of the convergence of the Algorithm \ref{alg:AdaGradNorm-Acc1} when
$\Delta=0$ or equivalently Algorithm \ref{alg:AdaGradNorm-Acc2}
when $\delta=1$ in Section \ref{subsec:Appendix-acc-Delta=00003D0}
in the appendix. While we can still show an accelerated $O(1/T^{2})$
asymptotic convergence rate, we only present an $O(1/T^{2}+1/T)$
non-asymptotic rate. The difference between these algorithms and the
ones in the previous section is that the stepsize $b_{t}$ increases
much faster. More precisely, the increment in each step is now $O(t^{2}\|\nabla F(v_{t})\|^{2})$
instead of $O(t\|\nabla F(x_{t})\|^{2})$. Thus we can only show an
upperbound for $b_{t}$ that grows linearly with time, which leads
to the $O(1/T^{2}+1/T)$ convergence rate.

\section{Conclusion and Future Work}

In this paper, we go back to the most basic AdaGrad algorithm and
study its convergence rate in generalized smooth convex optimization.
We prove explicit convergence guarantees for unconstrained problems
in both the deterministic and stochastic setting. Building on these
insights, we propose new algorithms that exhibit last iterate convergence,
with and without acceleration. We see our work as primarily theoretical
since the first and foremost goal is to understand properties of existing
algorithms that work well in practice. We refer the reader to the
long line of previous works \citep{duchi2011adaptive, levy2017online, kavis2019unixgrad,bach2019universal, antonakopoulos2020adaptive,ene2021adaptive,ene2022adaptive, antonakopoulos2022undergrad}
 that have already demonstrated the behavior of AdaGrad and accelerated
adaptive algorithms empirically.

\clearpage

\section*{Acknowledgments}

TN and AE were supported in part by NSF CAREER grant CCF-1750333,
NSF grant III-1908510, and an Alfred P. Sloan Research Fellowship.
HN was supported in part by NSF CAREER grant CCF-1750716 and NSF grant
CCF-1909314.

The authors would like to thank Thien Hang Nguyen for contributing
his ideas during our discussions.

\paragraph{Reproducibility Statement.}

We include the full proofs of all theorems in the Appendix.

\bibliographystyle{iclr2023_conference}
\bibliography{ref}

\clearpage
\appendix

\section{Missing proofs from Section \ref{sec:Main-AdaGrad} \label{sec:Appendix-adagrad}}

\subsection{AdaGradNorm\label{subsec:Appendix-AdaGradNorm}}

As we pointed out before, it is possible to obtain an improvement
by a factor $1/\gamma$ compared with Theorem \ref{thm:Main-AdaGradNorm-rate}
by assuming $L$-smoothness instead of weak $L$-smoothness.
\begin{thm}
\label{thm:Appendix-improved-AdaGradNorm-rate}With Assumptions 1
and 2', AdaGradNorm (Algorithm \ref{alg:AdaGradNorm}) admits
\[
\frac{\sum_{t=1}^{T}F(x_{t})-F^{*}}{T}\leq\frac{\left(\frac{L\|x_{1}-x^{*}\|^{2}}{\eta}+2\eta L\log^{+}\frac{\eta L}{b_{0}}+b_{0}\right)\left(\frac{\|x_{1}-x^{*}\|^{2}}{\gamma\eta}+\frac{2\eta}{\gamma}\log^{+}\frac{2\eta L}{b_{0}}\right)}{T}.
\]
\end{thm}
\begin{proof}
Note that Assumption 2' can imply Assumption 2, so following the same
proof of Theorem \ref{thm:Main-AdaGradNorm-rate}, we still have
\begin{align*}
\sum_{t=1}^{T}F(x_{t})-F^{*} & \leq b_{T}\left(\frac{\|x_{1}-x^{*}\|^{2}}{\gamma\eta}+\frac{2\eta}{\gamma}\log^{+}\frac{2\eta L}{\gamma b_{0}}\right).
\end{align*}

However, from here, we will bound $b_{T}$ directly, rathe than use
the self bounded argument in the previous proof. By the $L$-smoothness,
we know
\begin{align*}
F(x_{t+1})-F(x_{t}) & \leq\langle\nabla F(x_{t}),x_{t+1}-x_{t}\rangle+\frac{L}{2}\|x_{t+1}-x_{t}\|^{2}\\
 & =\left(\frac{L\eta^{2}}{2b_{t}^{2}}-\frac{\eta}{b_{t}}\right)\|\nabla F(x_{t})\|^{2}\\
\Rightarrow\frac{\|\nabla F(x_{t})\|^{2}}{b_{t}} & \leq\frac{2\left(F(x_{t})-F(x_{t+1})\right)}{\eta}+\left(\frac{L\eta}{b_{t}^{2}}-\frac{1}{b_{t}}\right)\|\nabla F(x_{t})\|^{2}.
\end{align*}
Sum up from $1$ to $T$, we know
\begin{align*}
\sum_{t=1}^{T}\frac{\|\nabla F(x_{t})\|^{2}}{b_{t}} & \leq\frac{2}{\eta}\left(F(x_{1})-F(x_{T+1})\right)+\sum_{t=1}^{T}\left(\frac{L\eta}{b_{t}^{2}}-\frac{1}{b_{t}}\right)\|\nabla F(x_{t})\|^{2}\\
 & \leq\frac{2}{\eta}\left(F(x_{1})-F(x^{*})\right)+\sum_{t=1}^{T}\left(\frac{L\eta}{b_{t}^{2}}-\frac{1}{b_{t}}\right)\|\nabla F(x_{t})\|^{2}.
\end{align*}
Use the the same proof technique as before, we can bound 
\[
\sum_{t=1}^{T}\left(\frac{L\eta}{b_{t}^{2}}-\frac{1}{b_{t}}\right)\|\nabla F(x_{t})\|^{2}\leq2\eta L\log^{+}\frac{\eta L}{b_{0}}.
\]
and
\[
\sum_{t=1}^{T}\frac{\|\nabla F(x_{t})\|^{2}}{b_{t}}=\sum_{t=1}^{T}\frac{b_{t}^{2}-b_{t-1}^{2}}{b_{t}}\geq\sum_{t=1}^{T}b_{t}-b_{t-1}=b_{T}-b_{0}.
\]
Hence, we know
\begin{align*}
b_{T} & \leq\frac{2}{\eta}\left(F(x_{1})-F(x^{*})\right)+2\eta L\log^{+}\frac{\eta L}{b_{0}}+b_{0}\\
 & \leq\frac{L\|x_{1}-x^{*}\|^{2}}{\eta}+2\eta L\log^{+}\frac{\eta L}{b_{0}}+b_{0}.
\end{align*}
By using this bound on $b_{T}$, we can get the final result with
an improvement by a factor $1/\gamma$.
\end{proof}

\subsection{Stochastic AdaGradNorm}

We will employ the following notations for convenience
\begin{align*}
\Delta_{t} & \coloneqq\sum_{s=1}^{t}F(x_{s})-F^{*};\\
\xi_{t} & \coloneqq\widehat{\nabla}F(x_{t})-\nabla F(x_{t});\\
M_{t} & \coloneqq\max_{s\in[t]}\left\Vert \xi_{s}\right\Vert ^{2}.
\end{align*}
Before diving into the details of our proof, we first present some
technical results we will use in the proof of Theorem \ref{thm:Main-AdaGradNorm-Stoc-rate}. 

\subsubsection{Technical lemmas}

To start with, under Assumptions 1 and 3 only, we can obtain a bound
for a term close to our final goal $\Delta_{T}$.
\begin{lem}
\label{lem:Appendix-function-gap-over-bT} (Lemma \ref{lem:Main-function-gap-over-bT})
Suppose $F$ satisfies Assumptions 1 and 3, we have
\[
\E\left[\frac{\Delta_{T}}{b_{T}}\right]\le\frac{\|x_{1}-x^{*}\|^{2}}{\gamma\eta}+\frac{2\eta}{\gamma}\E\left[\frac{M_{T}}{b_{0}^{2}}+\log\frac{b_{T}}{b_{0}}\right].
\]
\end{lem}
\begin{proof}
We start by using the $\gamma$-quasar convexity of the function $F$
\begin{align*}
F(x_{t})-F^{*} & \leq\frac{\langle\nabla F(x_{t}),x_{t}-x^{*}\rangle}{\gamma}\\
 & =\frac{\langle\nabla F(x_{t})-\widehat{\nabla}F(x_{t}),x_{t}-x^{*}\rangle}{\gamma}+\frac{\langle\widehat{\nabla}F(x_{t}),x_{t}-x^{*}\rangle}{\gamma}\\
 & =\frac{\langle-\xi_{t},x_{t}-x^{*}\rangle}{\gamma}+\frac{b_{t}}{2\gamma\eta}\left(\left\Vert x_{t}-x^{*}\right\Vert ^{2}-\left\Vert x_{t+1}-x^{*}\right\Vert ^{2}+\left\Vert x_{t+1}-x_{t}\right\Vert ^{2}\right).
\end{align*}
Dividing both sides by $2b_{t}-b_{0}$ and taking expactations, we
have
\begin{align*}
\E\left[\frac{F(x_{t})-F^{*}}{2b_{t}-b_{0}}\right] & \leq\E\left[\frac{\langle-\xi_{t},x_{t}-x^{*}\rangle}{\gamma(2b_{t}-b_{0})}+\frac{b_{t}}{2b_{t}-b_{0}}\times\frac{\left\Vert x_{t}-x^{*}\right\Vert ^{2}-\left\Vert x_{t+1}-x^{*}\right\Vert ^{2}+\left\Vert x_{t+1}-x_{t}\right\Vert ^{2}}{2\gamma\eta}\right].
\end{align*}
Now we no longer have a telescoping sum in the R.H.S.. However, this
is exactly what we want to handle the difficult term $\frac{\langle-\xi_{t},x_{t}-x^{*}\rangle}{\gamma(2b_{t}-b_{0})}$
which does not disappear after taking the expectation. The key step
is to use Cauchy-Schwarz inequality for the term 
\begin{align*}
\left|\langle-\xi_{t},x_{t}-x^{*}\rangle\right| & \le\frac{\lambda}{2}\left\Vert \xi_{t}\right\Vert ^{2}+\frac{1}{2\lambda}\left\Vert x_{t}-x^{*}\right\Vert ^{2}
\end{align*}
with the appropriate coefficient $\lambda$ so that the term $\left\Vert x_{t}-x^{*}\right\Vert ^{2}$
can be absorbed to make a telescoping sum $\frac{b_{t-1}\left\Vert x_{t}-x^{*}\right\Vert ^{2}}{2b_{t-1}-b_{0}}-\frac{b_{t}\left\Vert x_{t+1}-x^{*}\right\Vert ^{2}}{2b_{t}-b_{0}}$.
The remaining terms are free of $x^{*}$; hence can be more easily
bounded. To do this, note that 
\begin{align*}
 & \E\left[\frac{\langle-\xi_{t},x_{t}-x^{*}\rangle}{\gamma(2b_{t}-b_{0})}\right]\\
= & \E\left[\underbrace{\frac{1}{\gamma}\left(\frac{1}{2b_{t}-b_{0}}-\frac{1}{2b_{t-1}-b_{0}}\right)}_{A}\langle-\xi_{t},x_{t}-x^{*}\rangle\right]\\
\leq & \E\left[|A|\left|\langle-\xi_{t},x_{t}-x^{*}\rangle\right|\right]\\
\le & \E\left[\left|A\right|\left(\frac{\left|A\right|}{4}\left(\frac{b_{t-1}}{2\gamma\eta(2b_{t-1}-b_{0})}-\frac{b_{t}}{2\gamma\eta(2b_{t}-b_{0})}\right)^{-1}\left\Vert \xi_{t}\right\Vert ^{2}\right.\right.\\
 & \qquad\left.\left.+\left|A\right|^{-1}\left(\frac{b_{t-1}}{2\gamma\eta(2b_{t-1}-b_{0})}-\frac{b_{t}}{2\gamma\eta(2b_{t}-b_{0})}\right)\left\Vert x_{t}-x^{*}\right\Vert ^{2}\right)\right]\\
= & \E\left[\frac{\eta}{\gamma b_{0}}\left(\frac{1}{2b_{t-1}-b_{0}}-\frac{1}{2b_{t}-b_{0}}\right)\left\Vert \xi_{t}\right\Vert ^{2}\right]\\
 & \qquad+\E\left[\left(\frac{b_{t-1}}{2b_{t-1}-b_{0}}-\frac{b_{t}}{2b_{t}-b_{0}}\right)\frac{\left\Vert x_{t}-x^{*}\right\Vert ^{2}}{2\gamma\eta}\right]
\end{align*}
Thus we have
\begin{align*}
 & \E\left[\frac{F(x_{t})-F^{*}}{2b_{t}-b_{0}}\right]\\
\le & \E\left[\frac{\eta}{\gamma b_{0}}\left(\frac{1}{2b_{t-1}-b_{0}}-\frac{1}{2b_{t}-b_{0}}\right)\left\Vert \xi_{t}\right\Vert ^{2}\right]\\
 & +\E\left[\frac{b_{t-1}\left\Vert x_{t}-x^{*}\right\Vert ^{2}}{2\gamma\eta(2b_{t-1}-b_{0})}-\frac{b_{t}\left\Vert x_{t+1}-x^{*}\right\Vert ^{2}}{2\gamma\eta(2b_{t}-b_{0})}\right]+\E\left[\frac{\eta\left\Vert \widehat{\nabla}F(x_{t})\right\Vert ^{2}}{2\gamma b_{t}(2b_{t}-b_{0})}\right]\\
\le & \E\left[\frac{\eta}{\gamma b_{0}}\left(\frac{1}{2b_{t-1}-b_{0}}-\frac{1}{2b_{t}-b_{0}}\right)\left\Vert \xi_{t}\right\Vert ^{2}\right]\\
 & +\E\left[\frac{b_{t-1}\left\Vert x_{t}-x^{*}\right\Vert ^{2}}{2\gamma\eta(2b_{t-1}-b_{0})}-\frac{b_{t}\left\Vert x_{t+1}-x^{*}\right\Vert ^{2}}{2\gamma\eta(2b_{t}-b_{0})}\right]+\E\left[\frac{\eta\left\Vert \widehat{\nabla}F(x_{t})\right\Vert ^{2}}{2\gamma b_{t}^{2}}\right]
\end{align*}
Now we have a telescoping sum 
\[
\frac{b_{t-1}\left\Vert x_{t}-x^{*}\right\Vert ^{2}}{2\gamma\eta(2b_{t-1}-b_{0})}-\frac{b_{t}\left\Vert x_{t+1}-x^{*}\right\Vert ^{2}}{2\gamma\eta(2b_{t}-b_{0})}
\]
and the remaining terms are free of $x_{t}-x^{*}.$ Taking the sum
over $t$, we have 
\begin{align}
 & \E\left[\sum_{t=1}^{T}\frac{F(x_{t})-F^{*}}{2b_{t}-b_{0}}\right]\nonumber \\
\le & \E\left[\sum_{t=1}^{T}\frac{\eta}{2\gamma b_{0}}\left(\frac{1}{2b_{t}-b_{0}}-\frac{1}{2b_{t-1}-b_{0}}\right)\left\Vert \xi_{t}\right\Vert ^{2}\right]\nonumber \\
 & +\frac{\left\Vert x_{1}-x^{*}\right\Vert ^{2}}{2\gamma\eta}+\E\left[\sum_{t=1}^{T}\frac{\eta\|\widehat{\nabla}F(x_{t})\|^{2}}{2\gamma b_{t}^{2}}\right].\label{eq:Appendix-f-over-bT}
\end{align}
First for the easy term $\E\left[\sum_{t=1}^{T}\frac{\eta\|\widehat{\nabla}F(x_{t})\|^{2}}{2\gamma b_{t}^{2}}\right]$,
we have
\begin{align}
\E\left[\sum_{t=1}^{T}\frac{\eta\|\widehat{\nabla}F(x_{t})\|^{2}}{2\gamma b_{t}^{2}}\right] & =\frac{\eta}{2\gamma}\E\left[\sum_{t=1}^{T}\frac{b_{t}^{2}-b_{t-1}^{2}}{b_{t}^{2}}\right]\nonumber \\
 & \le\frac{\eta}{2\gamma}\E\left[\sum_{t=1}^{T}\log b_{t}^{2}-\log b_{t-1}^{2}\right]\nonumber \\
 & =\frac{\eta}{\gamma}\E\left[\log\frac{b_{T}}{b_{0}}\right].\label{eq:Appendix-f-over-bT-bound-1}
\end{align}
Next, we bound
\begin{align}
 & \E\left[\sum_{t=1}^{T}\frac{\eta}{\gamma b_{0}}\left(\frac{1}{2b_{t-1}-b_{0}}-\frac{1}{2b_{t}-b_{0}}\right)\left\Vert \xi_{t}\right\Vert ^{2}\right]\nonumber \\
\leq & \E\left[\sum_{t=1}^{T}\frac{\eta}{\gamma b_{0}}\left(\frac{1}{2b_{t-1}-b_{0}}-\frac{1}{2b_{t}-b_{0}}\right)M_{T}\right]\nonumber \\
\leq & \E\left[\frac{\eta}{\gamma b_{0}^{2}}M_{T}\right]\label{eq:Appendix-f-over-bT-bound-2}
\end{align}
Plugging the bounds (\ref{eq:Appendix-f-over-bT-bound-1}) and (\ref{eq:Appendix-f-over-bT-bound-2})
into (\ref{eq:Appendix-f-over-bT}), we have
\[
\E\left[\sum_{t=1}^{T}\frac{F(x_{t})-F^{*}}{2b_{t}}\right]\leq\E\left[\sum_{t=1}^{T}\frac{F(x_{t})-F^{*}}{2b_{t}-b_{0}}\right]\le\frac{\|x_{1}-x^{*}\|^{2}}{2\gamma\eta}+\frac{\eta}{\gamma b_{0}^{2}}\E\left[M_{T}\right]+\frac{\eta}{\gamma}\E\left[\log\frac{b_{T}}{b_{0}}\right].
\]
The last step is using $\sum_{t=1}^{T}\frac{F(x_{t})-F^{*}}{2b_{t}}\geq\frac{\sum_{t=1}^{T}F(x_{t})-F^{*}}{2b_{T}}=\frac{\Delta_{T}}{2b_{T}}$
to finish the proof.
\end{proof}

Due to the appearance of $M_{T}$ in Lemma \ref{lem:Appendix-function-gap-over-bT},
it is natural to consider what we can obtain under the additional
Assumption 4, i.e., sub-Weibull noise with parameter $\theta$. We
first provide the following simple bound on $\E\left[\|\xi_{t}\|^{2}\right]$.
The result is not new and the proof is only included for completeness.
\begin{lem}
\label{lem:Appendix-2nd-moments}Under Assumption 4, $\forall t\in[T]$,
we have
\[
\E\left[\|\xi_{t}\|^{2}\right]\leq\Gamma(2\theta+1)e\sigma^{2}.
\]
\end{lem}
\begin{proof}
We first note that from the definition of sub-Weibull noise, the tail
of $\|\xi_{t}\|$ can be bounded as follows
\[
\Pr\left[\|\xi_{t}\|\geq u\right]\leq\frac{\E\left[\exp\left((\|\xi_{t}\|/\sigma)^{1/\theta}\right)\right]}{\exp\left((u/\sigma)^{1/\theta}\right)}\leq\exp\left(1-(u/\sigma)^{1/\theta}\right).
\]
Then we can obtain
\begin{align*}
\E\left[\|\xi_{t}\|^{2}\right] & =\int_{0}^{\infty}2u\Pr\left[\|\xi\|\geq u\right]\text{d}u\\
 & \leq\int_{0}^{\infty}2u\exp\left(1-(u/\sigma)^{1/\theta}\right)\text{d}u\\
 & =2\theta e\sigma^{2}\int_{0}^{\infty}v^{2\theta-1}\exp(-v)\text{d}v\\
 & =\Gamma(2\theta+1)e\sigma^{2}
\end{align*}
where $u$ is substituted by $\sigma v^{\theta}$ in the second equation.
\end{proof}

Next, we prove a high probability bound on $M_{T}$, the proof of
which is inspired by Lemma 5 in \cite{li2020high}.
\begin{lem}
\label{lem:Appendix-LiOrabona-lemma}Under Assumption 4, given $0<\delta<1$,
define the event 
\[
E(\delta)=\left\{ M_{T}\leq\sigma^{2}\log^{2\theta}\frac{eT}{\delta}\right\} ,
\]
we have $\Pr\left[E(\delta)\right]\ge1-\delta$.
\end{lem}
\begin{proof}
Note that
\begin{align*}
\Pr\left[M_{T}\geq u\right] & =\Pr\left[\max_{s\in[T]}\|\xi_{s}\|^{2}\geq u\right]\\
 & =\Pr\left[\max_{s\in[T]}\|\xi_{s}\|^{\frac{1}{\theta}}\geq u^{\frac{1}{2\theta}}\right]\\
 & \leq\frac{\E\left[\exp\left(\max_{s\in[T]}(\|\xi_{s}\|/\sigma)^{1/\theta}\right)\right]}{\exp\left((u^{1/2}/\sigma)^{1/\theta}\right)}\\
 & \leq\frac{\sum_{s=1}^{T}\E\left[\exp\left((\|\xi_{s}\|/\sigma)^{1/\theta}\right)\right]}{\exp\left((u^{1/2}/\sigma)^{1/\theta}\right)}\\
 & =T\exp\left(1-(u^{1/2}/\sigma)^{1/\theta}\right).
\end{align*}
Choose $u=\sigma^{2}\log^{2\theta}\frac{eT}{\delta}$ to obtain
\[
\Pr\left[M_{T}\geq\sigma^{2}\log^{2\theta}\frac{eT}{\delta}\right]\leq\delta.
\]
\end{proof}

Lastly, we will find an upper bound on the $p$-th moment of $M_{T}$.
\begin{lem}
\label{lem:Appendix-max-xi-moment}Under Assumption 4, given $p>0$,
there is
\[
\E\left[M_{T}^{p}\right]\leq\sigma^{2p}\left(\log^{2\theta p}\left(\Gamma(4\theta p+1)e^{2}T^{2}\right)+1\right).
\]
\end{lem}
\begin{proof}
Note that in Lemma \ref{lem:Appendix-LiOrabona-lemma}, we proved
\[
\Pr\left[M_{T}\geq u\right]\leq T\exp\left(1-(u^{1/2}/\sigma)^{1/\theta}\right).
\]
Let $E(\delta)$ be the same as it in Lemma \ref{lem:Appendix-LiOrabona-lemma}.
Then, by Holder's inequality we have
\begin{align*}
\E\left[M_{T}^{p}\right] & =\E\left[M_{T}^{p}\mathds{1}_{E(\delta)}\right]+\E\left[M_{T}^{p}\mathds{1}_{E^{c}(\delta)}\right]\\
 & \le\E\left[M_{T}^{p}\mathds{1}_{E(\delta)}\right]+\sqrt{\E\left[M_{T}^{2p}\right]\E\left[\mathds{1}_{E^{c}(\delta)}\right]}\\
 & \leq\sigma^{2p}\log^{2\theta p}\frac{eT}{\delta}+\sqrt{\E\left[M_{T}^{2p}\right]\delta}\\
 & =\sigma^{2p}\log^{2\theta p}\frac{eT}{\delta}+\sqrt{\delta\int_{0}^{\infty}2pu^{2p-1}\Pr\left[M_{T}\geq u\right]\text{d}u}\\
 & \leq\sigma^{2p}\log^{2\theta p}\frac{eT}{\delta}+\sqrt{\delta\int_{0}^{\infty}2pu^{2p-1}T\exp\left(1-(u^{1/2}/\sigma)^{1/\theta}\right)\text{d}u}\\
 & =\sigma^{2p}\log^{2\theta p}\frac{eT}{\delta}+\sigma^{2p}\sqrt{\Gamma(4\theta p+1)eT\delta}\\
 & =\sigma^{2p}\left(\log^{2\theta p}\frac{eT}{\delta}+\sqrt{\Gamma(4\theta p+1)eT\delta}\right).
\end{align*}
Choose $\delta=\frac{1}{\Gamma(4\theta p+1)eT}<1$, we have
\[
\E\left[M_{T}^{p}\right]\leq\sigma^{2p}\left(\log^{2\theta p}\left(\Gamma(4\theta p+1)e^{2}T^{2}\right)+1\right).
\]
\end{proof}

Note that all the above results only depend on Assumptions 1, 3 and
4 without requiring the smoothness of $F$.

\subsubsection{Proof of Theorem \ref{thm:Main-AdaGradNorm-Stoc-rate}}

Theorem \ref{thm:Main-AdaGradNorm-Stoc-rate} requires Assumption
2' additionally. Thus we first show that under Assumptions 2' and
4. $b_{T}$ enjoys a $\widetilde{O}(1+\sigma\sqrt{T\log^{2\theta}T})$
upper bound with high probability.
\begin{lem}
\label{lem:Appendix-bound-bT} (Lemma \ref{lem:Main-AdaGradNorm-Stoc-b_T})
Suppose $F$ satisfies Assumptions 2'and 4. Under the event $E(\delta)=\left\{ M_{T}\leq\sigma^{2}\log^{2\theta}\frac{eT}{\delta}\right\} $,
we have
\[
b_{T}\le g_{T}(\delta)\coloneqq2b_{0}+\frac{4(F(x_{1})-F^{*})}{\eta}+4\eta L\log^{+}\frac{\eta L}{b_{0}}+4\sigma\sqrt{T\log^{2\theta}\frac{eT}{\delta}\log\left(1+\frac{16\sigma^{2}T\log^{2\theta}\frac{eT}{\delta}}{b_{0}^{2}}\right)}.
\]
Additionally, by Lemma \ref{lem:Appendix-LiOrabona-lemma}, there
is
\[
1-\delta\leq\Pr\left[E(\delta)\right]\leq\Pr\left[b_{T}\leq g_{T}(\delta)\right].
\]
\end{lem}
\begin{proof}
We start by using the smoothness of $F$
\begin{align*}
F(x_{t+1})-F(x_{t}) & \leq\langle\nabla F(x_{t}),x_{t+1}-x_{t}\rangle+\frac{L}{2}\|x_{t+1}-x_{t}\|^{2}\\
 & =-\frac{\eta}{b_{t}}\langle\nabla F(x_{t}),\widehat{\nabla}F(x_{t})\rangle+\frac{\eta^{2}L}{2b_{t}^{2}}\|\widehat{\nabla}F(x_{t})\|^{2}\\
 & =-\frac{\eta}{b_{t}}\langle\nabla F(x_{t})-\widehat{\nabla}F(x_{t}),\widehat{\nabla}F(x_{t})\rangle-\frac{\eta}{b_{t}}\|\widehat{\nabla}F(x_{t})\|^{2}+\frac{\eta^{2}L}{2b_{t}^{2}}\|\widehat{\nabla}F(x_{t})\|^{2}\\
\Rightarrow\frac{\|\widehat{\nabla}F(x_{t})\|^{2}}{b_{t}} & \leq\frac{2}{\eta}(F(x_{t})-F(x_{t+1}))+\frac{2\langle\xi_{t},\widehat{\nabla}F(x_{t})\rangle}{b_{t}}+\left(\frac{\eta L}{b_{t}^{2}}-\frac{1}{b_{t}}\right)\|\widehat{\nabla}F(x_{t})\|^{2}.
\end{align*}
Taking the sum over $t$ we have
\begin{align*}
\sum_{t=1}^{T}\frac{\|\widehat{\nabla}F(x_{t})\|^{2}}{b_{t}} & \le\frac{2(F(x_{1})-F^{*})}{\eta}+2\sum_{t=1}^{T}\frac{\langle\xi_{t},\widehat{\nabla}F(x_{t})\rangle}{b_{t}}+\sum_{t=1}^{T}\left(\frac{\eta L}{b_{t}^{2}}-\frac{1}{b_{t}}\right)\|\widehat{\nabla}F(x_{t})\|^{2}.
\end{align*}
Using the common technique, we know that $\sum_{t=1}^{T}\Big(\frac{\eta L}{b_{t}^{2}}-\frac{1}{b_{t}}\Big)\|\widehat{\nabla}F(x_{t})\|^{2}\le2\eta L\log^{+}\frac{\eta L}{b_{0}}$.
Moreover, for the L.H.S.
\begin{align*}
\sum_{t=1}^{T}\frac{\|\widehat{\nabla}F(x_{t})\|^{2}}{b_{t}} & =\sum_{t=1}^{T}\frac{b_{t}^{2}-b_{t-1}^{2}}{b_{t}}\ge\sum_{t=1}^{T}b_{t}-b_{t-1}=b_{T}-b_{0}.
\end{align*}
Thus we have
\begin{align*}
b_{T} & \le b_{0}+\frac{2(F(x_{1})-F^{*})}{\eta}+2\eta L\log^{+}\frac{\eta L}{b_{0}}+2\sum_{t=1}^{T}\frac{\langle\xi_{t},\widehat{\nabla}F(x_{t})\rangle}{b_{t}}
\end{align*}
For the last term in this equation, we notice that $\langle\xi_{t},\widehat{\nabla}F(x_{t})\rangle\le\|\xi_{t}\|\|\widehat{\nabla}F(x_{t})\|\le\sqrt{M_{T}}\|\widehat{\nabla}F(x_{t})\|$,
hence
\begin{align*}
b_{T} & \le\frac{2(F(x_{1})-F^{*})}{\eta}+2\eta L\log^{+}\frac{\eta L}{b_{0}}+b_{0}+2\sum_{t=1}^{T}\frac{\langle\xi_{t},\widehat{\nabla}F(x_{t})\rangle}{b_{t}}\\
 & \le\frac{2(F(x_{1})-F^{*})}{\eta}+2\eta L\log^{+}\frac{\eta L}{b_{0}}+b_{0}+2\sqrt{M_{T}}\sum_{t=1}^{T}\frac{\|\widehat{\nabla}F(x_{t})\|}{b_{t}}\\
 & \overset{(a)}{\le}\frac{2(F(x_{1})-F^{*})}{\eta}+2\eta L\log^{+}\frac{\eta L}{b_{0}}+b_{0}+2\sqrt{M_{T}}\sqrt{T\sum_{t=1}^{T}\frac{\|\widehat{\nabla}F(x_{t})\|^{2}}{b_{t}^{2}}}\\
 & =\frac{2(F(x_{1})-F^{*})}{\eta}+2\eta L\log^{+}\frac{\eta L}{b_{0}}+b_{0}+\sqrt{4M_{T}T\sum_{t=1}^{T}\frac{b_{t}^{2}-b_{t-1}^{2}}{b_{t}^{2}}}\\
 & \le\frac{2(F(x_{1})-F^{*})}{\eta}+2\eta L\log^{+}\frac{\eta L}{b_{0}}+b_{0}+\sqrt{4M_{T}T\log\frac{b_{T}^{2}}{b_{0}^{2}}}
\end{align*}
where $(a)$ is due to Jensen's inequality. We can write 
\begin{align*}
4M_{T}T\log\frac{b_{T}^{2}}{b_{0}^{2}} & =4M_{T}T\left(\log\frac{b_{T}^{2}}{b_{0}^{2}+16M_{T}T}+\log\frac{b_{0}^{2}+16M_{T}T}{b_{0}^{2}}\right)\\
 & \le4M_{T}T\left(\frac{b_{T}^{2}}{b_{0}^{2}+16M_{T}T}+\log\frac{b_{0}^{2}+16M_{T}T}{b_{0}^{2}}\right)\\
 & \le\frac{b_{T}^{2}}{4}+4M_{T}T\log\frac{b_{0}^{2}+16M_{T}T}{b_{0}^{2}}.
\end{align*}
Hence 
\begin{align*}
b_{T} & \le b_{0}+\frac{2(F(x_{1})-F^{*})}{\eta}+2\eta L\log^{+}\frac{\eta L}{b_{0}}+\sqrt{\frac{b_{T}^{2}}{4}+4M_{T}T\log\frac{b_{0}^{2}+16M_{T}T}{b_{0}^{2}}}\\
 & \le b_{0}+\frac{2(F(x_{1})-F^{*})}{\eta}+2\eta L\log^{+}\frac{\eta L}{b_{0}}+\frac{b_{T}}{2}+2\sqrt{M_{T}T\log\frac{b_{0}^{2}+16M_{T}T}{b_{0}^{2}}}
\end{align*}
which gives us 
\begin{align*}
b_{T} & \le2b_{0}+\frac{4(F(x_{1})-F^{*})}{\eta}+4\eta L\log^{+}\frac{\eta L}{b_{0}}+4\sqrt{M_{T}T\log\frac{b_{0}^{2}+16M_{T}T}{b_{0}^{2}}}.
\end{align*}
Recall the definition of the event $E(\delta)$ is $M_{T}\leq\sigma^{2}\log^{2\theta}\frac{eT}{\delta}$,
thus we know
\[
b_{T}\le2b_{0}+\frac{4(F(x_{1})-F^{*})}{\eta}+4\eta L\log^{+}\frac{\eta L}{b_{0}}+4\sigma\sqrt{T\left(\log^{2\theta}\frac{eT}{\delta}\right)\log\left(1+\frac{16\sigma^{2}T\log^{2\theta}\frac{eT}{\delta}}{b_{0}^{2}}\right)}.
\]
\end{proof}

By using Lemma \ref{lem:Appendix-bound-bT}, we can consider the following
decomposition
\begin{align*}
\E\left[\Delta_{T}\right] & =\E\left[\Delta_{T}\mathds{1}_{E(\delta)}\right]+\E\left[\Delta_{T}\mathds{1}_{E^{c}(\delta)}\right]\\
 & =\E\left[\frac{\Delta_{T}}{b_{T}}b_{T}\mathds{1}_{E(\delta)}\right]+\E\left[\Delta_{T}\mathds{1}_{E^{c}(\delta)}\right]\\
 & \leq g_{T}(\delta)\E\left[\frac{\Delta_{T}}{b_{T}}\mathds{1}_{E(\delta)}\right]+\E\left[\Delta_{T}\mathds{1}_{E^{c}(\delta)}\right]\\
 & \leq g_{T}(\delta)\E\left[\frac{\Delta_{T}}{b_{T}}\right]+\E\left[\Delta_{T}\mathds{1}_{E^{c}(\delta)}\right].
\end{align*}
Note that Lemma \ref{lem:Appendix-function-gap-over-bT} tells us
\[
\E\left[\frac{\Delta_{T}}{b_{T}}\right]\leq\frac{\|x_{1}-x^{*}\|^{2}}{\gamma\eta}+\frac{2\eta}{\gamma}\E\left[\frac{M_{T}}{b_{0}^{2}}+\log\frac{b_{T}}{b_{0}}\right].
\]
Hence our remaining task is to find a proper bound on $\E\left[\Delta_{T}\mathds{1}_{E^{c}(\delta)}\right]$,
which is stated in the following lemma.
\begin{lem}
\label{lem:Appendix-expectation-bound-2}Under Assumptions 2' and
4 we have
\begin{align*}
\E\left[\Delta_{T}\mathds{1}_{E^{c}(\delta)}\right] & \leq\left(F(x_{1})-F^{*}+\eta^{2}L\log^{+}\frac{\eta L}{2b_{0}}\right)T\delta+\eta\E^{1/4}\left[M_{T}^{2}\right]\sqrt{\log\E\left[\frac{b_{T}^{2}}{b_{0}^{2}}\right]}T^{3/2}\delta^{1/4}.
\end{align*}
\end{lem}
\begin{proof}
We restart from the smoothness of $F$:
\begin{align*}
F(x_{s+1})-F(x_{s}) & \le-\frac{\eta}{b_{s}}\langle\nabla F(x_{s})-\widehat{\nabla}F(x_{s}),\widehat{\nabla}F(x_{s})\rangle-\frac{\eta}{b_{s}}\|\widehat{\nabla}F(x_{s})\|^{2}+\frac{\eta^{2}L}{2b_{s}^{2}}\|\widehat{\nabla}F(x_{s})\|^{2}.
\end{align*}
Taking the sum over $s$, we have for $t\ge2$
\begin{align*}
F(x_{t})-F(x_{1}) & \le\sum_{s=1}^{t-1}-\frac{\eta}{b_{s}}\langle\nabla F(x_{s})-\widehat{\nabla}F(x_{s}),\widehat{\nabla}F(x_{s})\rangle+\sum_{s=1}^{t-1}\left(\frac{\eta^{2}L}{2b_{s}^{2}}-\frac{\eta}{b_{s}}\right)\|\widehat{\nabla}F(x_{s})\|^{2}\\
 & \le\eta^{2}L\log^{+}\frac{\eta L}{2b_{0}}+\sum_{s=1}^{t-1}\frac{\eta}{b_{s}}\|\xi_{s}\|\|\widehat{\nabla}F(x_{s})\|.
\end{align*}
Following the same proof of Lemma \ref{lem:Appendix-bound-bT}, we
have
\begin{align*}
F(x_{t})-F^{*} & \le F(x_{1})-F^{*}+\eta^{2}L\log^{+}\frac{\eta L}{2b_{0}}+\eta\sqrt{M_{t-1}(t-1)\log\frac{b_{t-1}^{2}}{b_{0}^{2}}}.
\end{align*}
Now we bound $\Delta_{T}$ as follows
\begin{align*}
\Delta_{T} & =\sum_{t=1}^{T}F(x_{t})-F^{*}\\
 & \leq F(x_{1})-F^{*}+\sum_{t=2}^{T}F(x_{1})-F^{*}+\eta^{2}L\log^{+}\frac{\eta L}{2b_{0}}+\eta\sqrt{M_{t-1}(t-1)\log\frac{b_{t-1}^{2}}{b_{0}^{2}}}\\
 & \leq\left(F(x_{1})-F^{*}+\eta^{2}L\log^{+}\frac{\eta L}{2b_{0}}\right)T+\sum_{t=2}^{T}\eta\sqrt{M_{t-1}(t-1)\log\frac{b_{t-1}^{2}}{b_{0}^{2}}}\\
 & \leq\left(F(x_{1})-F^{*}+\eta^{2}L\log^{+}\frac{\eta L}{2b_{0}}\right)T+\eta\sqrt{M_{T}\log\frac{b_{T}^{2}}{b_{0}^{2}}}T^{3/2}.
\end{align*}
Thus we obtain
\[
\E\left[\Delta_{T}\mathds{1}_{E^{c}(\delta)}\right]\leq\left(F(x_{1})-F^{*}+\eta^{2}L\log^{+}\frac{\eta L}{2b_{0}}\right)T\delta+\eta\E\left[\sqrt{M_{T}\log\frac{b_{T}^{2}}{b_{0}^{2}}}\mathds{1}_{E^{c}(\delta)}\right]T^{3/2}.
\]
Here we invoke Holder's inequality for three variables: for $p,q,r>0$,
$1/p+1/q+1/r=1$ then $\E[XYZ]\le\E^{1/p}[X^{p}]\E^{1/q}[Y^{q}]\E^{1/r}[Z^{r}]$.
By substituting $X=\sqrt{M_{T}}$, $Y=\sqrt{\log\frac{b_{T}^{2}}{b_{0}^{2}}}$,
$Z=\mathds{1}_{E^{c}(\delta)}$, and $p=4$, $q=2$, $r=4$, we have
\begin{align*}
\E\left[\sqrt{M_{T}\log\frac{b_{T}^{2}}{b_{0}^{2}}}\mathds{1}_{E^{c}(\delta)}\right] & \leq\E^{1/4}\left[M_{T}^{2}\right]\E^{1/2}\left[\log\frac{b_{T}^{2}}{b_{0}^{2}}\right]\E^{1/4}\left[\mathds{1}_{E^{c}(\delta)}\right]\\
 & \leq\E^{1/4}\left[M_{T}^{2}\right]\sqrt{\log\E\left[\frac{b_{T}^{2}}{b_{0}^{2}}\right]}\delta^{1/4}.
\end{align*}
So finally we get
\begin{align*}
\E\left[\Delta_{T}\mathds{1}_{E^{c}(\delta)}\right] & \leq\left(F(x_{1})-F^{*}+\eta^{2}L\log^{+}\frac{\eta L}{2b_{0}}\right)T\delta\\
 & \quad+\eta\E^{1/4}\left[M_{T}^{2}\right]\sqrt{\log\E\left[\frac{b_{T}^{2}}{b_{0}^{2}}\right]}T^{3/2}\delta^{1/4}.
\end{align*}
\end{proof}

\begin{lem}
\label{lem:Appendix-function-gap-final}Suppose $F$ satisfies Assumptions
1, 2', 3 and 4 then 
\begin{align*}
 & \E\left[\Delta_{T}\right]\\
\le & g_{T}\left(\frac{\|x_{1}-x^{*}\|^{2}}{2\gamma\eta}+\frac{2\eta\sigma^{2}\left(2^{(4\theta-1)\lor2\theta}\log^{2\theta}T+C_{1}\right)}{\gamma b_{0}^{2}}+\frac{\eta}{\gamma}\log\E\left[\frac{b_{T}^{2}}{b_{0}^{2}}\right]\right)\\
 & +\frac{F(x_{1})-F^{*}+\eta^{2}L\log^{+}\frac{\eta L}{2b_{0}}}{T^{3}}+\frac{\eta\sigma\left(2^{(2\theta-1)\lor\theta}\log^{\theta}T+C_{2}\right)}{2}\left(1+\log\E\left[\frac{b_{T}^{2}}{b_{0}^{2}}\right]\right)\sqrt{T}
\end{align*}
where $C_{1}=2^{(2\theta-1)^{+}}\log^{2\theta}\left(\Gamma(4\theta+1)e^{2}\right)+1$
and $C_{2}=2^{(\theta-1)^{+}}\log^{\theta}\left(\Gamma(8\theta+1)e^{2}\right)+1$
are two constants and
\begin{align*}
g_{T} & =2b_{0}+\frac{4(F(x_{1})-F^{*})}{\eta}+4\eta L\log^{+}\frac{\eta L}{b_{0}}\\
 & \quad+4\sigma\sqrt{T\log^{2\theta}(eT^{5})\log\left(1+\frac{16\sigma^{2}T\log^{2\theta}(eT^{5})}{b_{0}^{2}}\right)}.
\end{align*}
\end{lem}
\begin{proof}
As stated above, we know
\begin{align*}
\E\left[\Delta_{T}\right] & =\E\left[\Delta_{T}\mathds{1}_{E(\delta)}\right]+\E\left[\Delta_{T}\mathds{1}_{E^{c}(\delta)}\right]\\
 & =\E\left[\frac{\Delta_{T}}{b_{T}}b_{T}\mathds{1}_{E(\delta)}\right]+\E\left[\Delta_{T}\mathds{1}_{E^{c}(\delta)}\right]\\
 & \overset{(a)}{\leq}g_{T}(\delta)\E\left[\frac{\Delta_{T}}{b_{T}}\mathds{1}_{E(\delta)}\right]+\E\left[\Delta_{T}\mathds{1}_{E^{c}(\delta)}\right]\\
 & \leq g_{T}(\delta)\E\left[\frac{\Delta_{T}}{b_{T}}\right]+\E\left[\Delta_{T}\mathds{1}_{E^{c}(\delta)}\right]\\
 & \overset{(b)}{\leq}g_{T}(\delta)\left(\frac{\|x_{1}-x^{*}\|^{2}}{2\gamma\eta}+\frac{2\eta}{\gamma}\E\left[\frac{M_{T}}{b_{0}^{2}}+\log\frac{b_{T}}{b_{0}}\right]\right)+\E\left[\Delta_{T}\mathds{1}_{E^{c}(\delta)}\right]\\
 & \leq g_{T}(\delta)\left(\frac{\|x_{1}-x^{*}\|^{2}}{2\gamma\eta}+\frac{2\eta}{\gamma b_{0}^{2}}\E\left[M_{T}\right]+\frac{\eta}{\gamma}\log\E\left[\frac{b_{T}^{2}}{b_{0}^{2}}\right]\right)+\E\left[\Delta_{T}\mathds{1}_{E^{c}(\delta)}\right]
\end{align*}
where $(a)$ is due to Lemma \ref{lem:Appendix-bound-bT}. $(b)$
is by Lemma \ref{lem:Appendix-function-gap-over-bT}.

Lemma \ref{lem:Appendix-expectation-bound-2} gives us
\begin{align*}
\E\left[\Delta_{T}\mathds{1}_{E^{c}(\delta)}\right] & \leq\left(F(x_{1})-F^{*}+\eta^{2}L\log^{+}\frac{\eta L}{2b_{0}}\right)T\delta+\eta\E^{1/4}\left[M_{T}^{2}\right]\sqrt{\log\E\left[\frac{b_{T}^{2}}{b_{0}^{2}}\right]}T^{3/2}\delta^{1/4}.
\end{align*}
Pluggin in this bound, we have
\begin{align*}
 & \E\left[\Delta_{T}\right]\\
\leq & g_{T}(\delta)\left(\frac{\|x_{1}-x^{*}\|^{2}}{2\gamma\eta}+\frac{2\eta}{\gamma b_{0}^{2}}\E\left[M_{T}\right]+\frac{\eta}{\gamma}\log\E\left[\frac{b_{T}^{2}}{b_{0}^{2}}\right]\right)\\
 & +\left(F(x_{1})-F^{*}+\eta^{2}L\log^{+}\frac{\eta L}{2b_{0}}\right)T\delta+\eta\E^{1/4}\left[M_{T}^{2}\right]\sqrt{\log\E\left[\frac{b_{T}^{2}}{b_{0}^{2}}\right]}T^{3/2}\delta^{1/4}.
\end{align*}
Now we take $\delta=T^{-4}$ and let $g_{T}\coloneqq g_{T}(T^{-4})$
to obtain
\begin{align*}
 & \E\left[\Delta_{T}\right]\\
\leq & g_{T}\left(\frac{\|x_{1}-x^{*}\|^{2}}{2\gamma\eta}+\frac{2\eta}{\gamma b_{0}^{2}}\E\left[M_{T}\right]+\frac{\eta}{\gamma}\log\E\left[\frac{b_{T}^{2}}{b_{0}^{2}}\right]\right)\\
 & +\frac{1}{T^{3}}\left(F(x_{1})-F^{*}+\eta^{2}L\log^{+}\frac{\eta L}{2b_{0}}\right)+\eta\E^{1/4}\left[M_{T}^{2}\right]\sqrt{\log\E\left[\frac{b_{T}^{2}}{b_{0}^{2}}\right]}\sqrt{T}.
\end{align*}
From Lemma \ref{lem:Appendix-max-xi-moment}, we know
\[
\E\left[M_{T}\right]\leq\sigma^{2}\left(\log^{2\theta}\left(\Gamma(4\theta+1)e^{2}T^{2}\right)+1\right)\le\sigma^{2}(2^{(4\theta-1)\lor2\theta}\log^{2\theta}(T)+C_{1})
\]
and
\begin{align*}
\E\left[M_{T}^{2}\right] & \leq\sigma^{4}(\log^{4\theta}(\Gamma(8\theta+1)e^{2}T^{2})+1)\\
\Rightarrow\E^{1/4}\left[M_{T}^{2}\right] & =\sigma\left(\log^{4\theta}(\Gamma(8\theta+1)e^{2}T^{2})+1\right)^{1/4}\leq\sigma(\log^{\theta}(\Gamma(8\theta+1)e^{2}T^{2})+1)\\
 & \leq\sigma(2^{(2\theta-1)\lor\theta}\log^{\theta}T+C_{2})
\end{align*}
Hence we have
\begin{align*}
 & \E\left[\Delta_{T}\right]\\
\leq & g_{T}\left(\frac{\|x_{1}-x^{*}\|^{2}}{2\gamma\eta}+\frac{2\eta\sigma^{2}\left(2^{(4\theta-1)\lor2\theta}\log^{2\theta}T+C_{1}\right)}{\gamma b_{0}^{2}}+\frac{\eta}{\gamma}\log\E\left[\frac{b_{T}^{2}}{b_{0}^{2}}\right]\right)\\
 & +\frac{F(x_{1})-F^{*}+\eta^{2}L\log^{+}\frac{\eta L}{2b_{0}}}{T^{3}}+\eta\sigma\left(2^{(2\theta-1)\lor\theta}\log^{\theta}T+C_{2}\right)\sqrt{\log\E\left[\frac{b_{T}^{2}}{b_{0}^{2}}\right]}\sqrt{T}\\
\leq & g_{T}\left(\frac{\|x_{1}-x^{*}\|^{2}}{2\gamma\eta}+\frac{2\eta\sigma^{2}\left(2^{(4\theta-1)\lor2\theta}\log^{2\theta}T+C_{1}\right)}{\gamma b_{0}^{2}}+\frac{\eta}{\gamma}\log\E\left[\frac{b_{T}^{2}}{b_{0}^{2}}\right]\right)\\
 & +\frac{F(x_{1})-F^{*}+\eta^{2}L\log^{+}\frac{\eta L}{2b_{0}}}{T^{3}}+\frac{\eta\sigma\left(2^{(2\theta-1)\lor\theta}\log^{\theta}T+C_{2}\right)}{2}\left(1+\log\E\left[\frac{b_{T}^{2}}{b_{0}^{2}}\right]\right)\sqrt{T}
\end{align*}
\end{proof}

With these results, we can finally show the theorem \ref{thm:Main-AdaGradNorm-Stoc-rate}.

\begin{proof}[Proof of Theorem \ref{thm:Main-AdaGradNorm-Stoc-rate} ]
The key technique we use is the self-bounding argument. That is,
we have expressed a bound for $\E[\Delta_{T}]$ via $\E[b_{T}^{2}/b_{0}^{2}]$,
now we will show how to bound this term via $\Delta_{T}$. To do this,
we rely on the smoothness assumption and Lemma \ref{lem:Appendix-2nd-moments}
\begin{align*}
\E\left[b_{T}^{2}\right] & =\E\left[b_{0}^{2}+\sum_{t=1}^{T}\|\widehat{\nabla}F(x_{t})\|^{2}\right]\\
 & \leq b_{0}^{2}+\E\left[\sum_{t=1}^{T}2\|\xi_{t}\|^{2}\right]+\E\left[\sum_{t=1}^{T}2\|\nabla F(x_{t})\|^{2}\right]\\
 & \leq b_{0}^{2}+2\Gamma(2\theta+1)e\sigma^{2}T+\E\left[4L\sum_{t=1}^{T}F(x_{t})-F(x^{*})\right]\\
 & \leq b_{0}^{2}+2\Gamma(2\theta+1)e\sigma^{2}T+4L\E\left[\Delta_{T}\right].
\end{align*}
Thus from Lemma \ref{lem:Appendix-function-gap-final} we can write
\begin{align}
\E\left[\Delta_{T}\right] & \le G_{0}+G_{1}\log\left(1+\frac{2\Gamma(2\theta+1)e\sigma^{2}T}{b_{0}^{2}}+\frac{4L}{b_{0}^{2}}\E\left[\Delta_{T}\right]\right)\label{eq:final-solution}
\end{align}
where
\begin{align*}
G_{0} & =\frac{F(x_{1})-F^{*}+\eta^{2}L\log^{+}\frac{\eta L}{2b_{0}}}{T^{3}}+\frac{\eta\sigma\left(2^{(2\theta-1)\lor\theta}\log^{\theta}T+C_{2}\right)\sqrt{T}}{2}\\
 & \quad+g_{T}\left(\frac{\|x_{1}-x^{*}\|^{2}}{2\gamma\eta}+\frac{2\eta\sigma^{2}\left(2^{(4\theta-1)\lor2\theta}\log^{2\theta}(T)+C_{1}\right)}{\gamma b_{0}^{2}}\right)\\
 & =O\left(1+\sigma\sqrt{T\log^{2\theta}T}+(1+\sigma^{2}\log^{2\theta}T)g_{T}\right)\\
G_{1} & =\frac{\eta\sigma\left(2^{(2\theta-1)\lor\theta}\log^{\theta}T+C_{2}\right)\sqrt{T}}{2}+\frac{\eta g_{T}}{\gamma}\\
 & =O\left(\sigma\sqrt{T\log^{2\theta}T}+g_{T}\right)\\
g_{T} & =2b_{0}+\frac{4(F(x_{1})-F^{*})}{\eta}+4\eta L\log^{+}\frac{\eta L}{b_{0}}+4\sigma\sqrt{T\log^{2\theta}(eT^{5})\log\left(1+\frac{16\sigma^{2}T\log^{2\theta}(eT^{5})}{b_{0}^{2}}\right)}\\
 & =O\left(1+\sigma\sqrt{T\log^{2\theta}T\log(1+\sigma^{2}T\log^{2\theta}T)}\right)
\end{align*}
Now we solve (\ref{eq:final-solution}). Consider two cases:

If $4L\E\left[\Delta_{T}\right]\le2\Gamma(2\theta+1)e\sigma^{2}T$
then 
\begin{align*}
\E\left[\Delta_{T}\right] & \le G_{0}+G_{1}\log\left(1+\frac{4\Gamma(2\theta+1)e\sigma^{2}T}{b_{0}^{2}}\right).
\end{align*}

If $4L\E\left[\Delta_{T}\right]\ge2\Gamma(2\theta+1)e\sigma^{2}T$
then 
\begin{align*}
\E\left[\Delta_{T}\right] & \le G_{0}+G_{1}\log\left(1+\frac{8L}{b_{0}^{2}}\E\left[\Delta_{T}\right]\right)\\
 & =G_{0}+G_{1}\log\left(\frac{1+\frac{8L}{b_{0}^{2}}\E\left[\Delta_{T}\right]}{1+16LG_{1}/b_{0}^{2}}\right)+G_{1}\log\left(1+\frac{16LG_{1}}{b_{0}^{2}}\right)\\
 & \leq G_{0}+G_{1}\frac{1+\frac{8L}{b_{0}^{2}}\E\left[\Delta_{T}\right]}{1+16LG_{1}/b_{0}^{2}}+G_{1}\log\left(1+\frac{16LG_{1}}{b_{0}^{2}}\right)\\
 & \leq G_{0}+G_{1}+\frac{\E\left[\Delta_{T}\right]}{2}+G_{1}\log\left(1+\frac{16LG_{1}}{b_{0}^{2}}\right)\\
\Rightarrow\E\left[\Delta_{T}\right] & \leq2G_{0}+2G_{1}+2G_{1}\log\left(1+\frac{16eLG_{1}}{b_{0}^{2}}\right).
\end{align*}
In both cases, we have
\begin{align*}
\E\left[\Delta_{T}\right] & \leq3G_{0}+2G_{1}+2G_{1}\log\left(1+\frac{16eLG_{1}}{b_{0}^{2}}\right)+G_{1}\log\left(1+\frac{4\Gamma(2\theta+1)e\sigma^{2}T}{b_{0}^{2}}\right)\\
 & =O\left((1+\poly(\sigma^{2}\log^{2\theta}T,\log(1+\sigma^{2}T\log^{2\theta}T)))(1+\sigma\sqrt{T\log^{2\theta}T})\right)
\end{align*}
Dividing both sides by $T$ concludes the proof.
\end{proof}

\subsubsection{Convergence of Stochastic AdaGradNorm under weaker assumptions}

Note that Theorem \ref{thm:Main-AdaGradNorm-Stoc-rate} depends on
the stronger Assumption 2' instead of Assumption 2. Besides, in Section
\ref{sec:Main-AdaGradNorm}, we proved that Assumptions 1 and 2 are
enough to ensure that AdaGradNorm can converge in the deterministic
setting. Hence it is reasonable to conjecture Stochastic AdaGradNorm
can also converge if replacing Assumption 2' by Assumption 2. In this
section, we show that, indeed, this conjecture is true.
\begin{thm}
\label{thm:Appendix-Stoc-AdaGradNorm-weak-rate}Suppose $F$ satisfies
Assumptions 1, 2, 3 and 4. Stochastic AdaGradNorm (Algorithm \ref{alg:AdaGradNorm-stochastic})
admits
\[
\E\left[\sqrt{\frac{\sum_{t=1}^{T}F(x_{t})-F(x^{*})}{T}}\right]=O\left((1+\poly(\log(1+\sigma\sqrt{T}),\sigma^{2}\log^{2\theta}T))\left(\frac{1}{\sqrt{T}}+\frac{\sigma^{1/2}}{T^{1/4}}\right)\right).
\]
\end{thm}
\begin{proof}
First we invoke Lemma \ref{lem:Appendix-function-gap-over-bT} to
get
\[
\E\left[\frac{\Delta_{T}}{b_{T}}\right]\le\frac{\|x_{1}-x^{*}\|^{2}}{\gamma\eta}+\frac{2\eta}{\gamma}\E\left[\frac{M_{T}}{b_{0}^{2}}+\log\frac{b_{T}}{b_{0}}\right].
\]
Using Holder's inequality we have
\begin{align*}
\E\left[\sqrt{\Delta_{T}}\right] & \le\sqrt{\left(\frac{\|x_{1}-x^{*}\|^{2}}{\gamma\eta}+\frac{2\eta}{\gamma}\E\left[\frac{M_{T}}{b_{0}^{2}}+\log\frac{b_{T}}{b_{0}}\right]\right)\E\left[b_{T}\right]}\\
 & \leq\sqrt{\left(\frac{\|x_{1}-x^{*}\|^{2}}{\gamma\eta}+\frac{2\eta}{\gamma b_{0}^{2}}\E\left[M_{T}\right]+\frac{2\eta}{\gamma}\log\frac{\E\left[b_{T}\right]}{b_{0}}\right)\E\left[b_{T}\right]}.
\end{align*}

Applying Lemma \ref{lem:Appendix-max-xi-moment} with $p=1$ to get
\begin{align*}
\E\left[M_{T}\right] & \leq\sigma^{2}\left(\log^{2\theta}\left(\Gamma(4\theta+1)e^{2}T^{2}\right)+1\right)\\
 & \le\sigma^{2}\left(2^{2\theta}\log^{2\theta}\left(T^{2}\right)+2^{2\theta}\log^{2\theta}\left(\Gamma(4\theta+1)e^{2}\right)+1\right)\\
 & =\sigma^{2}\left(2^{4\theta}\log^{2\theta}T+C\right)
\end{align*}
where $C=2^{2\theta}\log^{2\theta}\left(\Gamma(4\theta+1)e^{2}\right)+1$. 

Besides, note that
\begin{align*}
b_{T} & =\sqrt{b_{0}^{2}+\sum_{t=1}^{T}\|\widehat{\nabla}F(x_{t})\|^{2}}\leq\sqrt{b_{0}^{2}+2\sum_{t=1}^{T}\|\xi_{t}\|^{2}+4L\Delta_{T}}\leq b_{0}+\sqrt{2\sum_{t=1}^{T}\|\xi_{t}\|^{2}}+2\sqrt{L\Delta_{T}}.
\end{align*}
Thus we know 
\begin{align*}
\E\left[b_{T}\right] & \le\E\left[b_{0}+\sqrt{2\sum_{t=1}^{T}\|\xi_{t}\|^{2}}+2\sqrt{L\Delta_{T}}.\right]\\
 & \leq b_{0}+\sqrt{2\sum_{t=1}^{T}\E\left[\|\xi_{t}\|^{2}\right]}+2\sqrt{L}\E\left[\sqrt{\Delta_{T}}\right]\\
 & \leq b_{0}+\sqrt{2\Gamma(2\theta+1)e\sigma^{2}T}+2\sqrt{L}\E\left[\sqrt{\Delta_{T}}\right]
\end{align*}
where the last inequality is due to Lemma \ref{lem:Appendix-2nd-moments}. 

Hence, by letting
\begin{align*}
B_{1} & =\frac{\|x_{1}-x^{*}\|^{2}}{\gamma\eta}+\frac{2\eta\left(2^{4\theta}\log^{2\theta}T+C\right)\sigma^{2}}{\gamma b_{0}^{2}}\\
 & =O(1+\sigma^{2}\log^{2\theta}T)\\
B_{2} & =b_{0}+\sqrt{2\Gamma(2\theta+1)e\sigma^{2}T}\\
 & =O(1+\sigma\sqrt{T})\\
X & =\E\left[\sqrt{\Delta_{T}}\right]
\end{align*}
we can solve the following inequality
\begin{align*}
X^{2} & \le\left(B_{1}+\frac{2\eta}{\gamma}\log\left(\frac{B_{2}+2\sqrt{L}X}{b_{0}}\right)\right)(B_{2}+2\sqrt{L}X)
\end{align*}
to get the final result.
\end{proof}

\subsection{AdaGrad\label{subsec:Appendix-AdaGrad}}

\begin{algorithm}[h]
\caption{AdaGrad}
\label{alg:AdaGrad}

Initialize: $x_{1},\eta>0$

for $t=1$ to $T$

$\quad$for $j=1$ to $d$

$\qquad$$b_{t,j}=\sqrt{b_{0,j}^{2}+\sum_{i=1}^{t}\left(\nabla_{j}F(x_{i})\right)^{2}}$

$\qquad$$x_{t+1,j}=x_{t,j}-\frac{\eta}{b_{t,j}}\nabla_{j}F(x_{t})$
\end{algorithm}

In this section, we will extend the result of AdaGradNorm to AdaGrad
(Algorithm \ref{alg:AdaGrad}) in the deterministic setting. To our
knowledge, we are the first to give the explicit bound of the counvergence
rate of AdaGrad on $\R^{d}$. First, we examine the growth of the
stepsize.
\begin{lem}
\label{lem:adagrad-b_T}Suppose $F$ satisfies Assumptions 1 and 2'',
we have
\[
\sum_{j=1}^{d}b_{T,j}\leq\sum_{j=1}^{d}b_{0,j}+\frac{2(F(x_{1})-F^{*})}{\eta}+2\eta\sum_{j=1}^{d}L_{j}\log^{+}\frac{\eta L_{j}}{b_{0,j}}.
\]
\end{lem}
\begin{proof}
By smoothness we have 
\begin{align*}
F(x_{t+1})-F(x_{t}) & \leq\langle\nabla F(x_{t}),x_{t+1}-x_{t}\rangle+\frac{\|x_{t+1}-x_{t}\|_{\bL}^{2}}{2}\\
 & =\sum_{j=1}^{d}\left(-\frac{\eta}{b_{t,j}}+\frac{L_{j}\eta^{2}}{2b_{t,j}^{2}}\right)\nabla_{j}F(x_{t}){}^{2}\\
\Rightarrow\sum_{j=1}^{d}\frac{\eta}{2b_{t,j}}\nabla_{j}F(x_{t}){}^{2} & \leq F(x_{t})-F(x_{t+1})+\sum_{j=1}^{d}\left(\frac{L_{j}\eta^{2}}{2b_{t,j}^{2}}-\frac{\eta}{2b_{t,j}}\right)\nabla_{j}F(x_{t}){}^{2}\\
\Rightarrow\sum_{t=1}^{T}\sum_{j=1}^{d}\frac{\eta}{2b_{t,j}}\nabla_{j}F(x_{t}){}^{2} & \leq F(x_{1})-F^{*}+\sum_{t=1}^{T}\sum_{j=1}^{d}\left(\frac{L_{j}\eta^{2}}{2b_{t,j}^{2}}-\frac{\eta}{2b_{t,j}}\right)\nabla_{j}F(x_{t}){}^{2}.
\end{align*}
Note that, for the L.H.S.,
\begin{align*}
\sum_{t=1}^{T}\sum_{j=1}^{d}\frac{\eta}{2b_{t,j}}\nabla_{j}F(x_{t}){}^{2} & =\frac{\eta}{2}\sum_{j=1}^{d}\sum_{t=1}^{T}\frac{b_{t,j}^{2}-b_{t-1,j}^{2}}{b_{t,j}}\\
 & \geq\frac{\eta}{2}\sum_{j=1}^{d}\sum_{t=1}^{T}b_{t,j}-b_{t-1,j}\\
 & =\frac{\eta}{2}\sum_{j=1}^{d}\left(b_{T,j}-b_{0,j}\right).
\end{align*}
Besides,
\begin{align*}
\sum_{t=1}^{T}\sum_{j=1}^{d}\left(\frac{L_{j}\eta^{2}}{2b_{t,j}^{2}}-\frac{\eta}{2b_{t,j}}\right)\nabla_{j}F(x_{t}){}^{2} & =\frac{\eta}{2}\sum_{t=1}^{T}\sum_{j=1}^{d}\left(\frac{L_{j}\eta}{b_{t,j}^{2}}-\frac{1}{b_{t,j}}\right)\nabla_{j}F(x_{t}){}^{2}\\
 & \leq\frac{\eta}{2}\sum_{j=1}^{d}\sum_{t=1}^{\tau_{j}}\frac{L_{j}\eta}{b_{t,j}^{2}}\nabla_{i}F(x_{t}){}^{2}\quad(\tau_{j}\text{ is the last }t\text{ such that}b_{t,j}\leq\eta L_{j})\\
 & =\frac{\eta^{2}}{2}\sum_{j=1}^{d}L_{j}\sum_{t=1}^{\tau_{i}}\frac{b_{t,j}^{2}-b_{t-1,j}^{2}}{b_{t,j}^{2}}\\
 & \leq\eta^{2}\sum_{j=1}^{d}L_{j}\log^{+}\frac{\eta L_{j}}{b_{0,j}}.
\end{align*}
Hence we have
\[
\sum_{j=1}^{d}b_{T,j}\leq\sum_{j=1}^{d}b_{0,j}+\frac{2(F(x_{1})-F^{*})}{\eta}+2\eta\sum_{j=1}^{d}L_{j}\log^{+}\frac{\eta L_{j}}{b_{0,j}}.
\]
\end{proof}

Theorem \ref{thm:adagrad} states the convergence guarantee for Algorithm
\ref{alg:AdaGrad}.
\begin{thm}
\label{thm:adagrad}Suppose $F$ satisfies Assumptions 1 and 2'',
AdaGrad (Algorithm \ref{alg:AdaGrad}) admits
\begin{align*}
 & \frac{\sum_{t=1}^{T}F(x_{t})-F^{*}}{T}\\
\le & \frac{\left(\sum_{j=1}^{d}b_{0,j}+\frac{2(F(x_{1})-F^{*})}{\eta}+2\eta\sum_{j=1}^{d}L_{j}\log^{+}\frac{\eta L_{j}}{b_{0,j}}\right)^{d}\left(\frac{\|x_{1}-x^{*}\|_{\bb_{1}}^{2}}{\gamma\eta}+\frac{2\eta}{\gamma}\sum_{j=1}^{d}\left(\frac{2\eta L_{j}}{\gamma}-b_{0,j}\right)^{+}\right)}{T\left(d^{d}\prod_{j=1}^{d}b_{0,j}\right)}
\end{align*}
where
\[
\|x_{1}-x^{*}\|_{\bb_{1}}^{2}=\sum_{j=1}^{d}b_{1,j}(x_{1,j}-x_{j}^{*})^{2}.
\]
\end{thm}
Before going into the proof, it is worth discussing the result above
as well as the main challenges and differences compared with AdaGradNorm.
For simplicity, let $\bb_{t}=\diag(b_{t,i})$. We can expect that,
by a similar argument that we used before to bound the function value
gap via the stepsize, we will have 
\begin{align*}
\frac{\sum_{t=1}^{T}F(x_{t})-F^{*}}{T} & \le\frac{g(\bb_{T})\left(\frac{\|x_{1}-x^{*}\|_{\bb_{1}}^{2}}{\gamma\eta}+\frac{2\eta}{\gamma}\sum_{j=1}^{d}\left(\frac{2\eta L_{j}}{\gamma}-b_{0,j}\right)^{+}\right)}{T}
\end{align*}
where $g(\bb_{T})$ is a function of the last stepsize and the factor
$\frac{\|x_{1}-x^{*}\|_{\bb_{1}}^{2}}{\gamma\eta}+\frac{2\eta}{\gamma}\sum_{j=1}^{d}\left(\frac{2\eta L_{j}}{\gamma}-b_{0,j}\right)^{+}$
is obtained in a similar manner as before, but in $d$-dimensions.
The challenge is that since the stepsize is a vector, it is not possible
to use ``division'' by the stepsize as in AdaGradNorm. On the one
hand, we can overcome this by rewriting the argument; on the other
hand, this problem will incur an exponential rate for $g(\bb_{T})$
dependent on the smoothness parameters.

\begin{proof}
We can write
\begin{align*}
x_{t+1} & =x_{t}-\eta\bb_{t}^{-1}\nabla F(x_{t}).
\end{align*}
Starting from $\gamma$-quasar convexity
\begin{align*}
F(x_{t})-F^{*} & \leq\frac{\langle\nabla F(x_{t}),x_{t}-x^{*}\rangle}{\gamma}\\
 & =\frac{\langle\bb_{t}\left(x_{t}-x_{t+1}\right),x_{t}-x^{*}\rangle}{\eta\gamma}\\
 & =\frac{\|x_{t}-x^{*}\|_{\bb_{t}}^{2}-\|x_{t+1}-x^{*}\|_{\bb_{t}}^{2}+\|x_{t+1}-x_{t}\|_{\bb_{t}}^{2}}{2\eta\gamma}\\
 & =\frac{\|x_{t}-x^{*}\|_{\bb_{t}}^{2}-\|x_{t+1}-x^{*}\|_{\bb_{t}}^{2}}{2\eta\gamma}+\frac{\eta}{2\gamma}\|\nabla F(x_{t})\|_{\bb_{t}^{-1}}^{2}.
\end{align*}
Note that $F$ also satisfies Assumption $2''$
\begin{align*}
F(x_{t})-F^{*} & \geq\frac{\|\nabla F(x_{t})\|_{\bL,*}^{2}}{2}.
\end{align*}
Hence
\begin{align*}
\frac{F(x_{t})-F^{*}}{2}+\frac{\|\nabla F(x_{t})\|_{\bL,*}^{2}}{4} & \leq F(x_{t})-F^{*}\\
 & \leq\frac{\|x_{t}-x^{*}\|_{\bb_{t}}^{2}-\|x_{t+1}-x^{*}\|_{\bb_{t}}^{2}}{2\eta\gamma}+\frac{\eta}{2\gamma}\|\nabla F(x_{t})\|_{\bb_{t}^{-1}}^{2}\\
\Rightarrow\frac{F(x_{t})-F^{*}}{2} & \leq\frac{\|x_{t}-x^{*}\|_{\bb_{t}}^{2}-\|x_{t+1}-x^{*}\|_{\bb_{t}}^{2}}{2\eta\gamma}+\frac{\eta}{2\gamma}\|\nabla F(x_{t})\|_{\bb_{t}^{-1}}^{2}-\frac{\|\nabla F(x_{t})\|_{\bL,*}^{2}}{4}\\
\Rightarrow F(x_{t})-F^{*} & \leq\frac{\|x_{t}-x^{*}\|_{\bb_{t}}^{2}-\|x_{t+1}-x^{*}\|_{\bb_{t}}^{2}}{\eta\gamma}+\sum_{j=1}^{d}\left(\frac{\eta}{\gamma b_{t,j}}-\frac{1}{2L_{j}}\right)\nabla_{j}F(x_{t})^{2}.
\end{align*}
Taking the sum over $t$
\begin{align*}
 & \sum_{t=1}^{T}F(x_{t})-F^{*}\\
\le & \sum_{t=1}^{T}\frac{\|x_{t}-x^{*}\|_{\bb_{t}}^{2}-\|x_{t+1}-x^{*}\|_{\bb_{t}}^{2}}{\eta\gamma}+\sum_{t=1}^{T}\sum_{j=1}^{d}\left(\frac{\eta}{\gamma b_{t,j}}-\frac{1}{2L_{j}}\right)\nabla_{j}F(x_{t})^{2}\\
= & \frac{\|x_{1}-x^{*}\|_{\bb_{1}}^{2}-\|x_{T+1}-x^{*}\|_{\bb_{T}}^{2}}{\gamma\eta}+\sum_{t=2}^{T}\frac{\|x_{t}-x^{*}\|_{\bb_{t}-\bb_{t-1}}^{2}}{\gamma\eta}+\sum_{t=1}^{T}\sum_{j=1}^{d}\left(\frac{\eta}{\gamma b_{t,j}}-\frac{1}{2L_{j}}\right)\nabla_{j}F(x_{t})^{2}.
\end{align*}
Due to the excess term $\sum_{t=2}^{T}\frac{\|x_{t}-x^{*}\|_{\bb_{t}-\bb_{t-1}}^{2}}{\gamma\eta}$
in the RHS, we need to proceed and bound $\|x_{t}-x^{*}\|_{\bb_{t}-\bb_{t-1}}^{2}$.
First, observe that since the L.H.S. is non-negative, 
\begin{align*}
\frac{\|x_{T+1}-x^{*}\|_{\bb_{T}}^{2}}{\gamma\eta} & \leq\frac{\|x_{1}-x^{*}\|_{\bb_{1}}^{2}}{\gamma\eta}+\sum_{t=2}^{T}\frac{\|x_{t}-x^{*}\|_{\bb_{t}-\bb_{t-1}}^{2}}{\gamma\eta}+\sum_{t=1}^{T}\sum_{j=1}^{d}\left(\frac{\eta}{\gamma b_{t,j}}-\frac{1}{2L_{j}}\right)\nabla_{j}F(x_{t})^{2}
\end{align*}
To upperbound $\|x_{t}-x^{*}\|_{\bb_{t}-\bb_{t-1}}^{2}$, a key observation
is that
\begin{align*}
\|x_{T+1}-x^{*}\|_{\bb_{T}}^{2} & =\|x_{T+1}-x^{*}\|_{\bb_{T+1}-\bb_{T}}^{2}\times\frac{\|x_{T+1}-x^{*}\|_{\bb_{T}}^{2}}{\|x_{T+1}-x^{*}\|_{\bb_{T+1}-\bb_{T}}^{2}}\\
 & \geq\|x_{T+1}-x^{*}\|_{\bb_{T+1}-\bb_{T}}^{2}\min_{k}\frac{b_{T,k}}{b_{T+1,k}-b_{T,k}}
\end{align*}
Hence for $T\geq1$
\begin{align*}
 & \frac{\|x_{T+1}-x^{*}\|_{\bb_{T+1}-\bb_{T}}^{2}}{\gamma\eta}\\
\le & \max_{k}\left(\frac{b_{T+1,k}}{b_{T,k}}-1\right)\left[\frac{\|x_{1}-x^{*}\|_{\bb_{1}}^{2}}{\gamma\eta}+\sum_{t=2}^{T}\frac{\|x_{t}-x^{*}\|_{\bb_{t}-\bb_{t-1}}^{2}}{\gamma\eta}+\sum_{t=1}^{T}\sum_{j=1}^{d}\left(\frac{\eta}{\gamma b_{t,j}}-\frac{1}{2L_{j}}\right)\nabla_{j}F(x_{t})^{2}\right]
\end{align*}
By using this bound for the last term $\frac{\|x_{T}-x^{*}\|_{\bb_{T}-\bb_{T-1}}^{2}}{\gamma\eta}$
we obtain 
\begin{align*}
 & \sum_{t=1}^{T}F(x_{t})-F^{*}\\
\leq & \frac{\|x_{1}-x^{*}\|_{\bb_{1}}^{2}-\|x_{T+1}-x^{*}\|_{\bb_{T}}^{2}}{\gamma\eta}+\sum_{t=2}^{T}\frac{\|x_{t}-x^{*}\|_{\bb_{t}-\bb_{t-1}}^{2}}{\gamma\eta}+\sum_{t=1}^{T}\sum_{j=1}^{d}\left(\frac{\eta}{\gamma b_{t,j}}-\frac{1}{2L_{j}}\right)\nabla_{j}F(x_{t})^{2}\\
\leq & \frac{\|x_{1}-x^{*}\|_{\bb_{1}}^{2}}{\gamma\eta}+\sum_{t=2}^{T-1}\frac{\|x_{t}-x^{*}\|_{\bb_{t}-\bb_{t-1}}^{2}}{\gamma\eta}+\sum_{t=1}^{T}\sum_{j=1}^{d}\left(\frac{\eta}{\gamma b_{t,j}}-\frac{1}{2L_{j}}\right)\nabla_{j}F(x_{t})^{2}\\
 & +\max_{k}\left(\frac{b_{T,k}}{b_{T-1,k}}-1\right)\left(\frac{\|x_{1}-x^{*}\|_{\bb_{1}}^{2}}{\gamma\eta}+\sum_{t=2}^{T-1}\frac{\|x_{t}-x^{*}\|_{\bb_{t}-\bb_{t-1}}^{2}}{\gamma\eta}+\sum_{t=1}^{T-1}\sum_{j=1}^{d}\left(\frac{\eta}{\gamma b_{t,j}}-\frac{1}{2L_{j}}\right)\nabla_{j}F(x_{t})^{2}\right)\\
= & \max_{k}\frac{b_{T,k}}{b_{T-1,k}}\left(\frac{\|x_{1}-x^{*}\|_{\bb_{1}}^{2}}{\gamma\eta}+\sum_{t=2}^{T-1}\frac{\|x_{t}-x^{*}\|_{\bb_{t}-\bb_{t-1}}^{2}}{\gamma\eta}+\sum_{t=1}^{T-1}\sum_{j=1}^{d}\left(\frac{\eta}{\gamma b_{t,j}}-\frac{1}{2L_{j}}\right)\nabla_{j}F(x_{t})^{2}\right)\\
 & +\sum_{j=1}^{d}\left(\frac{\eta}{\gamma b_{T,j}}-\frac{1}{2L_{j}}\right)\nabla_{j}F(x_{T})^{2}
\end{align*}
Continue to unroll this relation and for convenience let $\prod_{t=T+1}^{T}\max_{k}\frac{b_{t,k}}{b_{t-1,k}}=1$,
we have 
\begin{align*}
 & \sum_{t=1}^{T}F(x_{t})-F^{*}\\
\le & \left[\prod_{t=2}^{T}\max_{k}\frac{b_{t,k}}{b_{t-1,k}}\right]\frac{\|x_{1}-x^{*}\|_{\bb_{1}}^{2}}{\gamma\eta}+\sum_{t=1}^{T}\left(\prod_{\ell=t+1}^{T}\max_{k}\frac{b_{\ell,k}}{b_{\ell-1,k}}\right)\left(\sum_{j=1}^{d}\left(\frac{\eta}{\gamma b_{t,j}}-\frac{1}{2L_{j}}\right)\nabla_{j}F(x_{t})^{2}\right)\\
= & \left[\prod_{t=2}^{T}\max_{k}\frac{b_{t,k}}{b_{t-1,k}}\right]\frac{\|x_{1}-x^{*}\|_{\bb_{1}}^{2}}{\gamma\eta}+\sum_{j=1}^{d}\sum_{t=1}^{T}\left(\prod_{\ell=t+1}^{T}\max_{k}\frac{b_{\ell,k}}{b_{\ell-1,k}}\right)\left(\frac{\eta}{\gamma b_{t,j}}-\frac{1}{2L_{j}}\right)\nabla_{j}F(x_{t})^{2}
\end{align*}
Given $j$, if $b_{1,j}>\frac{2\eta L_{j}}{\gamma}$, we know
\[
\sum_{t=1}^{T}\left(\prod_{\ell=t+1}^{T}\max_{k}\frac{b_{\ell,k}}{b_{\ell-1,k}}\right)\left(\frac{\eta}{\gamma b_{t,j}}-\frac{1}{2L_{j}}\right)\nabla_{j}F(x_{t})^{2}<0\leq\frac{2\eta}{\gamma}\left[\prod_{t=2}^{T}\max_{k}\frac{b_{t,k}}{b_{t-1,k}}\right]\left(\frac{2\eta L_{j}}{\gamma}-b_{0,j}\right)^{+}.
\]
Otherwise, let $\tau_{j}$ be the last $t$ such that $b_{t,j}\le\frac{2\eta L_{j}}{\gamma}$,
we also have 
\begin{align*}
 & \sum_{t=1}^{T}\left(\prod_{\ell=t+1}^{T}\max_{k}\frac{b_{\ell,k}}{b_{\ell-1,k}}\right)\left(\frac{\eta}{\gamma b_{t,j}}-\frac{1}{2L_{j}}\right)\nabla_{j}F(x_{t})^{2}\\
\leq & \sum_{t=1}^{\tau_{j}}\left(\prod_{\ell=t+1}^{T}\max_{k}\frac{b_{\ell,k}}{b_{\ell-1,k}}\right)\left(\frac{\eta}{\gamma b_{t,j}}-\frac{1}{2L_{j}}\right)\nabla_{j}F(x_{t})^{2}\\
\leq & \left[\prod_{t=2}^{T}\max_{k}\frac{b_{t,k}}{b_{t-1,k}}\right]\left(\sum_{t=1}^{\tau_{j}}\left(\frac{\eta}{\gamma b_{t,j}}-\frac{1}{2L_{j}}\right)\nabla_{j}F(x_{t})^{2}\right)\\
\leq & \left[\prod_{t=2}^{T}\max_{k}\frac{b_{t,k}}{b_{t-1,k}}\right]\left(\sum_{t=1}^{\tau_{j}}\frac{\eta}{\gamma b_{t,j}}\nabla_{j}F(x_{t})^{2}\right)\\
= & \left[\prod_{t=2}^{T}\max_{k}\frac{b_{t,k}}{b_{t-1,k}}\right]\left(\frac{\eta}{\gamma}\sum_{t=1}^{\tau_{j}}\frac{b_{t,j}^{2}-b_{t-1,j}^{2}}{b_{t,j}}\right)\\
\leq & \left[\prod_{t=2}^{T}\max_{k}\frac{b_{t,k}}{b_{t-1,k}}\right]\left(\frac{2\eta}{\gamma}\sum_{t=1}^{\tau_{j}}b_{t,j}-b_{t-1,j}\right)\\
\leq & \frac{2\eta}{\gamma}\left[\prod_{t=2}^{T}\max_{k}\frac{b_{t,k}}{b_{t-1,k}}\right]\left(\frac{2\eta L_{j}}{\gamma}-b_{0,j}\right)^{+}.
\end{align*}
Hence we know
\begin{align*}
 & \sum_{j=1}^{d}\sum_{t=1}^{T}\left(\prod_{\ell=t+1}^{T}\max_{k}\frac{b_{\ell,k}}{b_{\ell-1,k}}\right)\left(\frac{\eta}{\gamma b_{t,j}}-\frac{1}{2L_{j}}\right)\nabla_{j}F(x_{t})^{2}\\
\leq & \frac{2\eta}{\gamma}\left[\prod_{t=2}^{T}\max_{k}\frac{b_{t,k}}{b_{t-1,k}}\right]\left[\sum_{j=1}^{d}\left(\frac{2\eta L_{j}}{\gamma}-b_{0,j}\right)^{+}\right].
\end{align*}
Thus we have
\begin{align*}
 & \sum_{t=1}^{T}F(x_{t})-F^{*}\\
\le & \left[\prod_{t=2}^{T}\max_{k}\frac{b_{t,k}}{b_{t-1,k}}\right]\frac{\|x_{1}-x^{*}\|_{\bb_{1}}^{2}}{\gamma\eta}+\sum_{j=1}^{d}\sum_{t=1}^{T}\left(\prod_{\ell=t+1}^{T}\max_{k}\frac{b_{\ell,k}}{b_{\ell-1,k}}\right)\left(\frac{\eta}{\gamma b_{t,j}}-\frac{1}{2L_{j}}\right)\nabla_{j}F(x_{t})^{2}\\
\leq & \left[\prod_{t=2}^{T}\max_{k}\frac{b_{t,k}}{b_{t-1,k}}\right]\left(\frac{\|x_{1}-x^{*}\|_{\bb_{1}}^{2}}{\gamma\eta}+\frac{2\eta}{\gamma}\sum_{j=1}^{d}\left(\frac{2\eta L_{j}}{\gamma}-b_{0,j}\right)^{+}\right)
\end{align*}
From Lemma \ref{lem:adagrad-b_T}
\begin{align*}
\sum_{j=1}^{d}b_{T,j} & \le\sum_{j=1}^{d}b_{0,j}+\frac{2(F(x_{1})-F^{*})}{\eta}+2\eta\sum_{j=1}^{d}L_{j}\log^{+}\frac{\eta L_{j}}{b_{0,j}}
\end{align*}
Using AM-GM we have 
\begin{align*}
\prod_{j=1}^{d}b_{T,j} & \le\left(\frac{\sum_{j=1}^{d}b_{T,j}}{d}\right)^{d}\le\frac{1}{d^{d}}\left(\sum_{j=1}^{d}b_{0,j}+\frac{2(F(x_{1})-F^{*})}{\eta}+2\eta\sum_{j=1}^{d}L_{j}\log^{+}\frac{\eta L_{j}}{b_{0,j}}\right)^{d}
\end{align*}
Note that
\begin{align*}
\prod_{t=2}^{T}\max_{j}\frac{b_{t,j}}{b_{t-1,j}} & \leq\prod_{t=2}^{T}\prod_{j=1}^{d}\frac{b_{t,j}}{b_{t-1,j}}\\
 & \le\prod_{j=1}^{d}\frac{b_{T,j}}{b_{0,j}}\\
 & \le\frac{1}{d^{d}\prod_{j=1}^{d}b_{0,j}}\left(\sum_{j=1}^{d}b_{0,j}+\frac{2(F(x_{1})-F^{*})}{\eta}+2\eta\sum_{j=1}^{d}L_{j}\log^{+}\frac{\eta L_{j}}{b_{0,j}}\right)^{d}
\end{align*}
Hence
\begin{align*}
 & \sum_{t=1}^{T}F(x_{t})-F^{*}\\
\le & \left[\prod_{t=2}^{T}\max_{j}\frac{b_{T,j}}{b_{T-1,j}}\right]\left(\frac{\|x_{1}-x^{*}\|_{\bb_{1}}^{2}}{\gamma\eta}+\frac{2\eta}{\gamma}\sum_{j=1}^{d}\left(\frac{2\eta L_{j}}{\gamma}-b_{0,j}\right)^{+}\right)\\
\le & \frac{\left(\sum_{j=1}^{d}b_{0,j}+\frac{2(F(x_{1})-F^{*})}{\eta}+2\eta\sum_{j=1}^{d}L_{j}\log^{+}\frac{\eta L_{j}}{b_{0,j}}\right)^{d}\left(\frac{\|x_{1}-x^{*}\|_{\bb_{1}}^{2}}{\gamma\eta}+\frac{2\eta}{\gamma}\sum_{j=1}^{d}\left(\frac{2\eta L_{j}}{\gamma}-b_{0,j}\right)^{+}\right)}{d^{d}\prod_{j=1}^{d}b_{0,j}}
\end{align*}
which finishes the proof.
\end{proof}

\section{Missing proofs from Section \ref{sec:Main-Last}}

\subsection{Important lemma}

First, we state a general lemma that can be used for a more general
setting. The proof of the lemma is standard.
\begin{lem}
\label{lem:Appendix-last-key-lemma}Suppose $F$ satisfies Assumptions
1 and 2' and the following conditions hold:
\begin{itemize}
\item $x_{t}$ is generated by $x_{t+1}=x_{t}-\frac{\eta}{c_{t}}\nabla F(x_{t})$,
with $\eta>0$ and $c_{t}>0$ is non-decreasing;
\item $p_{t}\in(0,1]$ satisfies $\frac{1}{p_{t}}\geq\frac{1-p_{t+1}}{p_{t+1}},p_{1}=1$;
\end{itemize}
Then we have
\begin{equation}
\frac{F(x_{T+1})-F^{*}}{p_{T}c_{T}}\leq\frac{\|x_{1}-x^{*}\|^{2}}{\gamma\eta}+\sum_{t=1}^{T}\left(\frac{L}{2c_{t}}-\frac{1}{\eta}+\frac{p_{t}}{\gamma\eta}-\frac{p_{t}c_{t}}{2\eta^{2}L}\right)\frac{\eta^{2}\|\nabla F(x_{t})\|^{2}}{c_{t}^{2}p_{t}}.\label{eq:last-lemma-bound-1}
\end{equation}
\end{lem}
\begin{proof}
Starting from $L$-smoothness
\begin{align*}
F(x_{t+1})-F(x_{t}) & \leq\langle\nabla F(x_{t}),x_{t+1}-x_{t}\rangle+\frac{L}{2}\|x_{t+1}-x_{t}\|^{2}\\
 & =\frac{2p_{t}}{\gamma}\langle\nabla F(x_{t}),x_{t+1}-x_{t}\rangle+\left(1-\frac{2p_{t}}{\gamma}\right)\langle\nabla F(x_{t}),x_{t+1}-x_{t}\rangle+\frac{L}{2}\|x_{t+1}-x_{t}\|^{2}\\
 & =\frac{2p_{t}}{\gamma}\langle\nabla F(x_{t}),x^{*}-x_{t}\rangle+\frac{2p_{t}}{\gamma}\langle\nabla F(x_{t}),x_{t+1}-x^{*}\rangle\\
 & \quad+\left(1-\frac{2p_{t}}{\gamma}\right)\langle\nabla F(x_{t}),x_{t+1}-x_{t}\rangle+\frac{L}{2}\|x_{t+1}-x_{t}\|^{2}\\
 & \leq2p_{t}\left(F^{*}-F(x_{t})\right)+\frac{p_{t}c_{t}}{\gamma\eta}\left[\|x_{t}-x^{*}\|^{2}-\|x_{t+1}-x^{*}\|^{2}-\|x_{t+1}-x\|^{2}\right]\\
 & \quad-\left(1-\frac{2p_{t}}{\gamma}\right)\frac{c_{t}}{\eta}\|x_{t+1}-x_{t}\|^{2}+\frac{L}{2}\|x_{t+1}-x_{t}\|^{2}\\
 & =2p_{t}\left(F^{*}-F(x_{t})\right)+\frac{p_{t}c_{t}}{\gamma\eta}\left[\|x_{t}-x^{*}\|^{2}-\|x_{t+1}-x^{*}\|^{2}\right]\\
 & \quad+\left(\frac{L}{2}-\frac{c_{t}}{\eta}+\frac{p_{t}c_{t}}{\gamma\eta}\right)\|x_{t+1}-x_{t}\|^{2}.
\end{align*}
Note that Assumption 2 can be implied by Assumption 2', hence we have
\begin{align*}
F^{*}-F(x_{t}) & \leq-\frac{\|\nabla F(x_{t})\|^{2}}{2L}=-\frac{c_{t}^{2}\|x_{t+1}-x_{t}\|^{2}}{2\eta^{2}L}.
\end{align*}
Therefore
\begin{align*}
F(x_{t+1})-F(x_{t}) & \leq2p_{t}\left(F^{*}-F(x_{t})\right)+\frac{p_{t}c_{t}}{\gamma\eta}\left[\|x_{t}-x^{*}\|^{2}-\|x_{t+1}-x^{*}\|^{2}\right]\\
 & \quad+\left(\frac{L}{2}-\frac{c_{t}}{\eta}+\frac{p_{t}c_{t}}{\gamma\eta}\right)\|x_{t+1}-x_{t}\|^{2}\\
 & \leq p_{t}\left(F^{*}-F(x_{t})\right)-\frac{p_{t}c_{t}^{2}\|x_{t+1}-x_{t}\|^{2}}{2\eta^{2}L}+\frac{p_{t}c_{t}}{\gamma\eta}\left[\|x_{t}-x^{*}\|^{2}-\|x_{t+1}-x^{*}\|^{2}\right]\\
 & \quad+\left(\frac{L}{2}-\frac{c_{t}}{\eta}+\frac{p_{t}c_{t}}{\gamma\eta}\right)\|x_{t+1}-x_{t}\|^{2}\\
 & =p_{t}\left(F^{*}-F(x_{t})\right)+\frac{p_{t}c_{t}}{\gamma\eta}\left[\|x_{t}-x^{*}\|^{2}-\|x_{t+1}-x^{*}\|^{2}\right]\\
 & \quad+\left(\frac{L}{2}-\frac{c_{t}}{\eta}+\frac{p_{t}c_{t}}{\gamma\eta}-\frac{p_{t}c_{t}^{2}}{2\eta^{2}L}\right)\|x_{t+1}-x_{t}\|^{2}.
\end{align*}
We obtain 
\begin{align*}
\frac{F(x_{t+1})-F^{*}}{p_{t}c_{t}} & \leq\frac{1-p_{t}}{p_{t}c_{t}}\left(F(x_{t})-F^{*}\right)+\frac{\|x_{t}-x^{*}\|^{2}-\|x_{t+1}-x^{*}\|^{2}}{\gamma\eta}\\
 & \quad+\left(\frac{L}{2c_{t}}-\frac{1}{\eta}+\frac{p_{t}}{\gamma\eta}-\frac{p_{t}c_{t}}{2\eta^{2}L}\right)\frac{\|x_{t+1}-x_{t}\|^{2}}{p_{t}}.
\end{align*}
Note that we require $\frac{1}{p_{t}}\geq\frac{1-p_{t+1}}{p_{t+1}}$
and $c_{t}$ is increasing hence
\[
\frac{1}{p_{t}c_{t}}\geq\frac{1-p_{t+1}}{p_{t+1}c_{t}}\geq\frac{1-p_{t+1}}{p_{t+1}c_{t+1}}
\]
which leads to
\begin{align*}
 & \frac{F(x_{T+1})-F^{*}}{p_{T}c_{T}}\\
\le & \frac{1-p_{1}}{p_{1}c_{1}}\left(F(x_{1})-F^{*}\right)+\frac{\|x_{1}-x^{*}\|^{2}}{\gamma\eta}+\sum_{t=1}^{T}\left(\frac{L}{2c_{t}}-\frac{1}{\eta}+\frac{p_{t}}{\gamma\eta}-\frac{p_{t}c_{t}}{2\eta^{2}L}\right)\frac{\|x_{t+1}-x_{t}\|^{2}}{p_{t}}\\
= & \frac{1-p_{1}}{p_{1}c_{1}}\left(F(x_{1})-F^{*}\right)+\frac{\|x_{1}-x^{*}\|^{2}}{\gamma\eta}+\sum_{t=1}^{T}\left(\frac{L}{2c_{t}}-\frac{1}{\eta}+\frac{p_{t}}{\gamma\eta}-\frac{p_{t}c_{t}}{2\eta^{2}L}\right)\frac{\eta^{2}\|\nabla F(x_{t})\|^{2}}{c_{t}^{2}p_{t}}.
\end{align*}
By setting $p_{1}=1$ we get the desired result.
\end{proof}

\subsection{First variant}

Note that if we assume $p_{t}$ satisfies the condition in Lemma \ref{lem:Appendix-last-key-lemma},
by replacing $c_{t}$ by $b_{t}$, we have 
\[
\frac{F(x_{T+1})-F^{*}}{p_{T}b_{T}}\leq\frac{\|x_{1}-x^{*}\|^{2}}{\gamma\eta}+\sum_{t=1}^{T}\left(\frac{L}{2b_{t}}-\frac{1}{\eta}+\frac{p_{t}}{\gamma\eta}-\frac{p_{t}b_{t}}{2\eta^{2}L}\right)\frac{\eta^{2}\|\nabla F(x_{t})\|^{2}}{b_{t}^{2}p_{t}}
\]
immediately. Now our two left tasks are to bound the residual term
$\sum_{t=1}^{T}\left(\frac{L}{2b_{t}}-\frac{1}{\eta}+\frac{p_{t}}{\gamma\eta}-\frac{p_{t}b_{t}}{2\eta^{2}L}\right)\frac{\eta^{2}\|\nabla F(x_{t})\|^{2}}{b_{t}^{2}p_{t}}$
and find an upper bound on $b_{T}$. Lemmas \ref{lem:Appendix-AdaGradNorm-last1-residual}
and \ref{lem:Appendix-AdaGradNorm-Last1-b_T} demonstrate how we achieve
these two goals. 
\begin{lem}
\label{lem:Appendix-AdaGradNorm-last1-residual}If $p_{t}\leq1$ for
every $t$, we have
\begin{align*}
\sum_{t=1}^{T}\left(\frac{L}{2b_{t}}-\frac{1}{2\eta}\right)\frac{\eta^{2}\|\nabla F(x_{t})\|^{2}}{b_{t}^{2}p_{t}} & \leq h(\Delta)\\
\sum_{t=1}^{T}\left(\frac{p_{t}}{\gamma\eta}-\frac{p_{t}b_{t}}{2\eta^{2}L}\right)\frac{\eta^{2}\|\nabla F(x_{t})\|^{2}}{b_{t}^{2}p_{t}} & \leq g(\Delta)
\end{align*}
where
\begin{align*}
h(\Delta) & \coloneqq\begin{cases}
\frac{\left(2+\Delta\right)\eta\left(\eta L\right)^{\Delta}}{2}\log^{+}\frac{\eta L}{b_{0}} & \Delta\geq1\\
\frac{\left(2+\Delta\right)\eta^{2}L}{2b_{0}^{1-\Delta}}\log^{+}\frac{\eta L}{b_{0}} & \Delta\in\left(0,1\right)
\end{cases}\;\text{and}\; & g(\Delta)\coloneqq\frac{(2+\Delta)\eta}{\gamma}\left(\frac{2\eta L}{\gamma}\right)^{\Delta}\log^{+}\frac{2\eta L}{\gamma b_{0}}.
\end{align*}
\end{lem}
\begin{proof}
We first bound 
\[
\sum_{t=1}^{T}\left(\frac{L}{2b_{t}}-\frac{1}{2\eta}\right)\frac{\eta^{2}\|\nabla F(x_{t})\|^{2}}{b_{t}^{2}p_{t}}.
\]
If $b_{1}>\eta L$, we know 
\[
\sum_{t=1}^{T}\left(\frac{L}{2b_{t}}-\frac{1}{2\eta}\right)\frac{\eta^{2}\|\nabla F(x_{t})\|^{2}}{b_{t}^{2}p_{t}}<0\leq h(\Delta).
\]
Otherwise, we define the time $\tau=\max\left\{ t\in[T],b_{t}\leq\eta L\right\} $.
Hence, we have
\begin{align*}
\sum_{t=1}^{T}\left(\frac{L}{2b_{t}}-\frac{1}{2\eta}\right)\frac{\eta^{2}\|\nabla F(x_{t})\|^{2}}{b_{t}^{2}p_{t}} & =\sum_{t=1}^{\tau}\left(\frac{L}{2b_{t}}-\frac{1}{2\eta}\right)\frac{\eta^{2}\|\nabla F(x_{t})\|^{2}}{b_{t}^{2}p_{t}}+\sum_{t=\tau}^{T}\left(\frac{L}{2b_{t}}-\frac{1}{2\eta}\right)\frac{\eta^{2}\|\nabla F(x_{t})\|^{2}}{b_{t}^{2}p_{t}}\\
 & \leq\sum_{t=1}^{\tau}\left(\frac{L}{2b_{t}}-\frac{1}{2\eta}\right)\frac{\eta^{2}\|\nabla F(x_{t})\|^{2}}{b_{t}^{2}p_{t}}\leq\sum_{t=1}^{\tau}\frac{L}{2b_{t}}\times\frac{\eta^{2}\|\nabla F(x_{t})\|^{2}}{b_{t}^{2}p_{t}}\\
 & =\frac{\eta^{2}L}{2}\sum_{t=1}^{\tau}\frac{b_{t}^{2+\Delta}-b_{t-1}^{2+\Delta}}{b_{t}^{3}}=\frac{\eta^{2}L}{2}\sum_{t=1}^{\tau}\frac{b_{t}^{2+\Delta}-b_{t-1}^{2+\Delta}}{b_{t}^{2+\Delta}}\times b_{t}^{\Delta-1}\\
 & \leq\begin{cases}
\frac{\eta^{2}L}{2}\sum_{t=1}^{\tau}\frac{b_{t}^{2+\Delta}-b_{t-1}^{2+\Delta}}{b_{t}^{2+\Delta}}\times\left(\eta L\right)^{\Delta-1} & \Delta\geq1\\
\frac{\eta^{2}L}{2}\sum_{t=1}^{\tau}\frac{b_{t}^{2+\Delta}-b_{t-1}^{2+\Delta}}{b_{t}^{2+\Delta}}\times\frac{1}{b_{0}^{1-\Delta}} & \Delta<1
\end{cases}\\
 & =\begin{cases}
\frac{\eta\left(\eta L\right)^{\Delta}}{2}\sum_{t=1}^{\tau}\frac{b_{t}^{2+\Delta}-b_{t-1}^{2+\Delta}}{b_{t}^{2+\Delta}} & \Delta\geq1\\
\frac{\eta^{2}L}{2b_{0}^{1-\Delta}}\sum_{t=1}^{\tau}\frac{b_{t}^{2+\Delta}-b_{t-1}^{2+\Delta}}{b_{t}^{2+\Delta}} & \Delta<1
\end{cases}\\
 & \le\begin{cases}
\frac{\left(2+\Delta\right)\eta\left(\eta L\right)^{\Delta}}{2}\log\frac{\eta L}{b_{0}} & \Delta\geq1\\
\frac{\left(2+\Delta\right)\eta^{2}L}{2b_{0}^{1-\Delta}}\log\frac{\eta L}{b_{0}} & \Delta<1
\end{cases}\\
 & \le h(\Delta).
\end{align*}

By applying a similar argument, we can prove
\[
\sum_{t=1}^{T}\left(\frac{p_{t}}{\gamma\eta}-\frac{p_{t}b_{t}}{2\eta^{2}L}\right)\frac{\eta^{2}\|\nabla F(x_{t})\|^{2}}{b_{t}^{2}p_{t}}\leq g(\Delta).
\]
\end{proof}

\begin{lem}
\label{lem:Appendix-AdaGradNorm-Last1-b_T}Suppose all the conditions
in Lemma \ref{lem:Appendix-last-key-lemma} are satisfied by replacing
$c_{t}$ by $b_{t}$, additionally, assume $p_{t}\leq1$, we will
have
\[
b_{T}\leq\left(\frac{2}{\eta}\left(\frac{\|x_{1}-x^{*}\|^{2}}{\gamma\eta}+h(\Delta)+g(\Delta)\right)+b_{0}^{\Delta}\right)^{\frac{1}{\Delta}}
\]
\end{lem}
\begin{proof}
Using Lemma \ref{lem:Appendix-last-key-lemma} by replacing $c_{t}$
by $b_{t}$, we know
\begin{align*}
 & \frac{F(x_{T+1})-F^{*}}{p_{T}b_{T}}\\
\le & \frac{\|x_{1}-x^{*}\|^{2}}{\gamma\eta}+\sum_{t=1}^{T}\left(\frac{L}{2b_{t}}-\frac{1}{\eta}+\frac{p_{t}}{\gamma\eta}-\frac{p_{t}b_{t}}{2\eta^{2}L}\right)\frac{\eta^{2}\|\nabla F(x_{t})\|^{2}}{b_{t}^{2}p_{t}}\\
= & \frac{\|x_{1}-x^{*}\|^{2}}{\gamma\eta}+\sum_{t=1}^{T}\left(\frac{L}{2b_{t}}-\frac{1}{2\eta}+\frac{p_{t}}{\gamma\eta}-\frac{p_{t}b_{t}}{2\eta^{2}L}\right)\frac{\eta^{2}\|\nabla F(x_{t})\|^{2}}{b_{t}^{2}p_{t}}-\frac{\eta\|\nabla F(x_{t})\|^{2}}{2b_{t}^{2}p_{t}}\\
\le & \frac{\|x_{1}-x^{*}\|^{2}}{\gamma\eta}+h(\Delta)+g(\Delta)-\sum_{t=1}^{T}\frac{\eta\|\nabla F(x_{t})\|^{2}}{2b_{t}^{2}p_{t}},
\end{align*}
where the last inequality is by Lemma \ref{lem:Appendix-AdaGradNorm-last1-residual}.
Noticing $F(x_{T+1})-F^{*}\geq0$, we know
\[
\sum_{t=1}^{T}\frac{\eta\|\nabla F(x_{t})\|^{2}}{2b_{t}^{2}p_{t}}\leq\frac{\|x_{1}-x^{*}\|^{2}}{\gamma\eta}+h(\Delta)+g(\Delta).
\]
Now we use the update rule of $b_{t}$ to get
\[
\sum_{t=1}^{T}\frac{\eta\|\nabla F(x_{t})\|^{2}}{2b_{t}^{2}p_{t}}=\frac{\eta}{2}\sum_{t=1}^{T}\frac{b_{t}^{2+\Delta}-b_{t-1}^{2+\Delta}}{b_{t}^{2}}\geq\frac{\eta}{2}\sum_{t=1}^{T}b_{t}^{\Delta}-b_{t-1}^{\Delta}=\frac{\eta}{2}\left(b_{T}^{\Delta}-b_{0}^{\Delta}\right)
\]
Hence we know
\[
b_{T}\leq\left(\frac{2}{\eta}\left(\frac{\|x_{1}-x^{*}\|^{2}}{\gamma\eta}+h(\Delta)+g(\Delta)\right)+b_{0}^{\Delta}\right)^{\frac{1}{\Delta}}
\]
\end{proof}

Equipped with Lemmas \ref{lem:Appendix-AdaGradNorm-last1-residual}
and \ref{lem:Appendix-AdaGradNorm-Last1-b_T}, we can give a proof
of Theorem \ref{thm:Main-AdaGradNorm-Last1-rate}.

\begin{proof}
Note that if $p_{t}=\frac{1}{t}$, all the conditions in Lemma \ref{lem:Appendix-last-key-lemma}
are satisfied by replacing $c_{t}$ by $b_{t}$. Hence we have
\begin{align*}
\frac{F(x_{T+1})-F^{*}}{p_{T}b_{T}} & \leq\frac{\|x_{1}-x^{*}\|^{2}}{\gamma\eta}+\sum_{t=1}^{T}\left(\frac{L}{2b_{t}}-\frac{1}{\eta}+\frac{p_{t}}{\gamma\eta}-\frac{p_{t}b_{t}}{2\eta^{2}L}\right)\frac{\eta^{2}\|\nabla F(x_{t})\|^{2}}{b_{t}^{2}p_{t}}\\
 & \leq\frac{\|x_{1}-x^{*}\|^{2}}{\gamma\eta}+\sum_{t=1}^{T}\left(\frac{L}{2b_{t}}-\frac{1}{2\eta}+\frac{p_{t}}{\gamma\eta}-\frac{p_{t}b_{t}}{2\eta^{2}L}\right)\frac{\eta^{2}\|\nabla F(x_{t})\|^{2}}{b_{t}^{2}p_{t}}\\
 & \leq\frac{\|x_{1}-x^{*}\|^{2}}{\gamma\eta}+h(\Delta)+g(\Delta),
\end{align*}
where the last inequality is by Lemma \ref{lem:Appendix-AdaGradNorm-last1-residual}.
Multiplying both sides by $p_{T}b_{T}$, we get
\begin{align*}
F(x_{T+1})-F^{*} & \leq\frac{b_{T}\left(\frac{\|x_{1}-x^{*}\|^{2}}{\gamma\eta}+h(\Delta)+g(\Delta)\right)}{T}.
\end{align*}
By using the upper bound of $b_{T}$ in Lemma \ref{lem:Appendix-AdaGradNorm-Last1-b_T},
we finish the proof.
\end{proof}

\subsection{Second variant}

Similar to the previous section, what we need to do is to bound the
residual term $\sum_{t=1}^{T}\left(\frac{L}{2c_{t}}-\frac{1}{\eta}+\frac{p_{t}}{\gamma\eta}-\frac{p_{t}c_{t}}{2\eta^{2}L}\right)\frac{\eta^{2}\|\nabla F(x_{t})\|^{2}}{c_{t}^{2}p_{t}}$
and find an upper bound on $c_{T}$ where $c_{t}=b_{t}^{\delta}b_{t-1}^{1-\delta}$
here. We first bound the residual term by the following lemma.
\begin{lem}
\label{lem:Appendix-AdaGradNorm-last2-residual}If $p_{t}\leq1$ for
every $t$, we have
\begin{align*}
\sum_{t=1}^{T}\left(\frac{L}{2c_{t}}-\frac{1}{2\eta}\right)\frac{\eta^{2}\|\nabla F(x_{t})\|^{2}}{c_{t}^{2}p_{t}} & \leq\frac{\eta^{2}L}{b_{0}}\left(1-\left(\frac{b_{0}}{\eta L}\right)^{\frac{1}{\delta}}\right)^{+}\\
\sum_{t=1}^{T}\left(\frac{p_{t}}{\gamma\eta}-\frac{p_{t}c_{t}}{2\eta^{2}L}\right)\frac{\eta^{2}\|\nabla F(x_{t})\|^{2}}{c_{t}^{2}p_{t}} & \leq\frac{2\eta}{\gamma\delta}\left(\frac{2\eta L}{\gamma b_{0}}\right)^{\frac{2}{\delta}-2}\log^{+}\frac{2\eta L}{\gamma b_{0}}
\end{align*}
where $c_{t}=b_{t}^{\delta}b_{t-1}^{1-\delta}$.
\end{lem}
\begin{proof}
Note that$\frac{c_{t}}{c_{t-1}}=\frac{b_{t}^{\delta}}{b_{t-1}^{2\delta-1}b_{t-2}^{1-\delta}}\geq1$,
this means $c_{t}$is monotone increasing. We first bound 
\[
\sum_{t=1}^{T}\left(\frac{L}{2c_{t}}-\frac{1}{2\eta}\right)\frac{\eta^{2}\|\nabla F(x_{t})\|^{2}}{c_{t}^{2}p_{t}}.
\]
If $c_{1}>\eta L$, we know 
\[
\sum_{t=1}^{T}\left(\frac{L}{2c_{t}}-\frac{1}{2\eta}\right)\frac{\eta^{2}\|\nabla F(x_{t})\|^{2}}{c_{t}^{2}p_{t}}<0\leq\frac{\eta^{2}L}{b_{0}}\left(1-\left(\frac{b_{0}}{\eta L}\right)^{\frac{1}{\delta}}\right)^{+}.
\]
Otherwise, let $\tau=\max\left\{ t\in\left[T\right],c_{t}\leq\eta L\right\} $.
We have
\begin{align*}
\sum_{t=1}^{T}\left(\frac{L}{2c_{t}}-\frac{1}{2\eta}\right)\frac{\eta^{2}\|\nabla F(x_{t})\|^{2}}{c_{t}^{2}p_{t}} & \leq\sum_{t=1}^{\tau}\frac{\eta^{2}L\|\nabla F(x_{t})\|^{2}}{2c_{t}^{3}p_{t}}\\
 & =\frac{\eta^{2}L}{2}\sum_{t=1}^{\tau}\frac{b_{t}^{2}-b_{t-1}^{2}}{b_{t}^{3\delta}b_{t-1}^{3-3\delta}}\\
 & \leq\eta^{2}L\sum_{t=1}^{\tau}\frac{b_{t}-b_{t-1}}{b_{t}^{3\delta-1}b_{t-1}^{3-3\delta}}\\
 & \leq\eta^{2}L\sum_{t=1}^{\tau}\frac{b_{t}-b_{t-1}}{b_{t}b_{t-1}}\\
 & \leq\eta^{2}L\left(\frac{1}{b_{0}}-\frac{1}{b_{\tau}}\right).
\end{align*}
Note that $c_{\tau}=b_{\tau}^{\delta}b_{\tau-1}^{1-\delta}\leq\eta L\Rightarrow b_{\tau}\leq\left(\frac{\eta L}{b_{\tau-1}^{1-\delta}}\right)^{1/\delta}$.
Hence
\[
\frac{1}{b_{0}}-\frac{1}{b_{\tau}}\leq\frac{1}{b_{0}}-\frac{\left(b_{\tau-1}\right)^{\frac{1}{\delta}-1}}{\left(\eta L\right)^{1/\delta}}\leq\frac{1}{b_{0}}\left(1-\left(\frac{b_{0}}{\eta L}\right)^{\frac{1}{\delta}}\right)^{+}.
\]
Combinging two cases, there is always 
\[
\sum_{t=1}^{T}\left(\frac{L}{2c_{t}}-\frac{1}{2\eta}\right)\frac{\eta^{2}\|\nabla F(x_{t})\|^{2}}{c_{t}^{2}p_{t}}\leq\frac{\eta^{2}L}{b_{0}}\left(1-\left(\frac{b_{0}}{\eta L}\right)^{\frac{1}{\delta}}\right)^{+}.
\]

Now we turn to the second bound. If $c_{1}>2\eta L/\gamma$, we know
\[
\sum_{t=1}^{T}\left(\frac{p_{t}}{\gamma\eta}-\frac{p_{t}c_{t}}{2\eta^{2}L}\right)\frac{\eta^{2}\|\nabla F(x_{t})\|^{2}}{c_{t}^{2}p_{t}}<0\leq\frac{2\eta}{\gamma\delta}\left(\frac{2\eta L}{\gamma b_{0}}\right)^{\frac{2}{\delta}-2}\log^{+}\frac{2\eta L}{\gamma b_{0}}.
\]
Otherwise, we define the time $\tau=\max\left\{ t\in\left[T\right],c_{t}\leq2\eta L/\gamma\right\} $.
Then, we have
\begin{align*}
\sum_{t=1}^{T}\left(\frac{p_{t}}{\gamma\eta}-\frac{p_{t}c_{t}}{2\eta^{2}L}\right)\frac{\eta^{2}\|\nabla F(x_{t})\|^{2}}{c_{t}^{2}p_{t}} & \leq\sum_{t=1}^{\tau}\frac{p_{t}}{\gamma\eta}\frac{\eta^{2}\|\nabla F(x_{t})\|^{2}}{c_{t}^{2}p_{t}}\\
 & \leq\frac{\eta}{\gamma}\sum_{t=1}^{\tau}\frac{\|\nabla F(x_{t})\|^{2}}{c_{t}^{2}p_{t}}\\
 & =\frac{\eta}{\gamma}\sum_{t=1}^{\tau}\frac{b_{t}^{2}-b_{t-1}^{2}}{b_{t}^{2\delta}b_{t-1}^{2-2\delta}}\\
 & =\frac{\eta}{\gamma}\sum_{t=1}^{\tau}\left(\frac{b_{t}}{b_{t-1}}\right)^{2-2\delta}\frac{b_{t}^{2}-b_{t-1}^{2}}{b_{t}^{2}}
\end{align*}
Because $b_{t}^{\delta}b_{t-1}^{1-\delta}=c_{t}\leq2\eta L/\gamma$
for $t\leq\tau$, so we know $b_{t}\leq\left(\frac{2\eta L}{\gamma b_{t-1}^{1-\delta}}\right)^{1/\delta}$.
Using this bound
\begin{align*}
\frac{\eta}{\gamma}\sum_{t=1}^{\tau}\left(\frac{b_{t}}{b_{t-1}}\right)^{2-2\delta}\frac{b_{t}^{2}-b_{t-1}^{2}}{b_{t}^{2}} & \leq\frac{\eta}{\gamma}\sum_{t=1}^{\tau}\left(\frac{2\eta L}{\gamma b_{t-1}}\right)^{\frac{2}{\delta}-2}\frac{b_{t}^{2}-b_{t-1}^{2}}{b_{t}^{2}}\\
 & \leq\frac{\eta}{\gamma}\left(\frac{2\eta L}{\gamma b_{0}}\right)^{\frac{2}{\delta}-2}\sum_{t=1}^{\tau}\frac{b_{t}^{2}-b_{t-1}^{2}}{b_{t}^{2}}\\
 & \leq\frac{2\eta}{\gamma}\left(\frac{2\eta L}{\gamma b_{0}}\right)^{\frac{2}{\delta}-2}\log\frac{b_{t}}{b_{0}}\\
 & \leq\frac{2\eta}{\gamma\delta}\left(\frac{2\eta L}{\gamma b_{0}}\right)^{\frac{2}{\delta}-2}\log^{+}\frac{2\eta L}{\gamma b_{0}}.
\end{align*}
The proof is completed.
\end{proof}

As before, our last task is to bound $c_{T}$. It is enough to bound
$b_{T}$ since $c_{T}\leq b_{T}$.
\begin{lem}
\label{lem:Appendix-AdaGradNorm-Last2-b_T}Suppose all the conditions
in Lemma \ref{lem:Appendix-last-key-lemma} are satisfied by replacing
$c_{t}$ by $b_{t}$, additionally, assume $p_{t}\leq1$, we will
have
\[
b_{T}\leq b_{0}\exp\left(\frac{\frac{\|x_{1}-x^{*}\|^{2}}{\gamma\eta^{2}}+\frac{\eta L}{b_{0}}\left(1-\left(\frac{b_{0}}{\eta L}\right)^{\frac{1}{\delta}}\right)^{+}+\frac{2}{\gamma\delta}\left(\frac{2\eta L}{\gamma b_{0}}\right)^{\frac{2}{\delta}-2}\log^{+}\frac{2\eta L}{\gamma b_{0}}}{1-\delta}\right).
\]
\end{lem}
\begin{proof}
By Lemma \ref{lem:Appendix-last-key-lemma}, we know
\begin{align*}
0 & \leq\frac{F(x_{T+1})-F^{*}}{p_{T}c_{T}}\\
 & \leq\frac{\|x_{1}-x^{*}\|^{2}}{\gamma\eta}+\sum_{t=1}^{T}\left(\frac{L}{2c_{t}}-\frac{1}{\eta}+\frac{p_{t}}{\gamma\eta}-\frac{p_{t}c_{t}}{2\eta^{2}L}\right)\frac{\eta^{2}\|\nabla F(x_{t})\|^{2}}{c_{t}^{2}p_{t}}\\
0 & \leq\frac{\|x_{1}-x^{*}\|^{2}}{\gamma\eta}+\sum_{t=1}^{T}\left(\frac{L}{2c_{t}}-\frac{1}{2\eta}+\frac{p_{t}}{\gamma\eta}-\frac{p_{t}c_{t}}{2\eta^{2}L}\right)\frac{\eta^{2}\|\nabla F(x_{t})\|^{2}}{c_{t}^{2}p_{t}}-\frac{\eta\|\nabla F(x_{t})\|^{2}}{2c_{t}^{2}p_{t}}\\
\sum_{t=1}^{T}\frac{\eta\|\nabla F(x_{t})\|^{2}}{2c_{t}^{2}p_{t}} & \leq\frac{\|x_{1}-x^{*}\|^{2}}{\gamma\eta}+\sum_{t=1}^{T}\left(\frac{L}{2c_{t}}-\frac{1}{2\eta}+\frac{p_{t}}{\gamma\eta}-\frac{p_{t}c_{t}}{2\eta^{2}L}\right)\frac{\eta^{2}\|\nabla F(x_{t})\|^{2}}{c_{t}^{2}p_{t}}\\
 & \leq\frac{\|x_{1}-x^{*}\|^{2}}{\gamma\eta}+\frac{\eta^{2}L}{b_{0}}\left(1-\left(\frac{b_{0}}{\eta L}\right)^{\frac{1}{\delta}}\right)^{+}+\frac{2\eta}{\gamma\delta}\left(\frac{2\eta L}{\gamma b_{0}}\right)^{\frac{2}{\delta}-2}\log^{+}\frac{2\eta L}{\gamma b_{0}}
\end{align*}
where the last inequality is by Lemma \ref{lem:Appendix-AdaGradNorm-last2-residual}.
Note that for the L.H.S., we have
\begin{align*}
\sum_{t=1}^{T}\frac{\eta\|\nabla F(x_{t})\|^{2}}{2c_{t}^{2}p_{t}} & =\frac{\eta}{2}\sum_{t=1}^{T}\frac{b_{t}^{2}-b_{t-1}^{2}}{b_{t}^{2\delta}b_{t-1}^{2-2\delta}}=\frac{\eta}{2}\sum_{t=1}^{T}\left(\frac{b_{t}}{b_{t-1}}\right)^{2-2\delta}-\left(\frac{b_{t-1}}{b_{t}}\right)^{2\delta}\\
 & \geq\eta(1-\delta)\sum_{t=1}^{T}\log\frac{b_{t}}{b_{t-1}}=\eta(1-\delta)\log\frac{b_{T}}{b_{0}}.
\end{align*}
Hence we know
\[
b_{T}\leq b_{0}\exp\left(\frac{\frac{\|x_{1}-x^{*}\|^{2}}{\gamma\eta^{2}}+\frac{\eta L}{b_{0}}\left(1-\left(\frac{b_{0}}{\eta L}\right)^{\frac{1}{\delta}}\right)^{+}+\frac{2}{\gamma\delta}\left(\frac{2\eta L}{\gamma b_{0}}\right)^{\frac{2}{\delta}-2}\log^{+}\frac{2\eta L}{\gamma b_{0}}}{1-\delta}\right)
\]
\end{proof}

Finally, the proof of Theorem \ref{thm:Main-AdaGradNorm-Last2-rate}
is similar to the proof of Theorem \ref{thm:Main-AdaGradNorm-Last1-rate},
hence, which is omitted.

\subsection{An asymptotic rate when $\Delta=0$ and $\delta=1$\label{subsec:Appendix-Last-Delta=00003D0-asy}}

As mentioned before, by setting $\Delta=0$ in Algorithm \ref{alg:AdaGradNorm-Last1}
and $\delta=1$ in Algorithm \ref{alg:AdaGradNorm-Last2} we obtain
the same algorithm. The square root update rule of $b_{t}$ and the
step size now are both more similar to the original AdaGradNorm. Intuitively,
we can also expect the convergence of the last iterate in this case;
furthermore, by taking the limit when $\Delta\to0$ and $\delta\to1$,
we can have a sense of the exponential dependency of the provable
convergence rate on the problem parameters. However, previous analysis
strictly requires that $\Delta>0$ and $\delta<1$, thus does not
apply here.

In this section, we partially confirm the convergence of this variant
by proving an asymptotic rate, i.e., $F(x_{T+1})-F^{*}=O\left(1/T\right)$.
Unfortunately, under Assumptions 1 and 2', we cannot figure out the
explicit dependency of the convergence rate on the problem parameters.
However, in the next section, we will give an explicit rate by replacing
Assumption 1 with the stronger Assumption 1'. As stated, our goal
is to prove Theorem \ref{thm:Appendix-last-asy-rate} in this section.
\begin{thm}
\label{thm:Appendix-last-asy-rate}Suppose $F$ satisfies Assumptions
1 and 2', when $\Delta=0$ for Algorithm \ref{alg:AdaGradNorm-Last1},
or equivalently, $\delta=1$ for Algorithm \ref{alg:AdaGradNorm-Last2},
by taking $p_{t}=\frac{1}{t}$, we have
\begin{align*}
F(x_{T+1})-F^{*} & =O\left(1/T\right).
\end{align*}
\end{thm}
Before starting the proof, we first discuss why we can obtain only
an asymptotic rate when $\Delta=0$ and $\delta=1$. As before, one
can still expect that $F(x_{T+1})-F^{*}\leq\frac{b_{T}C}{T}$ remains
true for some constant $C$. However, a critical difference will show
up when we want to find an explicit upper bound on $b_{T}$. Using
the proof of Lemma \ref{lem:Appendix-AdaGradNorm-Last1-b_T} as an
example (similarly for the proof of Lemma \ref{lem:Appendix-AdaGradNorm-Last2-b_T}),
one key step is to get $\sum_{t=1}^{T}\frac{\|\nabla F(x_{t})\|^{2}}{b_{t}^{2}p_{t}}=O(1)$,
where in the previous analysis, by replacing $\frac{\|\nabla F(x_{t})\|^{2}}{p_{t}}$
by $b_{t}^{2+\Delta}-b_{t-1}^{2+\Delta}$ with $\Delta>0$, we can
lower bound $\sum_{t=1}^{T}\frac{\|\nabla F(x_{t})\|^{2}}{b_{t}^{2}p_{t}}$
by a function of $b_{T}$ and finally give an explicit bound on $b_{T}$.
However, this is not possible when $\Delta=0$ as $\sum_{t=1}^{T}\frac{\|\nabla F(x_{t})\|^{2}}{b_{t}^{2}p_{t}}=\sum_{t=1}^{T}\frac{b_{t}^{2}-b_{t-1}^{2}}{b_{t}^{2}}$.
The only information we can get from $\sum_{t=1}^{T}\frac{b_{t}^{2}-b_{t-1}^{2}}{b_{t}^{2}}=O(1)$
is $\lim_{T\to\infty}\frac{b_{T-1}^{2}}{b_{T}^{2}}=1$. This is not
enough to tell us whether $b_{T}$ is upper bounded or not. In Lemma
\ref{lem:last-asy-b_T}, we will use a new argument to finally show
that $\lim_{T\to\infty}b_{T}<\infty$, which leads to an asymptotic
rate as desired. It is worth pointing out that finding an asymptotic
without explicit dependency on the problem parameters is the approach
used in some of the previous work, such as \cite{antonakopoulos2022undergrad}.
This also gives us a glimpse of the method used to analyze the convergence
of the accelerated methods in Section \ref{sec:Main-Acc}.

Now we start the proof. As before, we can employ Lemma \ref{lem:Appendix-last-key-lemma}.
Hence we only need to bound the residual terms as following
\begin{lem}
\label{lem:Appendix-last-asy-residual}Suppose $p_{t}\leq1$, when
$\Delta=0$ for Algorithm \ref{alg:AdaGradNorm-Last1}, or equivalently,
$\delta=1$ for Algorithm \ref{alg:AdaGradNorm-Last2}, we have
\begin{align*}
\sum_{t=1}^{T}\left(\frac{L}{2b_{t}}-\frac{1}{2\eta}\right)\frac{\eta^{2}\|\nabla F(x_{t})\|^{2}}{b_{t}^{2}p_{t}} & \leq\eta\left(\frac{\eta L}{b_{0}}-1\right)^{+}\\
\sum_{t=1}^{T}\left(\frac{p_{t}}{\gamma\eta}-\frac{p_{t}b_{t}}{2\eta^{2}L}\right)\frac{\eta^{2}\|\nabla F(x_{t})\|^{2}}{b_{t}^{2}p_{t}} & \leq\frac{2\eta}{\gamma}\log^{+}\frac{2\eta L}{\gamma b_{0}}
\end{align*}
\end{lem}
The proof is essentially similar to the proof of Lemmas \ref{lem:Appendix-AdaGradNorm-last1-residual}
and \ref{lem:Appendix-AdaGradNorm-last2-residual}, hence we omit
it here.
\begin{lem}
\label{lem:last-asy-b_T}Suppose all the conditions in Lemma \ref{lem:Appendix-last-key-lemma}
are satisfied by replacing $c_{t}$ by $b_{t}$, then when $\Delta=0$
for Algorithm \ref{alg:AdaGradNorm-Last1}, or equivalently, $\delta=1$
for Algorithm \ref{alg:AdaGradNorm-Last2}, we have
\[
\lim_{T\to\infty}b_{T}=b_{\infty}<\infty.
\]
\end{lem}
\begin{proof}
First note that $b_{t}$ is increasing, by the Monotone convergence
theorem, we know $\lim_{T\to\infty}b_{T}=b_{\infty}$ exists. We aim
to show $b_{\infty}<\infty$. By Lemma \ref{lem:Appendix-last-key-lemma}
and replacing $c_{t}$ by $b_{t}$, we have
\begin{align*}
 & \frac{F(x_{T+1})-F^{*}}{p_{T}b_{T}}\\
\le & \frac{\|x_{1}-x^{*}\|^{2}}{\gamma\eta}+\sum_{t=1}^{T}\left(\frac{L}{2b_{t}}-\frac{1}{\eta}+\frac{p_{t}}{\gamma\eta}-\frac{p_{t}b_{t}}{2\eta^{2}L}\right)\frac{\eta^{2}\|\nabla F(x_{t})\|^{2}}{b_{t}^{2}p_{t}}\\
= & \frac{\|x_{1}-x^{*}\|^{2}}{\gamma\eta}+\sum_{t=1}^{T}\left(\frac{L}{2b_{t}}-\frac{1}{2\eta}+\frac{p_{t}}{\gamma\eta}-\frac{p_{t}b_{t}}{2\eta^{2}L}\right)\frac{\eta^{2}\|\nabla F(x_{t})\|^{2}}{b_{t}^{2}p_{t}}-\frac{1}{2\eta}\times\frac{\eta^{2}\|\nabla F(x_{t})\|^{2}}{b_{t}^{2}p_{t}}\\
= & \frac{\|x_{1}-x^{*}\|^{2}}{\gamma\eta}+\sum_{t=1}^{T}\left(\frac{L}{2b_{t}}-\frac{1}{2\eta}+\frac{p_{t}}{\gamma\eta}-\frac{p_{t}b_{t}}{2\eta^{2}L}\right)\frac{\eta^{2}\|\nabla F(x_{t})\|^{2}}{b_{t}^{2}p_{t}}-\frac{\eta\|\nabla F(x_{t})\|^{2}}{2b_{t}^{2}p_{t}}\\
\le & \frac{\|x_{1}-x^{*}\|^{2}}{\gamma\eta}+\eta\left(\frac{\eta L}{b_{0}}-1\right)^{+}+\frac{2\eta}{\gamma}\log^{+}\frac{2\eta L}{\gamma b_{0}}-\sum_{t=1}^{T}\frac{\eta\|\nabla F(x_{t})\|^{2}}{2b_{t}^{2}p_{t}},
\end{align*}
where the last inequality is by Lemma \ref{lem:Appendix-last-asy-residual}.
Noticing $F(x_{T+1})-F^{*}\geq0$, we know
\[
\sum_{t=1}^{T}\frac{\eta\|\nabla F(x_{t})\|^{2}}{2b_{t}^{2}p_{t}}\leq\frac{\|x_{1}-x^{*}\|^{2}}{\gamma\eta}+\eta\left(\frac{\eta L}{b_{0}}-1\right)^{+}+\frac{2\eta}{\gamma}\log^{+}\frac{2\eta L}{\gamma b_{0}},
\]
which implies
\begin{equation}
\sum_{t=1}^{\infty}\frac{\|\nabla F(x_{t})\|^{2}}{b_{t}^{2}p_{t}}\leq\frac{2\|x_{1}-x^{*}\|^{2}}{\gamma\eta^{2}}+2\left(\frac{\eta L}{b_{0}}-1\right)^{+}+\frac{4}{\gamma}\log^{+}\frac{2\eta L}{\gamma b_{0}}.\label{eq:sum-bound}
\end{equation}
We observe that
\begin{align*}
b_{T}^{2} & =b_{T-1}^{2}+\frac{\|\nabla F(x_{T})\|^{2}}{p_{T}}\\
\Rightarrow & b_{T}^{2}=\frac{b_{T-1}^{2}}{1-\frac{\|\nabla F(x_{T})\|^{2}}{b_{T}^{2}p_{T}}}=b_{0}^{2}\prod_{t=1}^{T}\frac{1}{1-\frac{\|\nabla F(x_{t})\|^{2}}{b_{t}^{2}p_{t}}}.
\end{align*}
Taking $\log$ to both sides, we get
\begin{align*}
\log b_{T}^{2} & =\log b_{0}^{2}+\sum_{t=1}^{T}\log\frac{1}{1-\frac{\|\nabla F(x_{t})\|^{2}}{b_{t}^{2}p_{t}}}\leq\log b_{0}^{2}+\sum_{t=1}^{T}\frac{1}{1-\frac{\|\nabla F(x_{t})\|^{2}}{b_{t}^{2}p_{t}}}-1\\
 & =\log b_{0}^{2}+\sum_{t=1}^{T}\frac{\frac{\|\nabla F(x_{t})\|^{2}}{b_{t}^{2}p_{t}}}{1-\frac{\|\nabla F(x_{t})\|^{2}}{b_{t}^{2}p_{t}}}\leq\log b_{0}^{2}+\sum_{t=1}^{\infty}\frac{\frac{\|\nabla F(x_{t})\|^{2}}{b_{t}^{2}p_{t}}}{1-\frac{\|\nabla F(x_{t})\|^{2}}{b_{t}^{2}p_{t}}}
\end{align*}
Note that Inequality (\ref{eq:sum-bound}) tells us $\lim_{t\to\infty}\frac{\|\nabla F(x_{t})\|^{2}}{b_{t}^{2}p_{t}}=0$,
hence we can let $\tau$ be the time such that $\frac{\|\nabla F(x_{t})\|^{2}}{b_{t}^{2}p_{t}}\leq\frac{1}{2}$
for $t\geq\tau$. Then we know
\begin{align*}
\sum_{t=1}^{\infty}\frac{\frac{\|\nabla F(x_{t})\|^{2}}{b_{t}^{2}p_{t}}}{1-\frac{\|\nabla F(x_{t})\|^{2}}{b_{t}^{2}p_{t}}} & =\sum_{t=1}^{\tau-1}\frac{\frac{\|\nabla F(x_{t})\|^{2}}{b_{t}^{2}p_{t}}}{1-\frac{\|\nabla F(x_{t})\|^{2}}{b_{t}^{2}p_{t}}}+\sum_{t=\tau}^{\infty}\frac{\frac{\|\nabla F(x_{t})\|^{2}}{b_{t}^{2}p_{t}}}{1-\frac{\|\nabla F(x_{t})\|^{2}}{b_{t}^{2}p_{t}}}\\
 & \leq\sum_{t=1}^{\tau-1}\frac{\frac{\|\nabla F(x_{t})\|^{2}}{b_{t}^{2}p_{t}}}{1-\frac{\|\nabla F(x_{t})\|^{2}}{b_{t}^{2}p_{t}}}+2\sum_{t=\tau}^{\infty}\frac{\|\nabla F(x_{t})\|^{2}}{b_{t}^{2}p_{t}}\\
 & \leq\sum_{t=1}^{\tau-1}\frac{\frac{\|\nabla F(x_{t})\|^{2}}{b_{t}^{2}p_{t}}}{1-\frac{\|\nabla F(x_{t})\|^{2}}{b_{t}^{2}p_{t}}}+2\left(\frac{2\|x_{1}-x^{*}\|^{2}}{\gamma\eta^{2}}+2\left(\frac{\eta L}{b_{0}}-1\right)^{+}+\frac{4}{\gamma}\log^{+}\frac{2\eta L}{\gamma b_{0}}\right)\\
 & <\infty.
\end{align*}
The above result implies $\log b_{T}^{2}$ has a uniform upper bound
which means $b_{\infty}<\infty$.
\end{proof}

Now we can start to prove Theorem \ref{thm:Appendix-last-asy-rate}.

\begin{proof}
Note that when $\Delta=0$ for Algorithm \ref{alg:AdaGradNorm-Last1},
or equivalently, $\delta=1$ for Algorithm \ref{alg:AdaGradNorm-Last2},
if $p_{t}=\frac{1}{t}$, all the conditions in Lemma \ref{lem:Appendix-last-key-lemma}
are satisfied by replacing $c_{t}$ by $b_{t}$. Hence we have
\begin{align*}
\frac{F(x_{T+1})-F^{*}}{p_{T}b_{T}} & \leq\frac{\|x_{1}-x^{*}\|^{2}}{\gamma\eta}+\sum_{t=1}^{T}\left(\frac{L}{2b_{t}}-\frac{1}{\eta}+\frac{p_{t}}{\gamma\eta}-\frac{p_{t}b_{t}}{2\eta^{2}L}\right)\frac{\eta^{2}\|\nabla F(x_{t})\|^{2}}{b_{t}^{2}p_{t}}\\
 & \leq\frac{\|x_{1}-x^{*}\|^{2}}{\gamma\eta}+\sum_{t=1}^{T}\left(\frac{L}{2b_{t}}-\frac{1}{2\eta}+\frac{p_{t}}{\gamma\eta}-\frac{p_{t}b_{t}}{2\eta^{2}L}\right)\frac{\eta^{2}\|\nabla F(x_{t})\|^{2}}{b_{t}^{2}p_{t}}\\
 & =\frac{\|x_{1}-x^{*}\|^{2}}{\gamma\eta}+\eta\left(\frac{\eta L}{b_{0}}-1\right)^{+}+\frac{2\eta}{\gamma}\log^{+}\frac{2\eta L}{\gamma b_{0}},
\end{align*}
where the last inequality is by Lemma \ref{lem:Appendix-last-asy-residual}.
Multiplying both sides by$p_{T}b_{T}$, we wknow
\begin{align*}
F(x_{T+1})-F^{*} & \leq\frac{b_{T}}{T}\left(\frac{\|x_{1}-x^{*}\|^{2}}{\gamma\eta}+\eta\left(\frac{\eta L}{b_{0}}-1\right)^{+}+\frac{2\eta}{\gamma}\log^{+}\frac{2\eta L}{\gamma b_{0}}\right)\\
 & \leq\frac{b_{\infty}}{T}\left(\frac{\|x_{1}-x^{*}\|^{2}}{\gamma\eta}+\eta\left(\frac{\eta L}{b_{0}}-1\right)^{+}+\frac{2\eta}{\gamma}\log^{+}\frac{2\eta L}{\gamma b_{0}}\right)\\
 & =O\left(\frac{1}{T}\right),
\end{align*}
where the last line is by Lemma \ref{lem:last-asy-b_T}.
\end{proof}

\subsection{A non-asymptotic rate when $\Delta=0$ and $\delta=1$ for convex
smooth functions\label{subsec:Appendix-Last-Delta=00003D0-non-asy}}

In the previous section, we only give an asymptotic rate when $\Delta=0$
and $\delta=1$. In the following, we will show that, by replacing
Assumption 1 by the stronger Assumption 1', a non-asymptotic rate
can be obtained as stated in Theorem \ref{thm:Appendix-last-non-asy-rate}.
\begin{thm}
\label{thm:Appendix-last-non-asy-rate}Suppose $F$ satisfies Assumptions
1' and 2', when $\Delta=0$ for Algorithm \ref{alg:AdaGradNorm-Last1},
or equivalently, $\delta=1$ for Algorithm \ref{alg:AdaGradNorm-Last2},
by taking $p_{t}=\frac{1}{t}$, we have
\begin{align*}
F(x_{T+1})-F^{*} & \leq\frac{b\left(\frac{\|x_{1}-x^{*}\|^{2}}{2\eta}+\frac{\eta}{2}\left(\frac{2\eta L}{b_{0}}-1\right)^{+}\right)}{T},
\end{align*}
where $b=\max\left\{ \frac{\eta L}{2},\sqrt{b_{0}^{2}+\|\nabla F(x_{1})\|^{2}}\exp\left(\frac{3\|x_{1}-x^{*}\|^{2}}{\eta^{2}}+3\left(\frac{2\eta L}{b_{0}}-1\right)^{+}\right),\right.$
$\left.\eta L\sqrt{\frac{1}{4}+\frac{\|x_{t}-x^{*}\|^{2}}{\eta^{2}}+\left(\frac{2\eta L}{b_{0}}-1\right)^{+}}\exp\left(\frac{3\|x_{1}-x^{*}\|^{2}}{\eta^{2}}+3\left(\frac{2\eta L}{b_{0}}-1\right)^{+}\right)\right\} $.
\end{thm}
We first give another well-known characterization of convex and $L$-smooth
functions without proof.
\begin{lem}
\label{lem:Appendix-cocoercivity}Suppose $F$ satisfies Assumption
1' and 2', then $\forall x,y\in\R^{d}$
\[
\langle\nabla F(x)-\nabla F(y),x-y\rangle\geq\frac{\|\nabla F(x)-\nabla F(y)\|^{2}}{L}.
\]
\end{lem}
Next, we state a simple variant of Lemma \ref{lem:Appendix-last-key-lemma},
the proof of which is essentially the same as the proof of Lemma \ref{lem:Appendix-last-key-lemma},
hence we omit it.
\begin{lem}
\label{lem:Appendix-last-key-lemma-changed}Suppose the following
conditions hold:
\begin{itemize}
\item $F$ satisfies Assumptions 1' and 2';
\item $p_{t}\in(0,1)$ satisfies $\frac{1}{p_{t}}\geq\frac{1-p_{t+1}}{p_{t+1}},p_{1}=1$.
\end{itemize}
When $\Delta=0$ for Algorithm \ref{alg:AdaGradNorm-Last1}, or equivalently,
$\delta=1$ for Algorithm \ref{alg:AdaGradNorm-Last2}, we have
\[
\frac{F(x_{T+1})-F^{*}}{p_{T}b_{T}}\leq\frac{\|x_{1}-x^{*}\|^{2}}{2\eta}+\sum_{t=1}^{T}\left(\frac{L}{2b_{t}}-\frac{1}{\eta}+\frac{p_{t}}{2\eta}\right)\frac{\eta^{2}\|\nabla F(x_{t})\|^{2}}{b_{t}^{2}p_{t}}
\]
\end{lem}
The same as Lemma \ref{lem:Appendix-last-asy-residual}, we give the
following bound on the residual term without proof.
\begin{lem}
\label{lem:Appendix-last-non-asy-residual}Suppose $p_{t}\leq1$,
when $\Delta=0$ for Algorithm \ref{alg:AdaGradNorm-Last1}, or equivalently,
$\delta=1$ for Algorithm \ref{alg:AdaGradNorm-Last2}, we have
\begin{align*}
\sum_{t=1}^{T}\left(\frac{L}{2b_{t}}-\frac{1}{4\eta}\right)\frac{\eta^{2}\|\nabla F(x_{t})\|^{2}}{b_{t}^{2}p_{t}} & \leq\frac{\eta}{2}\left(\frac{2\eta L}{b_{0}}-1\right)^{+}
\end{align*}
\end{lem}
Again, the above two lemmas give us 
\begin{equation}
F(x_{T+1})-F^{*}\leq p_{T}b_{T}\left(\frac{\|x_{1}-x^{*}\|^{2}}{2\eta}+\frac{\eta}{2}\left(\frac{2\eta L}{b_{0}}-1\right)^{+}\right)\label{eq:last-non-asy-rate-by-b_T}
\end{equation}
W.l.o.g., we assume $b_{T}>\frac{\eta L}{2}$ in the following analysis.
Otherwise, we can use the bound $b_{T}\leq\frac{\eta L}{2}$ to get
a trivial convergence rate. Now we define the time 
\[
\tau=\max\left\{ t\in\left[T\right],b_{t}\leq\frac{\eta L}{2}\right\} \lor0.
\]
This time $\tau$ is extremly useful and will finally help us bound
$b_{T}$. Now we list the following three important lemmas related
to time $\tau$. 
\begin{lem}
\label{lem:last-non-asy-decreasing}With Assumptions 1' and 2', when
$t\geq\tau+1$, $\|\nabla F(x_{t})\|$ is non-increasing.
\end{lem}
\begin{proof}
Taking $x=x_{t}$,$y=x_{t+1}$ in Lemma \ref{lem:Appendix-cocoercivity},
we get
\begin{align*}
\frac{\|\nabla F(x_{t})-\nabla F(x_{t+1})\|^{2}}{L} & \leq\langle\nabla F(x_{t})-\nabla F(x_{t+1}),x_{t}-x_{t+1}\rangle\\
 & =\langle\nabla F(x_{t})-\nabla F(x_{t+1}),\frac{\eta}{b_{t}}\nabla F(x_{t})\rangle\\
\Rightarrow\left(\frac{1}{L}-\frac{\eta}{b_{t}}\right)\|\nabla F(x_{t})\|^{2}+\frac{1}{L}\|\nabla F(x_{t+1})\|^{2} & \leq\left(\frac{2}{L}-\frac{\eta}{b_{t}}\right)\langle\nabla F(x_{t}),\nabla F(x_{t+1})\rangle.
\end{align*}
Note that when $t\geq\tau+1$, we know $b_{t}>\frac{\eta L}{2}\Rightarrow\frac{2}{L}-\frac{\eta}{b_{t}}>0$,
hence we have
\begin{align*}
 & \left(\frac{1}{L}-\frac{\eta}{b_{t}}\right)\|\nabla F(x_{t})\|^{2}+\frac{1}{L}\|\nabla F(x_{t+1})\|^{2}\\
\leq & \left(\frac{2}{L}-\frac{\eta}{b_{t}}\right)\langle\nabla F(x_{t}),\nabla F(x_{t+1})\rangle\\
\leq & \left(\frac{1}{L}-\frac{\eta}{2b_{t}}\right)\|\nabla F(x_{t})\|^{2}+\left(\frac{1}{L}-\frac{\eta}{2b_{t}}\right)\|\nabla F(x_{t+1})\|^{2},
\end{align*}
which implies $\|\nabla F(x_{t+1})\|^{2}\leq\|\nabla F(x_{t})\|^{2}$.
This is just what we want.
\end{proof}

\begin{lem}
\label{lem:last-non-asy-bound-away}With Assumptions 1' and 2', if
$p_{t}=\frac{1}{t}$, when $t\geq\tau+2\geq2$,
\[
\frac{\|\nabla F(x_{t})\|^{2}}{b_{t}^{2}p_{t}}\leq\frac{2}{3}
\]
\end{lem}
\begin{proof}
This is because
\begin{align*}
\frac{\|\nabla F(x_{t})\|^{2}}{b_{t}^{2}p_{t}} & =\frac{t\|\nabla F(x_{t})\|^{2}}{b_{0}^{2}+\sum_{i=1}^{t}i\|\nabla F(x_{i})\|^{2}}\\
 & \leq\frac{t\|\nabla F(x_{t})\|^{2}}{(t-1)\|\nabla F(x_{t-1})\|^{2}+t\|\nabla F(x_{t})\|^{2}}\\
 & \leq\frac{t\|\nabla F(x_{t})\|^{2}}{(t-1)\|\nabla F(x_{t})\|^{2}+t\|\nabla F(x_{t})\|^{2}}\\
 & =\frac{t}{2t-1},
\end{align*}
where the last inequality is because $t-1\geq\tau+1$, hence $\|\nabla F(x_{t-1})\|\geq\|\nabla F(x_{t})\|$
by Lemma \ref{lem:last-non-asy-decreasing}. Note that $t\geq2$,
so $\frac{\|\nabla F(x_{t})\|^{2}}{b_{t}^{2}p_{t}}\leq\frac{t}{2t-1}\leq\frac{2}{3}.$
\end{proof}

\begin{lem}
\label{lem:last-non-asy-b_tau+1}With Assumptions 1' and 2', if $p_{t}=\frac{1}{t}$
\[
b_{\tau+1}\leq\sqrt{b_{0}^{2}+\|\nabla F(x_{1})\|^{2}}\lor\eta L\sqrt{\frac{1}{4}+\frac{\|x_{1}-x^{*}\|^{2}}{\eta^{2}}+\left(\frac{2\eta L}{b_{0}}-1\right)^{+}}
\]
\end{lem}
\begin{proof}
If $\tau=0$, we have
\[
b_{\tau+1}=b_{1}=\sqrt{b_{0}^{2}+\|\nabla F(x_{1})\|^{2}}.
\]
Otherwise, we know $\tau+1\geq2$, hence
\begin{align*}
b_{\tau+1}^{2} & =b_{\tau}^{2}+\left(\tau+1\right)\|\nabla F(x_{\tau+1})\|^{2}\\
 & \leq b_{\tau}^{2}+2L\left(\tau+1\right)(F(x_{\tau+1})-F^{*})\\
 & \leq b_{\tau}^{2}+2L\frac{\tau+1}{\tau}b_{\tau}\left(\frac{\|x_{1}-x^{*}\|^{2}}{2\eta}+\frac{\eta}{2}\left(\frac{2\eta L}{b_{0}}-1\right)^{+}\right)\\
 & \leq b_{\tau}^{2}+4Lb_{\tau}\left(\frac{\|x_{1}-x^{*}\|^{2}}{2\eta}+\frac{\eta}{2}\left(\frac{2\eta L}{b_{0}}-1\right)^{+}\right)\\
 & \leq\left(\frac{\eta L}{2}\right)^{2}+\eta^{2}L^{2}\left(\frac{\|x_{1}-x^{*}\|^{2}}{\eta^{2}}+\left(\frac{2\eta L}{b_{0}}-1\right)^{+}\right)\\
\Rightarrow b_{\tau+1} & \leq\eta L\sqrt{\frac{1}{4}+\frac{\|x_{1}-x^{*}\|^{2}}{\eta^{2}}+\left(\frac{2\eta L}{b_{0}}-1\right)^{+}}
\end{align*}
where the second inequality is due to (\ref{eq:last-non-asy-rate-by-b_T}).
\end{proof}

Now we combine Lemmas \ref{lem:last-non-asy-bound-away} and \ref{lem:last-non-asy-b_tau+1}
to get an upper bound for $b_{T}$.
\begin{lem}
\label{lem:Appendix-last-non-asy-b_T}With Assumptions 1' and 2',
if $p_{t}=\frac{1}{t}$
\begin{align*}
b_{T} & \leq\max\left\{ \frac{\eta L}{2},\sqrt{b_{0}^{2}+\|\nabla F(x_{1})\|^{2}}\exp\left(\frac{3\|x_{1}-x^{*}\|^{2}}{\eta^{2}}+3\left(\frac{2\eta L}{b_{0}}-1\right)^{+}\right),\right.\\
 & \left.\eta L\sqrt{\frac{1}{4}+\frac{\|x_{1}-x^{*}\|^{2}}{\eta^{2}}+\left(\frac{2\eta L}{b_{0}}-1\right)^{+}}\exp\left(\frac{3\|x_{1}-x^{*}\|^{2}}{\eta^{2}}+3\left(\frac{2\eta L}{b_{0}}-1\right)^{+}\right)\right\} 
\end{align*}
\end{lem}
\begin{proof}
Note that if $b_{T}\leq\frac{\eta L}{2}$, we are done. If $b_{T}>\frac{\eta L}{2}$,
we will bound $b_{T}$ as follows:
\begin{align*}
b_{T}^{2} & =b_{T-1}^{2}+\frac{\|\nabla F(x_{T})\|^{2}}{p_{T}}\\
\Rightarrow b_{T}^{2} & =\frac{b_{T-1}^{2}}{1-\frac{\|\nabla F(x_{T})\|^{2}}{b_{T}^{2}p_{T}}}=b_{\tau+1}^{2}\prod_{t=\tau+2}^{T}\frac{1}{1-\frac{\|\nabla F(x_{t})\|^{2}}{b_{t}^{2}p_{t}}}\\
\Rightarrow\log b_{T}^{2} & \leq\log b_{\tau+1}^{2}+\sum_{t=\tau+2}^{T}\log\frac{1}{1-\frac{\|\nabla F(x)\|^{2}}{b_{t}^{2}p_{t}}}\\
 & \leq\log b_{\tau+1}^{2}+\sum_{t=\tau+2}^{T}\frac{\frac{\|\nabla F(x)\|^{2}}{b_{t}^{2}p_{t}}}{1-\frac{\|\nabla F(x)\|^{2}}{b_{t}^{2}p_{t}}}\\
 & \leq\log b_{\tau+1}^{2}+\sum_{t=\tau+2}^{T}\frac{3\|\nabla F(x)\|^{2}}{b_{t}^{2}p_{t}},
\end{align*}
where the last inequality is by Lemma \ref{lem:last-non-asy-bound-away}.
Noticing $p_{t}=\frac{1}{t}\leq1$, combining Lemmas \ref{lem:Appendix-last-key-lemma-changed}
and \ref{lem:Appendix-last-non-asy-residual},we can find
\[
\sum_{t=1}^{T}\frac{\|\nabla F(x_{t})\|^{2}}{b_{t}^{2}p_{t}}\leq\frac{2\|x_{1}-x^{*}\|^{2}}{\eta^{2}}+2\left(\frac{2\eta L}{b_{0}}-1\right)^{+}.
\]
Hence we know
\begin{align*}
\log b_{T}^{2} & \leq\log b_{\tau+1}^{2}+\sum_{t=\tau+2}^{T}\frac{3\|\nabla F(x)\|^{2}}{b_{t}^{2}p_{t}}\\
 & \leq\log b_{\tau+1}^{2}+\frac{6\|x_{1}-x^{*}\|^{2}}{\eta^{2}}+6\left(\frac{2\eta L}{b_{0}}-1\right)^{+}\\
\Rightarrow b_{T}^{2} & \leq b_{\tau+1}^{2}\exp\left(\frac{6\|x_{1}-x^{*}\|^{2}}{\eta^{2}}+6\left(\frac{2\eta L}{b_{0}}-1\right)^{+}\right).
\end{align*}
The last step is to use the bound on $b_{\tau+1}$ in Lemma \ref{lem:last-non-asy-b_tau+1}.
\end{proof}

Finally, the proof of Theorem \ref{thm:Appendix-last-non-asy-rate}
is obtained by simply using Lemma \ref{lem:Appendix-last-non-asy-b_T}
to Equation (\ref{eq:last-non-asy-rate-by-b_T}).

\section{Missing proofs from Section \ref{sec:Main-Acc} \label{sec:Appendix-Acc}}

\subsection{Important lemma}

First, we state a general lemma that can be used for a more general
setting. The proof of the lemma is standard.
\begin{lem}
\label{lem:Appendix-acc-key-lemma}Suppose $F$ satisfies Assumptions
1' and 2' and the following conditions hold:
\begin{itemize}
\item $w_{t}$ is generated by 
\begin{align*}
v_{t} & =(1-a_{t})w_{t}+a_{t}x_{t}\\
x_{t+1} & =x_{t}-\frac{\eta}{q_{t}c_{t}}\nabla F(v_{t})\\
w_{t+1} & =(1-a_{t})w_{t}+a_{t}x_{t+1}
\end{align*}
with $\eta>0$ and $c_{t}>0$ is non-decreasing;
\item $a_{t}\in(0,1]$ and $q_{t}\geq a_{t}$ satisfy $\frac{1}{a_{t}q_{t}}\geq\frac{1-a_{t+1}}{a_{t+1}q_{t+1}},a_{1}=1$;
\end{itemize}
Then we have
\[
\frac{F(w_{T+1})-F^{*}}{a_{T}q_{T}c_{T}}\leq\frac{\|x_{1}-x^{*}\|^{2}}{2\eta}+\sum_{t=1}^{T}\left(\frac{L}{2c_{t}}-\frac{1}{2\eta}\right)\frac{\eta^{2}\|\nabla F(v_{t})\|^{2}}{c_{t}^{2}q_{t}^{2}}
\]
\end{lem}
\begin{proof}
Starting from smoothness

\begin{align*}
 & F(w_{t+1})-F(v_{t})\\
\le & \langle\nabla F(v_{t}),w_{t+1}-v_{t}\rangle+\frac{L}{2}\|w_{t+1}-v_{t}\|^{2}\\
= & (1-a_{t})\langle\nabla F(v_{t}),w_{t}-v_{t}\rangle+a_{t}\langle\nabla F(v_{t}),x_{t+1}-v_{t}\rangle+\frac{L}{2}\|w_{t+1}-v_{t}\|^{2}\\
= & (1-a_{t})\langle\nabla F(v_{t}),w_{t}-v_{t}\rangle+a_{t}\langle\nabla F(v_{t}),x^{*}-v_{t}\rangle+a_{t}\langle\nabla F(v_{t}),x_{t+1}-x^{*}\rangle+\frac{L}{2}\|w_{t+1}-v_{t}\|^{2}\\
\le & (1-a_{t})(F(w_{t})-F(v_{t}))+a_{t}(F^{*}-F(v_{t}))+a_{t}\langle\nabla F(v_{t}),x_{t+1}-x^{*}\rangle+\frac{L}{2}\|w_{t+1}-v_{t}\|^{2}
\end{align*}
Thus 
\begin{align*}
F(w_{t+1})-F^{*} & \leq(1-a_{t})(F(w_{t})-F^{*})+a_{t}\langle\nabla F(v_{t}),x_{t+1}-x^{*}\rangle+\frac{L}{2}\|w_{t+1}-v_{t}\|^{2}
\end{align*}
where the last inequality is due to the convexity of $F$. Using the
update rule $\nabla F(v_{t})=\frac{q_{t}c_{t}}{\eta}\left(x_{t}-x_{t+1}\right)$
and $w_{t+1}-v_{t}=a_{t}\left(x_{t+1}-x_{t}\right)$ we obtain
\begin{align*}
F(w_{t+1})-F^{*} & \leq(1-a_{t})(F(w_{t})-F^{*})\\
 & +\frac{a_{t}q_{t}c_{t}}{2\eta}\left(\left\Vert x^{*}-x_{t}\right\Vert ^{2}-\left\Vert x^{*}-x_{t+1}\right\Vert ^{2}-\left\Vert x_{t+1}-x_{t}\right\Vert ^{2}\right)+\frac{La_{t}^{2}}{2}\|x_{t+1}-x_{t}\|^{2}
\end{align*}
Dividing both sides by $a_{t}q_{t}c_{t}$ and sumning up from $1$
to $T$, we have
\begin{align*}
\sum_{t=1}^{T}\frac{1}{a_{t}q_{t}c_{t}}(F(w_{t+1})-F^{*}) & \leq\sum_{t=1}^{T}\frac{1-a_{t}}{a_{t}q_{t}c_{t}}(F(w_{t})-F^{*})\\
 & +\sum_{t=1}^{T}\left(\frac{La_{t}}{2q_{t}c_{t}}-\frac{1}{2\eta}\right)\|x_{t+1}-x_{t}\|^{2}+\frac{1}{2\eta}\left\Vert x^{*}-x_{1}\right\Vert ^{2}
\end{align*}
Note that $a_{t}\le q_{t}$, $\frac{1}{a_{t}q_{t}c_{t}}\geq\frac{1-a_{t+1}}{a_{t+1}q_{t+1}c_{t+1}}$
and $a_{t}=1$. Thus
\begin{align*}
\frac{F(w_{T+1})-F^{*}}{a_{T}q_{T}c_{T}} & \leq\frac{\left\Vert x_{1}-x^{*}\right\Vert ^{2}}{2\eta}+\sum_{t=1}^{T}\left(\frac{L}{2c_{t}}-\frac{1}{2\eta}\right)\|x_{t+1}-x_{t}\|^{2}\\
 & =\frac{\left\Vert x_{1}-x^{*}\right\Vert ^{2}}{2\eta}+\sum_{t=1}^{T}\left(\frac{L}{2c_{t}}-\frac{1}{2\eta}\right)\frac{\eta^{2}\|\nabla F(v_{t})\|^{2}}{c_{t}^{2}q_{t}^{2}}.
\end{align*}
\end{proof}

\subsection{First variant}

By using Lemma \ref{lem:Appendix-acc-key-lemma}, the proof idea of
Theorem \ref{thm:Main-AdaGradNorm-Acc1-rate} is the same as the proof
of Theorem \ref{thm:Main-AdaGradNorm-Last1-rate}. Hence, we omit
it for brevity.

\subsection{Second variant}

By using Lemma \ref{lem:Appendix-acc-key-lemma}, the proof idea of
Theorem \ref{thm:Main-AdaGradNorm-Acc2-rate} is the same as the proof
of Theorem \ref{thm:Main-AdaGradNorm-Last2-rate}. Hence, we omit
it here.

\subsection{A discussion on when $\Delta=0$ and $\delta=1$\label{subsec:Appendix-acc-Delta=00003D0}}

Algorithms \ref{alg:AdaGradNorm-Acc1} and \ref{alg:AdaGradNorm-Acc2}
become one when $\Delta=0$ and $\delta=1$. As discussed in \ref{subsec:Appendix-Last-Delta=00003D0-asy},
the challenge is to find a explicit bound on $b_{T}$. First, we give
an asymptotic rate in Theorem \ref{thm:acc-asy-rate} of which the
proof idea is the same as the proof of Theorem \ref{thm:Appendix-last-asy-rate},
thus is omitted.
\begin{thm}
\label{thm:acc-asy-rate}Suppose $F$ satisfies Assumptions 1 and
2', when $\Delta=0$ for Algorithm \ref{alg:AdaGradNorm-Acc1}, or
equivalently, $\delta=1$ for Algorithm \ref{alg:AdaGradNorm-Acc2},
by taking $a_{t}=\frac{2}{t+1}$, $p_{t}=\frac{2}{t}$, we have
\begin{align*}
F(w_{T+1})-F^{*} & =O\left(1/T^{2}\right).
\end{align*}
\end{thm}
Now, we aim to prove the following non-asymptotic rate.
\begin{thm}
\label{thm:acc-non-asy-rate}Suppose $F$ satisfies Assumptions 1
and 2', when $\Delta=0$ for Algorithm \ref{alg:AdaGradNorm-Acc1},
or equivalently, $\delta=1$ for Algorithm \ref{alg:AdaGradNorm-Acc2},
by taking $a_{t}=\frac{2}{t+1}$, $p_{t}=\frac{2}{t}$, we have
\begin{align*}
F(w_{T+1})-F^{*} & \leq\frac{4\left(b_{0}+\frac{4\eta^{2}L^{2}}{b_{0}}\right)\left(\frac{\|x_{1}-x^{*}\|^{2}}{2\eta}+\frac{\eta^{2}L}{b_{0}}\log^{+}\frac{\eta L}{b_{0}}\right)}{T(T+1)}+\frac{16L\left(\frac{\|x_{1}-x^{*}\|^{2}}{2\eta}+\frac{\eta^{2}L}{b_{0}}\log^{+}\frac{\eta L}{b_{0}}\right)^{2}}{T+1}.
\end{align*}
\end{thm}
We shortly discuss here why we can only give a rate in the order of
$1/T$ but not $1/T^{2}$. Recall that in the proof of Theorem \ref{thm:Appendix-last-non-asy-rate},
the key step is that after a certain time, $\|\nabla F(x_{t})\|$
is a non-increasing sequence, by using which we can finally give a
constant upper bound on $b_{T}$ that finally helps us to get the
final $1/T$ rate. However, it is unclear under what condition on
$b_{t}$, $\|\nabla F(v_{t})\|$ now will be a non-increasing sequence
in our accelerated algorithm. Thus it is unclear to us whether it
is possible to give a constant bound on $b_{T}$. Instead, we will
show $b_{t}$ can increase at most linearly in this accelerated scheme
by a new trick, for which reason, we can finally obtain the rate in
the order of $1/T$. This guarantees that the convergence of the last
iterate is no worse than the variants in Section \ref{sec:Main-Last}.

\begin{proof}
As before, to start with, we use Lemma \ref{lem:Appendix-acc-key-lemma}
(replace $c_{t}$ by $b_{t}$)
\begin{align*}
\frac{F(w_{T+1})-F^{*}}{a_{T}q_{T}b_{T}} & \leq\frac{\|x_{1}-x^{*}\|^{2}}{2\eta}+\sum_{t=1}^{T}\left(\frac{L}{2b_{t}}-\frac{1}{2\eta}\right)\frac{\eta^{2}\|\nabla F(v_{t})\|^{2}}{b_{t}^{2}q_{t}^{2}}.
\end{align*}
By using $b_{t}^{2}=b_{t-1}^{2}+\frac{\|\nabla F(v_{t})\|^{2}}{b_{t}^{2}}$
and the same technique in the previous proof, we know
\begin{align*}
\sum_{t=1}^{T}\left(\frac{L}{2b_{t}}-\frac{1}{2\eta}\right)\frac{\eta^{2}\|\nabla F(v_{t})\|^{2}}{b_{t}^{2}q_{t}^{2}} & \leq\frac{\eta^{2}L}{b_{0}}\log^{+}\frac{\eta L}{b_{0}}.
\end{align*}
So we have
\[
F(w_{T+1})-F^{*}\leq a_{T}q_{T}b_{T}\underbrace{\left(\frac{\|x_{1}-x^{*}\|^{2}}{2\eta}+\frac{\eta^{2}L}{b_{0}}\log^{+}\frac{\eta L}{b_{0}}\right)}_{D}
\]
Now we turn to bound $b_{t}$ by observing
\begin{align*}
b_{t}^{2} & =b_{t-1}^{2}+\frac{\|\nabla F(v_{t})\|^{2}}{q_{t}^{2}}\\
 & \leq b_{t-1}^{2}+\frac{2\|\nabla F(v_{t})-\nabla F(w_{t+1})\|^{2}}{q_{t}^{2}}+\frac{2\|\nabla F(w_{t+1})\|^{2}}{q_{t}^{2}}\\
 & \leq b_{t-1}^{2}+\frac{2L^{2}\|v_{t}-w_{t+1}\|^{2}}{q_{t}^{2}}+\frac{2\|\nabla F(w_{t+1})\|^{2}}{q_{t}^{2}}\\
 & =b_{t-1}^{2}+\frac{2L^{2}a_{t}^{2}\|x_{t+1}-x_{t}\|^{2}}{q_{t}^{2}}+\frac{2\|\nabla F(w_{t+1})\|^{2}}{q_{t}^{2}}\\
 & \leq b_{t-1}^{2}+2L^{2}\|x_{t+1}-x_{t}\|^{2}+\frac{2\|\nabla F(w_{t+1})\|^{2}}{q_{t}^{2}},
\end{align*}
where the last inequality is due to $a_{t}\leq p_{t}$. Then we use
$\|x_{t+1}-x_{t}\|^{2}=\frac{\eta^{2}\|\nabla F(v_{t})\|^{2}}{b_{t}^{2}q_{t}^{2}}=\frac{\eta^{2}(b_{t}^{2}-b_{t-1}^{2})}{b_{t}^{2}}$
and $\|\nabla F(w_{t+1})\|^{2}\leq2L(F(w_{t+1})-F^{*})\leq2La_{t}q_{t}b_{t}D$
to get
\begin{align*}
b_{t}^{2} & \leq b_{t-1}^{2}+\frac{2\eta^{2}L^{2}\left(b_{t}^{2}-b_{t-1}^{2}\right)}{b_{t}^{2}}+\frac{4La_{t}q_{t}b_{t}D}{q_{t}^{2}}\\
 & \leq b_{t-1}^{2}+\frac{2\eta^{2}L^{2}\left(b_{t}^{2}-b_{t-1}^{2}\right)}{b_{t}^{2}}+4Lb_{t}D\\
\Rightarrow b_{t} & \leq\frac{b_{t-1}^{2}}{b_{t}}+2\eta^{2}L^{2}\frac{b_{t}^{2}-b_{t-1}^{2}}{b_{t}^{3}}+4LD\\
 & \leq b_{t-1}+4\eta^{2}L^{2}\left(\frac{1}{b_{t-1}}-\frac{1}{b_{t}}\right)+4LD\\
\Rightarrow b_{t} & \leq b_{0}+\frac{4\eta^{2}L^{2}}{b_{0}}+4LDt.
\end{align*}
Using this bound, we finally get
\begin{align*}
F(w_{T+1})-F^{*} & \leq a_{T}q_{T}b_{T}D\\
 & =\frac{4\left(b_{0}+\frac{4\eta^{2}L^{2}}{b_{0}}+4LDT\right)D}{T(T+1)}\\
 & =\frac{4\left(b_{0}+\frac{4\eta^{2}L^{2}}{b_{0}}\right)\left(\frac{\|x_{1}-x^{*}\|^{2}}{2\eta}+\frac{\eta^{2}L}{b_{0}}\log^{+}\frac{\eta L}{b_{0}}\right)}{T(T+1)}+\frac{16L\left(\frac{\|x_{1}-x^{*}\|^{2}}{2\eta}+\frac{\eta^{2}L}{b_{0}}\log^{+}\frac{\eta L}{b_{0}}\right)^{2}}{T+1}.
\end{align*}
\end{proof}

\section{Experiments}

\begin{figure}[h]
\centering{}\subfloat{\includegraphics[width=1\textwidth]{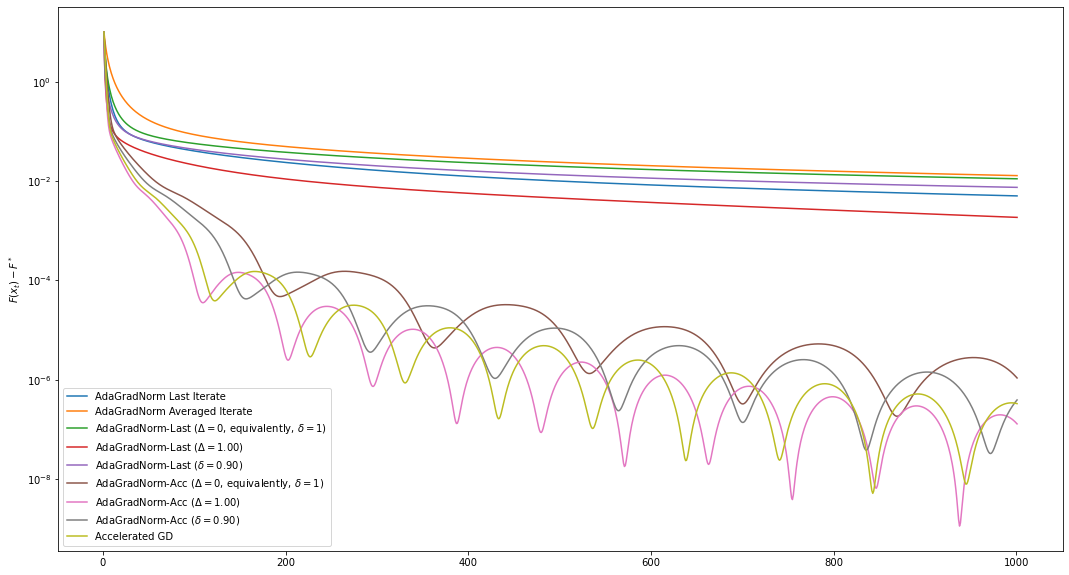}}\caption{\label{fig:Appendix-experiments}Function value gap for different
algorithms}
\end{figure}

In this section, we provide some empirical evidence to compare the
performances of our algorithms in the deterministic setting. Our test
function follows the quadratic function used to prove the lower bound
of the first order method constructed by Nesterov \citep{nesterov2018lectures}.
That is
\[
F(x)=\frac{x[1]^{2}+x[d]^{2}+\sum_{i=1}^{d-1}\left(x[i]-x[i+1]\right)^{2}}{2}-x[1].
\]
where $x[i]$ refers to the $i$-th coordinate of point $x\in\R^{d}$.
It is known that $F$ is $4$-smooth and convex with the unique minimizer
\[
x^{*}[i]=1-\frac{i}{d+1},\forall i\in\left[d\right].
\]
We fix $d=101$ and set the time horizon to $T=1000$ in the test.
The starting point $x_{1}$ is initialized randomly satisfying that
every coordinate is uniformly chosen in $\left[0,1\right)$. All algorithms
share the same $x_{1}$. For the adaptive algorithms, we choose $b_{0}=10^{-2}$
and set $\eta=1$ without any further tuning. We also compare with
an accelerated algorithm \citep{lan2020first}, which requires using
the smoothness constant $L=4$. 

The result is shown in Figure \ref{fig:Appendix-experiments}. We
can find that our Algorithms \ref{alg:AdaGradNorm-Last1} and \ref{alg:AdaGradNorm-Last2}
admit the last iterate convergence. Additionally, both our accelerated
algorithms, i.e., Algorithms \ref{alg:AdaGradNorm-Acc1} and \ref{alg:AdaGradNorm-Acc2},
enjoy the accelerated property without knowing the smoothness parameter
and are competitive against Accelerated Gradient Descent \citep{lan2020first}
which requires the smooth parameter to set the step size. Another
interesting observation is that it seems AdaGradNorm also exhibits
the last iterate convergence. However, whether this is indeed a property
of AdaGradNorm has not been confirmed by the theory. We leave this
as a future direction.

\end{document}